\documentclass[twoside,11pt]{article}

\usepackage{amsmath}
\usepackage{amssymb}
\usepackage{mathtools}
\usepackage{siunitx}
\usepackage{dsfont}
\usepackage{caption}
\usepackage{placeins}
\usepackage{subcaption}
\usepackage[dvipsnames]{xcolor}
\usepackage{todonotes}
\usepackage{tabto}
\usepackage{extdash}
\usepackage{enumitem}

\newcommand{\first}{(i)}
\newcommand{\second}{(ii)}
\newcommand{\third}{(iii)}
\newcommand{\fourth}{(iv)}

\newlength{\hdist}
\newlength{\hdistshort}
\newlength{\vdist}

\usepackage{tikz}
\usetikzlibrary{external}
\tikzexternaldisable %
\usetikzlibrary{positioning}
\usetikzlibrary{fit}
\usetikzlibrary{patterns}
\usetikzlibrary{arrows.meta}
\usetikzlibrary{decorations.markings}
\usetikzlibrary{decorations.pathreplacing}
\usetikzlibrary{shapes.misc}
\usetikzlibrary{backgrounds}
\usetikzlibrary{calc}
\tikzstyle{var}=[draw, circle, fill=white]
\tikzstyle{obs}=[draw, circle, fill=black!25!white]
\tikzstyle{plate}=[draw, rectangle, rounded corners=2mm, line width=0.75pt]
\definecolor{col2}{RGB}{106,125,100}
\definecolor{col1}{RGB}{0,97,147}
\definecolor{col3}{RGB}{217, 129, 0}

\tikzset{gradientline/.style={
		postaction={
			decorate,
			decoration={
				markings,
				mark=between positions 0 and \pgfdecoratedpathlength-3pt step 0.1pt with {
					\pgfmathsetmacro\myval{multiply(divide(
						\pgfkeysvalueof{/pgf/decoration/mark info/distance from start}, \pgfdecoratedpathlength),100)};
					\pgfsetfillcolor{BurntOrange!\myval!MidnightBlue};
					\pgfpathcircle{\pgfpointorigin}{#1};
					\pgfusepath{fill};},
				mark=at position \pgfdecoratedpathlength-0.5pt with {\arrow[black,line width=3pt] {stealth}; }
}}}}
	
\makeatletter
\tikzset{
	dot diameter/.store in=\dot@diameter,
	dot diameter=1pt,
	dot spacing/.store in=\dot@spacing,
	dot spacing=3pt,
	dots/.style={
		line width=\dot@diameter,
		line cap=round,
		dash pattern=on 0pt off \dot@spacing
	}
}
\makeatother

\newcommand{\p}{p}
\newcommand{\given}{\,\vert\,}
\newcommand{\E}{\mathbb{E}}
\newcommand{\supp}{\text{supp}}

\newcommand{\hitting}{\text{\textsc{T}}}

\newcommand{\softmax}{{\pi}}

\DeclareMathOperator*{\argmax}{arg\,max}

\newcommand{\without}{\backslash}
\newcommand{\defeq}{\coloneqq}

\newcommand{\eqkomma}{,}
\newcommand{\eqpunkt}{.}

\newcommand{\ie}{i.e.}
\newcommand{\eg}{e.g.}

\newcommand{\einschub}[1]{\Emdash*#1\Emdash*}
\newcommand{\nachschub}[1]{\Emdash*#1}

\newcommand{\ddBNIRLS}{\mbox{ddBNIRL-S}}
\newcommand{\ddBNIRLT}{\mbox{ddBNIRL-T}}

\usepackage{jmlr2e}
\firstpageno{1}

\usepackage{lastpage}
\jmlrheading{19}{2018}{1-\pageref{LastPage}}{3/18; Revised 10/18}{10/18}{18-113}{Adrian \v{S}o\v{s}i\'c, Elmar Rueckert, Jan Peters, Abdelhak M.\ Zoubir and Heinz Koeppl}
\ShortHeadings{IRL via Nonparametric Spatio-Temporal Subgoal Modeling}{\v{S}o\v{s}i\'c, Rueckert, Peters, Zoubir and Koeppl}

\begin{document}

\title{Inverse Reinforcement Learning \\ via Nonparametric Spatio-Temporal Subgoal Modeling}

\author{\name Adrian \v{S}o\v{s}i\'c \email adrian.sosic@spg.tu-darmstadt.de \\
       \name Abdelhak M.\ Zoubir \email zoubir@spg.tu-darmstadt.de \\
       \addr Signal Processing Group \\
       Technische Universit\"at Darmstadt \\
       64283 Darmstadt, Germany
       \AND
       \name Elmar Rueckert \email rueckert@rob.uni-luebeck.de \\
       \addr Institute for Robotics and Cognitive Systems \\
       University of L\"ubeck \\
       23538 L\"ubeck, Germany
       \AND
       \name Jan Peters \email mail@jan-peters.net \\
       \addr Autonomous Systems Labs \\
       Technische Universit\"at Darmstadt \\
       64289 Darmstadt, Germany
        \AND
        \name Heinz Koeppl \email heinz.koeppl@bcs.tu-darmstadt.de \\
        \addr Bioinspired Communication Systems \\
        Technische Universit\"at Darmstadt \\
       64283 Darmstadt, Germany}

\editor{}

\editor{George Konidaris}
\maketitle

\begin{abstract}%
Advances in the field of inverse reinforcement learning (IRL) 
have %
led to sophisticated inference frameworks that relax the original modeling assumption %
of observing an agent behavior that reflects only a single intention. %
Instead of learning a global behavioral model, 
recent IRL methods divide the demonstration data into parts, to account for the fact that different trajectories may correspond to different intentions, \eg, because they were generated by different domain experts. In this work, we go one step further: using the intuitive concept of \textit{subgoals}, we build upon the premise that even a single trajectory can be explained more efficiently \textit{locally} within a certain context than globally, enabling a more compact representation of the observed behavior. Based on this assumption, we build an implicit intentional model of the agent's goals to forecast its behavior in unobserved situations. The result is an integrated Bayesian prediction framework %
that significantly outperforms existing IRL solutions and provides smooth policy estimates %
consistent with the expert's plan. Most notably, our framework %
naturally handles %
situations where the intentions of the agent change over time %
and classical IRL algorithms fail. %
In addition, due to its probabilistic nature, the model can be %
straightforwardly
applied in %
active learning %
scenarios to guide the demonstration process of the expert.\looseness-1
\end{abstract}


\section{Introduction}
Inverse reinforcement learning (IRL) refers to the problem of inferring %
the intention of an agent, called \textit{the expert},
from observed behavior. Under the Markov decision process~(MDP) formalism \citep{sutton1998reinforcement}, that intention is encoded in the form of a reward function, which provides the agent with instantaneous feedback for each situation %
encountered during the decision-making process. %
Classical IRL methods \citep{ng2000algorithms,abbeel2004apprenticeship,ziebart2008maximum,ramachandran2007bayesian,levine2011nonlinear} %
assume there exists a single \textit{global} reward %
model that explains the %
entire %
set of demonstrations provided by the expert.
In order to relax this rather restrictive modeling assumption, recent IRL methods allow that the agent's intention can change over time \citep{nguyen2015inverse}, or %
they presume that the demonstration data set is inherently composed of %
 several parts \citep{dimitrakakis2011bayesian}, %
 where different trajectories reflect the intentions of different domain experts. %

In this work, we go a step further and start from the premise that\einschub{even in the case of a single expert or trajectory}the demonstrated behavior can be %
explained more efficiently %
\textit{locally} (i.e., within a certain context) %
than by a global reward %
model.
As an illustrative example, we may consider the task shown in Figure~\ref{subfig:JMLR:illustrationSubgoal}, where the %
expert approaches a set of intermediate target positions before finally heading toward a global goal state. Similarly, in Figure~\ref{subfig:JMLR:illustrationCyclic}, the agent eventually returns to its initial position, from where the cyclic process %
repeats. Despite the simplicity of these tasks, the encoding %
of such behaviors in a global intention model %
requires a %
reward structure %
that %
comprises a comparably
large number of redundant state-action-based %
rewards. Alternative %
modeling strategies 
rely on task-dependent expansions of the agent's state representation, \eg, to memorize the last visited goal \citep{hirl2016}, or %
they resort to %
more general decision-making frameworks like semi-MDPs/options~\citep{bradtke1995reinforcement,sutton1999between} in order to achieve %
the necessary  %
level of task abstraction. %

\begin{figure*}
	\centering
	\fboxrule0.5pt%
	\fboxsep0pt%
	\captionsetup{justification=centerlast}%
	\subcaptionbox{sequenced target positions\\\citep{michini2012bayesian}\label{subfig:JMLR:illustrationSubgoal}}[6cm][centerlast]{%
		\fbox{\includegraphics[height=5.5cm]{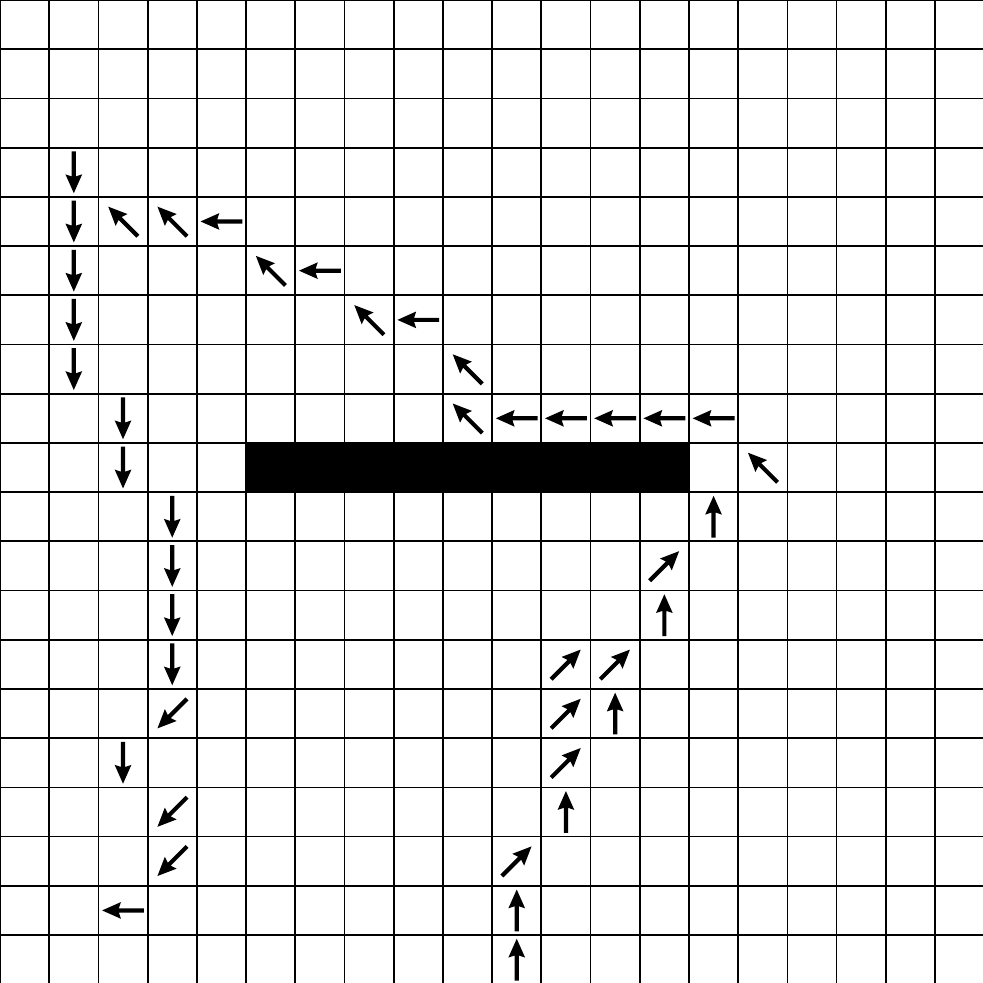}}%
	}
	\hspace{5ex}
	\subcaptionbox{cyclic behavior\label{subfig:JMLR:illustrationCyclic}}{%
		\fbox{\includegraphics[height=5.5cm]{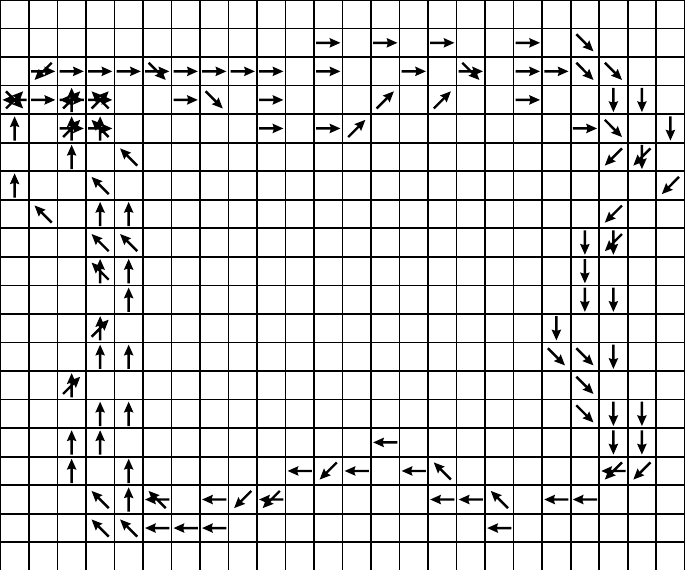}}%
	}
	\caption{Two simple behavior examples that motivate the subgoal principle. The setting is based on the grid world dynamics described in Section~\ref{sec:proofConcept}. In both cases, a task description based on a global reward function is inefficient as it requires %
		many state-action-based rewards to explain the observed trajectory structures. However, the data can be described efficiently through subgoal-based encodings. Both scenarios are analyzed in detail in Section~\ref{sec:results}.}
	\label{fig:JMLR:motivationFigures}
\end{figure*}

In this paper, we present a substantially simpler modeling framework that requires only minimal adaptations to the standard MDP formalism but comes with a hypothesis space of behavioral models that is sufficiently large to cover a broad class of expert policies. The key insight that motivates our approach is that many tasks, like those in Figure~\ref{fig:JMLR:motivationFigures}, can be decomposed into smaller subtasks that require considerably less modeling effort. The resulting low-level task descriptions can then be used as building blocks to synthesize arbitrarily complex behavioral strategies through a suitable sequencing of subtasks. This %
offers the possibility to learn comparably simple task representations %
using the intuitive concept of \textit{subgoals}, which is achieved by efficiently encoding the expert behavior %
using task-adapted partitionings of the system state space/the expert data.

The %
proposed framework builds upon %
the method of Bayesian nonparametric inverse reinforcement learning~\citep[BNIRL,][]{michini2012bayesian}, which can be used to build a subgoal representation of a task %
based on demonstration data\nachschub{however,
without %
learning %
the underlying subgoal relationships %
or providing a policy model %
that can generalize the %
strategy of the %
demonstrator.}
In order to address this limitation, we generalize the BNIRL model using insights from our previous works on nonparametric subgoal modeling \citep{sosic2018} and policy recognition \citep{sosic2018pami}, building a compact intentional model of the expert's behavior that explicitly describes the local dependencies between the demonstrations and the underlying subgoal structure.
The result is an integrated Bayesian prediction framework %
that  %
exploits the spatio-temporal context of the %
demonstrations and is %
capable of producing smooth policy estimates that are consistent with the expert's plan. Furthermore, capturing the full posterior information of the data set enables us to apply the proposed approach in an active learning setting, where the data acquisition process is %
controlled by the posterior predictive distribution of our model.

 In our experimental study, we compare the proposed approach with common baseline methods on a variety of benchmark tasks and real-world scenarios. The results reveal that our approach performs significantly better than the original BNIRL model and alternative IRL solutions on all considered tasks. 
Interestingly enough, our algorithm %
outperforms the baselines %
even when the expert's true reward structure %
is dense and the underlying subgoal 
assumption is violated. %

\subsection{Related Work}

The idea of decomposing complex behavior into smaller %
parts has been around for long and researchers have approached the problem
in many different ways. While the overall field of methods is %
too large to be covered here, most existing approaches can be clearly %
categorized according to %
certain criteria. %
Often, two approaches differ in their exact problem formulation, \ie, we can distinguish between \textit{active} methods, where the learning %
algorithm can interact freely with the environment \citep[\eg, hierarchical reinforcement learning,][]{botvinick2012hierarchical,al2015hierarchical}, and \textit{passive} methods, where the behavioral model is trained %
solely through observation \citep[learning from demonstration,][]{argall2009survey}. Furthermore, we can discriminate between methods that build an explicit intentional model of the underlying task \citep[IRL and option-based models,][]{NIPS2012_4737,sutton1999between}, and such that work directly on the control/trajectory level \citep[skill learning, movement primitives,][]{konidaris2012robot,schaal2005learning}. The latter distinction is sometimes also 
referred to
as \textit{intentional/subintentional} approaches \citep{albrecht2017autonomous,panella2017interactive}.
In order to give a concise summary of the work that is most relevant to ours, we restrict ourselves %
to %
passive approaches, with %
a focus on intentional methods, the field of which is considerably smaller. 
For an overview of %
active approaches, we refer to existing literature, \eg, the work by \citet{daniel2016probabilistic}. 

First, there %
is  the class of
methods that pursue a %
decomposition of the observed behavior on the global level, %
using trajectory-based IRL approaches.
For example, \citet{dimitrakakis2011bayesian} proposed a hierarchical prior over reward functions to %
account for the fact that different trajectories in a data set could %
reflect different behavioral intentions, %
\eg, because they %
were generated by different domain experts. Similarly, \citet{babes2011apprenticeship} follow an expectation-maximization-based clustering approach to group individual trajectories according to %
their underlying reward functions. 
\citet{NIPS2012_4737} generalized this %
idea
by proposing a nonparametric Bayesian model in which the number of intentions %
is {a~priori} unbounded.

While the above methods %
consider %
the expert data %
at a global scale, 
our work is concerned with the problem of %
\textit{subgoal modeling},
which %
is often
conducted in the form of option-based %
reasoning \citep{sutton1999between}. For instance, \cite{tamassia2015learning} proposed a clustering approach based on state distances to find a minimal set of options that can explain the %
expert behavior. While the method provides a simple alternative to handcrafting options, it does not allow any probabilistic treatment of the data and involves many ad-hoc design choices.
Going in the same direction, \cite{daniel2016probabilistic} presented a more principled, probabilistic option framework based on expectation-maximization. Not only is the framework capable of inferring %
sub-policies %
automatically, it can be also used in a reinforcement learning context for intra-option learning. However, the resulting behavioral model is based on point estimates of the %
policy parameters, and the number of sub-policies needs to be specified manually.
The latter problem was solved by \cite{hirl2016}, who proposed a hierarchical nonparametric IRL framework to learn a sequential representation of the demonstrated task, based on a set of transition regions that are defined through local changes in linearity of the observed behavior. %
However, in contrast to the work by \cite{daniel2016probabilistic}, inference is not performed jointly but in several isolated stages where, again, each stage only propagates a point estimate of the associated model 
parameters. Moreover, the temporal relationship of the demonstration data, used to identify the local linearity changes, is considered only in an ad-hoc fashion with the help of a windowing function.

Another general class of models, which explicitly addresses this issue, employs a hidden Markov model~(HMM) structure to establish a temporal relationship between the demonstrations. For instance, the work presented by \cite{nguyen2015inverse} can be regarded as a generalization of the model by \citet{babes2011apprenticeship}, which extends the expectation-maximization framework by imposing a Markov structure on the reward model.
Similarly, \citet{niekum2012learning} use an extended HMM to segment the demonstrations into %
vector autoregressive models, %
in order to learn a suitable set of %
movement primitives. However, the learning of those primitives %
is done in a post-processing step, %
meaning that the quality of the %
final representation crucially depends on the success of the initial segmentation stage. %
In contrast, the method by \citet{Rueckert2013} automatically learns the position and timing of subgoals in the form of via-points, but the number of via-points is assumed to be known and %
the system objective gets finally encoded in form of a global cost function.  Recently, \citet{lioutikov2017learning} presented a related approach based on probabilistic movement primitives that jointly solves the segmentation and learning step for an unknown number of primitives, %
using an expectation-maximization framework. Yet, the model operates purely on the trajectory level and cannot %
reveal the latent intentions of the demonstrator.
Another variant of the approach by \citet{niekum2012learning} that explicitly addresses this problem was proposed by \citet{surana2014bayesian}. In their paper, the authors propose to replace the HMM emission model with an MDP model, in order to infer a policy model from the segmented trajectories instead of recognizing changes in the dynamics. The model was later extended by \cite{ranchod2015nonparametric}, who %
augmented the HMM representation with a beta process model %
to facilitate skill sharing across trajectories. While the resulting model formulation is highly flexible, its major drawback %
 is that inference %
becomes computationally expensive as it involves multiple IRL iterations %
per Gibbs step.

In contrast to the %
HMM-based solutions, which by their %
sequential nature focus on the \textit{temporal} relationship of subtasks, 
the approach presented in this paper establishes a more general correlation structure between demonstrations by employing %
\textit{non-exchangeable prior distributions} over subgoal assignments, \ie, 
without committing to %
purely temporal factorizations of %
subgoals. This results in a compact model representation (\eg, it avoids the need of estimating %
 latent subgoal transition probabilities required in an HMM structure) and adds the flexibility to capture both, the {temporal} and the {spatial} dependencies %
between subtasks.\looseness-1

\subsection{Paper Outline}
The organization %
of the %
paper is as follows: in Section~\ref{sec:BNIRL}, we briefly revisit the BNIRL model and discuss its limitations, which forms the basis for our work. %
Section~\ref{sec:ddBNIRL} then introduces a new intentional subgoal framework, which addresses the %
shortcomings of BNIRL discussed in Section~\ref{sec:BNIRL}. In Section~\ref{sec:predAndInf}, we derive a sampling-based inference scheme for our model and explain %
how the new framework can be used for subgoal extraction and action prediction. Experimental results on both synthetic and real-world data are presented in Section~\ref{sec:results} before we finally conclude our work in Section~\ref{sec:conclusion}.

\section{Bayesian Nonparametric Inverse Reinforcement Learning}
\label{sec:BNIRL}

The purpose of this section is to recapitulate the principle of Bayesian nonparametric inverse reinforcement learning. After briefly discussing all building blocks of the model, we focus on the limitations of the framework, which motivates the need for an extended model formulation and finally leads to a new inference approach, presented afterwards in Section~\ref{sec:ddBNIRL}.

\subsection{Revisiting the BNIRL Framework}
\label{sec:BNIRLrefresher}

Following the common IRL paradigm \citep{ng2000algorithms,zhifei2012}, the goal of BNIRL is to %
infer %
the intentions %
of an agent %
based on demonstration data. 
Starting from a standard 
MDP model, the %
problem is formalized on  
a finite state space~$\mathcal{S}$, assuming a %
time-invariant state transition model $T:\mathcal{S}\times\mathcal{S}\times\mathcal{A}\rightarrow[0,1]$, %
where $\mathcal{A}$ is a finite set of actions available to the agent
at each state. %
For notational convenience, we represent the states in~$\mathcal{S}$ by the integer values $\{1,\ldots,|\mathcal{S}|\}$, where $|\mathcal{S}|$ denotes the cardinality of %
the state space.\looseness-1

In BNIRL,
it is assumed that we can observe a number of expert demonstrations provided in the form of state-action pairs, $\mathcal{D} \defeq \{ (s_d,a_d)\}_{d=1}^D$, 
where each pair $(s_d,a_d)\in\mathcal{S}\times\mathcal{A}$ consists of a state $s_d$ visited by the agent and the corresponding action $a_d$ taken. Herein, $D$ denotes the size of the demonstration set. Throughout the rest of this paper, we will %
use the shorthand notations $\mathbf{s}\defeq\{s_d\}_{d=1}^D$ and $\mathbf{a}\defeq\{a_d\}_{d=1}^D$ to access the collections of expert states and actions individually. Note that the BNIRL model makes no assumptions about the temporal ordering of the demonstrations, \ie, each state-action pair is %
considered to have arisen from a specific but arbitrary time instant of the agent's decision-making process. We will come back to this point later in Sections~\ref{sec:limitations} and~\ref{sec:ddBNIRL-T}. %

In contrast to the classical MDP formalism and most other IRL frameworks, BNIRL does \textit{not} %
presuppose
that the observed expert %
behavior necessarily originates from a single underlying reward function. Instead, it introduces the concept of \textit{subgoals} (and corresponding \textit{subgoal assignments}) with the underlying assumption that, at each decision instant, the expert selects a particular subgoal to plan the next action. %
Each %
subgoal is herein represented by a certain reward function defined on the system state space; in the simplest case, it 
corresponds to a single reward mass placed at a particular goal state %
in $\mathcal{S}$, which we identify with a reward function $R_g:\mathcal{S}\rightarrow\{0,C\}$ of the form
\begin{equation}
R_g(s) \defeq \begin{cases} C & \text{if } g=s\eqkomma \\ 0 & \text{otherwise}\eqkomma \end{cases}
\label{eq:zeroOneReward}
\end{equation}
where $g\in\{1,\ldots,|\mathcal{S}|\}$ indicates the %
subgoal location and $C\in(0,\infty)$ is some positive constant \citep[compare][]{Simsek2005,stolle2002learning,tamassia2015learning}. 

Although %
in principle it is legitimate to associate each subgoal with an arbitrary reward structure %
to encode more complex forms of goal-oriented behavior \citep[see, for example,][]{ranchod2015nonparametric}, the restriction to the reward function class in Equation~\eqref{eq:zeroOneReward} is sufficient in the sense that the same behavioral complexity can be synthesized through a combination of subgoals. This is made possible by the nonparametric nature of BNIRL, \ie, %
because the number of possible subgoals is assumed to be unbounded. The use of the reward model in Equation~\eqref{eq:zeroOneReward} has the advantage, however, that posterior inference about the expert's subgoals %
becomes computationally tractable, as will be explained in Section~\ref{sec:complexity}.
In the following, we therefore focus on the above reward model %
and
summarize the infinite collection of subgoals in the multiset $\mathcal{G}\defeq\{g_k\}_{k=1}^\infty\in %
\bigtimes_{k=1}^\infty\mathcal{S}$,
where we adopt the assumption that 
$\p(\mathcal{G} \given \mathbf{s})=\prod_{k=1}^\infty\p_g(g_k \given \mathbf{s})$.\footnote{\label{foot:conditioning}Notice that the subgoal prior distribution in the original BNIRL formulation does not take the state variable~$\mathbf{s}$ as an argument. Nonetheless, the authors of BNIRL suggest to restrict the support of the distribution %
	to the set of visited states, which indeed implies a conditioning on $\mathbf{s}$.}

The subgoal assignment in BNIRL 
is achieved using a set of indicator variables $\mathbf{\tilde{z}}\defeq\{\tilde{z}_d\in\mathbb{N}\}_{d=1}^D$, which annotate each demonstration pair $(s_d,a_d)$ with its %
unique subgoal index. The prior distribution $\p(\mathbf{\tilde{z}})$ %
is modeled by a Chinese restaurant process~\citep[CRP,][]{aldous1985exchangeability}, which assigns the event that indicator $\tilde{z}_d$ %
points to the $j$th subgoal the prior probability
\begin{equation*}
\p(\tilde{z}_d=j \given \mathbf{\tilde{z}}_{\without d}) \propto
\begin{cases}
n_j & \text{if } j \in \{1,\ldots,K\}\eqkomma \\
\alpha & \text{if } j = K+1 \eqkomma
\end{cases}
\end{equation*}
where $\mathbf{\tilde{z}}_{\without d}\defeq \{\tilde{z}_d\}\setminus \tilde{z}_d$ is a shorthand notation for the collection of all indicator variables except $\tilde{z}_d$. Further, $n_j$ %
denotes the number of assignments to the $j$th subgoal in $\mathbf{\tilde{z}}_{\without d}$, $K$ represents the number of distinct entries in $\mathbf{\tilde{z}}_{\without d}$, and $\alpha\in[0,\infty)$ %
is a parameter controlling the diversity of assignments.

Having targeted a particular subgoal $g_{\tilde{z}_d}$ while being at some state $s_d$, the expert is assumed to choose the next action $a_d$ according to a softmax decision rule, $\softmax:\mathcal{A}\times\mathcal{S}\times\mathcal{S}\rightarrow[0,1]$, which weighs the expected returns of all actions against one another,
\begin{equation}
	\softmax(a_d \given s_d, g_{\tilde{z}_d}) \defeq \frac{\exp\big\{\beta Q^*(s_d, a_d \given g_{\tilde{z}_d})\big\}}{\sum_{a\in\mathcal{A}}\exp\big\{\beta Q^*(s_d, a \given g_{\tilde{z}_d})\big\}} \eqpunkt
	\label{eq:softmaxPolicy}
\end{equation}
Herein, $Q^*(s,a \given g)$ denotes the state-action value \citep[or \textit{Q-value},][]{sutton1998reinforcement} of action $a$ at state $s$ under an optimal policy for the subgoal reward function $R_g$,
\begin{equation}
	Q^*(s,a \given g) \defeq \max_{\bar{\pi}}\, \E \left[ \sum_{n=0}^\infty \gamma^{n}R_g(s_{t=n}) \ \Big|\  s_{t=0}=s, a_{t=0}=a, \bar{\pi} \right] \eqkomma
	\label{eq:Qfunction}
\end{equation}
where the expectation is with respect to the stochastic state-action sequence %
induced by the fixed policy $\bar{\pi}:\mathcal{S}\rightarrow\mathcal{A}$,  with initial action $a$ executed at the starting state~$s$. 
The explicit notation $s_{t=n}$ and $a_{t=n}$ %
is used to disambiguate the temporal index of the %
decision-making process from the demonstration index of the state-action pairs $\{(s_d,a_d)\}$.

The softmax policy $\softmax$ models the expert's %
(in-)ability to maximize the future expected return in view of the targeted subgoal, %
while the %
coefficient $\beta\in[0,\infty)$ is used to express %
the expert's level of confidence in the optimal action.
Combined with the subgoal prior distribution~$p_g$ %
and the partitioning model $p(\mathbf{\tilde{z}})$, we obtain the joint distribution of all demonstrated actions~$\mathbf{a}$, subgoals~$\mathcal{G}$, and subgoal assignments~$\mathbf{\tilde{z}}$ as
\begin{equation}
\p(\mathbf{a},\mathbf{\tilde{z}},\mathcal{G}\given \mathbf{s}) = \p(\mathbf{\tilde{z}}) \prod_{k=1}^{\infty} \p_g(g_k \given \mathbf{s}) \prod_{d=1}^D \softmax(a_d \given s_d, g_{\tilde{z}_d}) \eqpunkt %
\label{eq:BNIRLjoint}
\end{equation}
The structure of this distribution is visualized in form of a Bayesian network in Figure~\ref{fig:graphicalModelA}. It is worth emphasizing that $\pi$\einschub{although referred to as the likelihood model for the state-action \textit{pairs} in the original BNIRL paper}is really just a model for the actions \textit{conditional on the states}. %
In contrast to what is stated in the original paper, the distribution in Equation~\eqref{eq:BNIRLjoint} therefore takes the form of a \textit{conditional} distribution (\ie, conditional on~$\mathbf{s}$), which does not provide any generative model for the state variables.

Posterior inference in BNIRL refers to the (approximate) computation of the conditional distribution $\p(\mathbf{{\tilde{z}}},\mathcal{G} \given \mathcal{D})$, which allows to identify potential subgoal locations and the corresponding subgoal assignments based on the available demonstration data. %
For further details, the reader is referred to the original paper \citep{michini2012bayesian}.

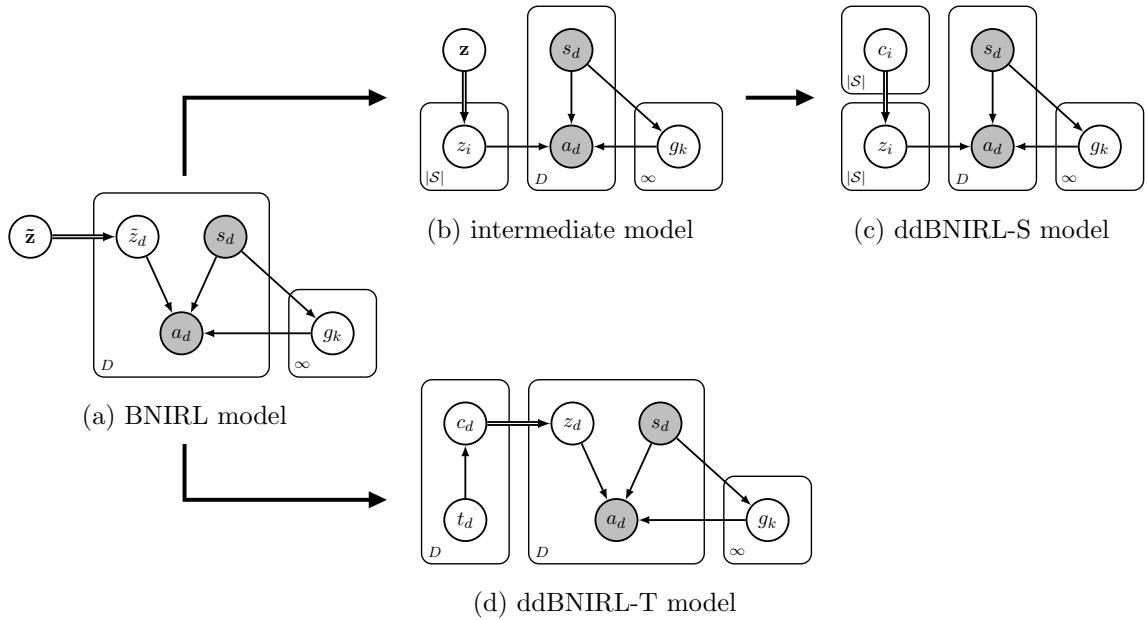
\begin{figure*}[t]
	\hspace{-2ex}
	\begin{tikzpicture}
	
	\node (l) {
		\subcaptionbox{BNIRL model\label{fig:graphicalModelA}}{
			\scalebox{0.7}{
				\begin{tikzpicture}[inner sep=0cm, minimum size = 0.8cm, line width = 1pt, text depth = 0ex, -latex]
				
				\setlength{\hdist}{1.2cm}
				\setlength{\hdistshort}{0.0cm}
				\setlength{\vdist}{1cm}
				
				\node (z) [var, text depth=0.25ex] {$\mathbf{\tilde{z}}$};
				\node (zd) [var, right =\hdist of z, text depth=0.25ex] {$\tilde{z}_d$};
				\node (dummy) [right=\hdistshort of zd] {};
				\node (a) [obs, below=\vdist of dummy] {$a_d$};
				\node (s) [obs, right=\hdistshort of dummy] {$s_d$};
				\node (dummy2) [right=\hdist of s] {};
				\node (R) [var, below=\vdist of dummy2] {$g_k$};
				
				\node (dplate) [plate, fit=(a) (s) (zd), inner sep=2.5ex] {};
				\node (dlabel) [above right=-1.8mm and -1.5mm of dplate.south west] {\scriptsize $D$};
				
				\node (rplate) [plate, fit=(R), inner sep=2.5ex] {};
				\node (rlabel) [above right=-1.8mm and -1.5mm of rplate.south west] {\scriptsize $\infty$};
				
				\draw [double] (z) to (zd);
				\draw (zd) to (a);
				\draw (s) to (a);
				\draw (s) to (R);
				\draw (R) to (a);
				\end{tikzpicture}}}};
	
	\node (m) [above right = -1cm and 0cm of l] {
		\subcaptionbox{intermediate model\label{fig:graphicalModelB}}{
			\scalebox{0.7}{
				\begin{tikzpicture}[inner sep=0cm, minimum size = 0.8cm, line width = 1pt, text depth = 0ex, -latex]
				
				\setlength{\hdist}{1.2cm}
				\setlength{\hdistshort}{0.9cm}
				\setlength{\vdist}{1cm}
				
				\node (z) [var, text depth=-0.125ex] {$\mathbf{z}$};
				\node (zi) [var, below=\vdist of z] {${z}_i$};
				\node (s) [obs, right=\hdist of z] {$s_d$};
				\node (a) [obs, below=\vdist of s] {$a_d$};
				\node (R) [var, right=\hdist of a] {$g_k$};
				
				\node (dplate) [plate, fit=(a) (s), inner sep=2.5ex] {};
				\node (dlabel) [above right=-1.8mm and -1.5mm of dplate.south west] {\scriptsize $D$};
				
				\node (zplate) [plate, fit=(zi), inner sep=2.5ex] {};
				\node (zlabel) [above right=-1.5mm and -1.3mm of zplate.south west] {\scriptsize $|\mathcal{S}|$};
				
				\node (rplate) [plate, fit=(R), inner sep=2.5ex] {};
				\node (rlabel) [above right=-1.8mm and -1.5mm of rplate.south west] {\scriptsize $\infty$};
				
				\draw [double] (z) to (zi);
				\draw (zi) to (a);
				\draw (s) to (a);
				\draw (s) to (R);
				\draw (R) to (a);
				\end{tikzpicture}}}};
	
	\node (r) [right = 1cm of m] {
		\subcaptionbox{\ddBNIRLS\  model\label{fig:graphicalModelC}}{
			\scalebox{0.7}{
				\begin{tikzpicture}[inner sep=0cm, minimum size = 0.8cm, line width = 1pt, text depth = 0ex, -latex]
				
				\setlength{\hdist}{1.2cm}
				\setlength{\hdistshort}{0.9cm}
				\setlength{\vdist}{1cm}
				
				\node (ci) [var, text depth=-0.125ex] {$c_i$};
				\node (zi) [var, below=\vdist of ci] {${z}_i$};
				\node (s) [obs, right=\hdist of ci] {$s_d$};
				\node (a) [obs, below=\vdist of s] {$a_d$};
				\node (R) [var, right=\hdist of a] {$g_k$};

				\node (dplate) [plate, fit=(a) (s), inner sep=2.5ex] {};
				\node (dlabel) [above right=-1.8mm and -1.5mm of dplate.south west] {\scriptsize $D$};
				
				\node (zplate) [plate, fit=(zi), inner sep=2.5ex] {};
				\node (zlabel) [above right=-1.5mm and -1.3mm of zplate.south west] {\scriptsize $|\mathcal{S}|$};
				
				\node (cplate) [plate, fit=(ci), inner sep=2.5ex] {};
				\node (clabel) [above right=-1.5mm and -1.3mm of cplate.south west] {\scriptsize $|\mathcal{S}|$};
				
				\node (rplate) [plate, fit=(R), inner sep=2.5ex] {};
				\node (rlabel) [above right=-1.8mm and -1.5mm of rplate.south west] {\scriptsize $\infty$};
				
				\draw [double] (ci) to (zi);
				\draw (zi) to (a);
				\draw (s) to (a);
				\draw (s) to (R);
				\draw (R) to (a);
				\end{tikzpicture}}}};

	\node (t) [below right = -1cm and 0cm of l] {
		\subcaptionbox{\ddBNIRLT\  model}{
			\scalebox{0.7}{
				\begin{tikzpicture}[inner sep=0cm, minimum size = 0.8cm, line width = 1pt, text depth = 0ex, -latex]
				
				\setlength{\hdist}{1.2cm}
				\setlength{\hdistshort}{0.0cm}
				\setlength{\vdist}{1cm}
				
				\node (c) [var] {${c}_d$};
				\node (td) [var, below =\vdist of c, text depth=0.25ex] {${t}_d$};
				\node (zd) [var, right =\hdist of c, text depth=0.25ex] {${z}_d$};
				\node (dummy) [right=\hdistshort of zd] {};
				\node (a) [obs, below=\vdist of dummy] {$a_d$};
				\node (s) [obs, right=\hdistshort of dummy] {$s_d$};
				\node (dummy2) [right=\hdist of s] {};
				\node (R) [var, below=\vdist of dummy2] {$g_k$};
				
				\node (dplate) [plate, fit=(a) (s) (zd), inner sep=2.5ex] {};
				\node (dlabel) [above right=-1.8mm and -1.5mm of dplate.south west] {\scriptsize $D$};
				
				\node (rplate) [plate, fit=(R), inner sep=2.5ex] {};
				\node (rlabel) [above right=-1.8mm and -1.5mm of rplate.south west] {\scriptsize $\infty$};
				
				\node (dplate2) [plate, fit=(td) (c), inner sep=2.5ex] {};
				\node (dlabel2) [above right=-1.8mm and -1.5mm of dplate2.south west] {\scriptsize $D$};
				
				\draw [double] (c) to (zd);
				\draw (td) to (c);
				\draw (zd) to (a);
				\draw (s) to (a);
				\draw (s) to (R);
				\draw (R) to (a);
				\end{tikzpicture}}}};
	
	\draw [line width=1.5pt, arrows={-Triangle[width=6pt]}] ([yshift=0.5ex]l.north) |- ([yshift=2.3ex]m.west);
	\draw [line width=1.5pt, arrows={-Triangle[width=6pt]}] ([yshift=2.3ex, xshift=1ex]m.east) |- ([xshift=0.5ex, yshift=2.3ex]r.west);
	\draw [line width=1.5pt, arrows={-Triangle[width=6pt]}] ([yshift=-0ex]l.south) |- ([yshift=-0ex]t.west);
\end{tikzpicture}
\caption{%
	Relationships between all discussed subgoal models, illustrated in the form of Bayesian networks. Shaded nodes represent observed variables; deterministic dependencies are highlighted using double strokes.}
\label{fig:graphicalModels}
\end{figure*}

\subsection{Limitations of BNIRL}
\label{sec:limitations}
Subgoal-based inference is a well-motivated approach to IRL %
and the BNIRL framework has shown promising results in a variety of real-world scenarios. Yet, the model formulation by \cite{michini2012bayesian} comes with a number of significant conceptual %
limitations, which we explain in detail in the following paragraphs.

\phantomsection
\subsubsection*{Limitation 1: Subgoal Exchangeability and Posterior Predictive Policy} 
\label{phantom:lim1}
The central limitation of BNIRL is that %
the framework is restricted to pure subgoal extraction and does \textit{not} inherently 
provide a reasonable %
mechanism to %
generalize the expert behavior based on %
the inferred subgoals. %
The reason lies in the particular design of the framework, which, at its heart, treats the subgoal assignments~$\mathbf{\tilde{z}}$ as \textit{exchangeable random variables} \citep{aldous1985exchangeability}. By implication, %
the %
induced partitioning model $p(\mathbf{\tilde{z}})$ is agnostic about the %
covariate information contained 
in the data set %
and the resulting behavioral model is unable to 
propagate the expert %
knowledge 
to new situations. %

To illustrate the problem, let us investigate the predictive action distribution that arises from the original BNIRL formulation. %
For simplicity and without loss of generality, %
we may %
assume that we have 
perfectly 
inferred all subgoals $\mathcal{G}$ and %
corresponding subgoal assignments~$\mathbf{\tilde{z}}$ from the demonstration set $\mathcal{D}$.
Denoting by $a^*\in\mathcal{A}$ the predicted action %
at some new state~$s^*\in\mathcal{S}$, %
the BNIRL model yields
\begin{align}
	\p(a^* \given s^*,\mathcal{D},\mathbf{\tilde{z}},\mathcal{G}) &= \sum_{\tilde{z}^* \in \mathbb{N}} \p(a^*,\tilde{z}^* \given s^*,\mathcal{D},\mathbf{\tilde{z}},\mathcal{G}) \nonumber \\
	&= \sum_{\tilde{z}^* \in \mathbb{N}} \p(a^* \given \tilde{z}^*,s^*,\mathcal{D},\mathbf{\tilde{z}},\mathcal{G}) \p(\tilde{z}^* \given s^*,\mathcal{D},\mathbf{\tilde{z}},\mathcal{G}) \nonumber
	\\ &\stackrel{(\star)}{=} \sum_{\tilde{z}^* \in \mathbb{N}} \p(a^* \given \tilde{z}^*,s^*,g_{\tilde{z}^*}) \p(\tilde{z}^* \given \mathbf{\tilde{z}})  \eqkomma 
	\label{eq:BNIRLpredictive}
\end{align}
where $\tilde{z}^*\in\mathbb{N}$ is the latent subgoal index belonging to $s^*$. Note that $\p(a^* \given \tilde{z}^*,s^*,g_{\tilde{z}^*})$ can either represent the softmax decision rule $\softmax(a^* \given s^*, g_{\tilde{z}^*})$ from Equation~\eqref{eq:softmaxPolicy} or %
an %
optimal (deterministic) policy for subgoal $g_{\tilde{z}^*}$, depending on whether we aspire to %
describe the \textit{noisy expert behavior} at $s^*$ or want to determine %
an \textit{optimal action} according to the inferred reward model. 
The last equality in Equation~\eqref{eq:BNIRLpredictive}, indicated by $(\star)$, follows from the conditional independence properties implied by Equation~\eqref{eq:BNIRLjoint}, which can be easily verified using d-separation \citep{koller2009probabilistic} on the graphical model in Figure~\ref{fig:graphicalModelA}.

As Equation~\eqref{eq:BNIRLpredictive} reveals, the predictive model is characterized by the posterior distribution $\p(\tilde{z}^* \given s^*,\mathcal{D},\mathbf{\tilde{z}},\mathcal{G})$ of the latent subgoal assignment $\tilde{z}^*$ of state $s^*$\nachschub{the intuition being that, in order to generalize the expert's plan to a new situation, we need %
to %
take into account the gathered information about what would be a likely subgoal targeted by the expert at~$s^*$.} %
However, in BNIRL, %
the %
distribution 
$\p(\tilde{z}^* \given s^*,\mathcal{D},\mathbf{\tilde{z}},\mathcal{G})$ is modeled \textit{without consideration of the query state~$s^*$, or any other observed variable}. %
By conditional independence~(Equation~\ref{eq:BNIRLjoint}),
the distribution effectively reduces~$(\star)$ to the CRP prior $\p(\tilde{z}^* \given \mathbf{\tilde{z}})$, which, due to its intrinsic exchangeability property, only considers the subgoal frequencies of the readily inferred %
assignments~$\mathbf{\tilde{z}}$. 
Clearly, a subgoal assignment mechanism %
based solely on
frequency information is of little use when it comes to predicting the expert behavior %
as it 
will inevitably  ignore the %
structural information contained in the demonstration set
and %
always return the same subgoal probabilities %
at all query states, regardless of the agent's actual 
situation. %
By contrast, a reasonable %
assignment mechanism should inherently take into account the context of the agent's current state~$s^*$ 
when deciding about the next action. %

While the authors of BNIRL %
discuss the action selection problem in their paper and propose %
an assignment strategy for new states based on action marginalization, their approach does not provide a satisfactory solution to the problem
because the alleged conditioning on the query state \citep[see Equation~19 in the original paper,][]{michini2012bayesian} has no %
effect on the involved subgoal indicator variable, as shown by Equation~\eqref{eq:BNIRLpredictive} above. %
The only way to remedy the problem without modifying the model is to use an external post-processing scheme like the waypoint method, %
discussed in the next section.

\phantomsection
\subsubsection*{Limitation 2: Spatial and Temporal Context} 
\label{phantom:lim2}
The waypoint method, described at full length in a follow-up paper %
by \citet{michini2015bayesian}, is a post-processing routine %
to convert the subgoals %
identified through BNIRL into a valid option model \citep{sutton1999between}. The obtained model reconstructs the high-level plan of the demonstrator by sequencing the inferred subgoals in a way that %
complies with the spatio-temporal relationships of the expert's decisions as observed %
during the demonstration phase. %
To this end, the required initiation and termination sets of the option-policies are constructed %
by considering the state distances to the identified 
subgoals as well as their %
temporal ordering %
prescribed by the expert.
\begin{figure*}
	\centering
	
	\newlength{\bagOfWordsScale}
	\setlength{\bagOfWordsScale}{0.22mm}
	\begin{tikzpicture}
	
	\usetikzlibrary{arrows.meta, arrows, positioning}
	
	\definecolor{col8}{RGB}{191,85,105}
	\definecolor{col1}{RGB}{176,41,41}
	\definecolor{col2}{RGB}{238,129,79}
	\definecolor{col3}{RGB}{245,186,65}
	\definecolor{col4}{RGB}{248,215,87}
	\definecolor{col5}{RGB}{198,207,96}
	\definecolor{col6}{RGB}{115,194,205}
	\definecolor{col7}{RGB}{84,167,209}
	
	\tikzset{
		arr0/.pic={\draw[solid, line width=2pt, -{Triangle[scale=0.75]}] (0,0) to (0.8,0);}, 
		arr45/.pic={\draw [solid, line width=2pt, -{Triangle[scale=0.75]}] (0,0) -- (0.566,0.566);},
		arr90/.pic={\draw [solid, line width=2pt, -{Triangle[scale=0.75]}] (0,0) -- (0,0.8);},
		arr135/.pic={\draw [solid, line width=2pt, -{Triangle[scale=0.75]}] (0,0) -- (-0.566,0.566);},
		arr180/.pic={\draw [solid, line width=2pt, -{Triangle[scale=0.75]}] (0,0) -- (-0.8,0);},
		arr225/.pic={\draw [solid, line width=2pt, -{Triangle[scale=0.75]}] (0,0) -- (-0.566,-0.566);},
		arr270/.pic={\draw [solid, line width=2pt, -{Triangle[scale=0.75]}] (0,0) -- (0,-0.8);},
		arr315/.pic={\draw [solid, line width=2pt, -{Triangle[scale=0.75]}] (0,0) -- (0.566,-0.566);}
	}
	
	\node (a) [label={below:imperfect demonstrations}, scale=\bagOfWordsScale] {
		\begin{tikzpicture}
		\draw[step=1.0, black, line width=1pt, xshift=0.5cm, yshift=0.5cm] (0,0) grid (5,5);
		\node at (1,5) {\tikz{\draw pic {arr315}}};
		\node at (2,5) {\tikz{\draw pic {arr315}}};
		\node at (3,5) {\tikz{\draw pic {arr0}}};
		\node at (4,5) {\tikz{\draw pic {arr0}}};
		\node at (5,5) {\tikz{\draw pic {arr270}}};
		\node at (5,4) {\tikz{\draw pic {arr225}}};
		\node at (5,3) {\tikz{\draw pic {arr180}}};
		\node at (4,3) {\tikz{\draw pic {arr180}}};
		\node at (3,3) {\tikz{\draw pic {arr135}}};
		\node at (2,3) {\tikz{\draw pic {arr180}}};
		\node at (1,3) {\tikz{\draw pic {arr270}}};
		\node at (1,2) {\tikz{\draw pic {arr315}}};
		\node at (1,1) {\tikz{\draw pic {arr45}}};
		\node at (2,1) {\tikz{\draw pic {arr0}}};
		\node at (3,1) {\tikz{\draw pic {arr0}}};
		\node at (4,1) {\tikz{\draw pic {arr45}}};
		\node at (5,1) {\tikz{\draw pic {arr315}}};
		\end{tikzpicture}
	};

	\node (b) [right=2cm of a, label={below:bag-of-words clustering}, scale=\bagOfWordsScale] {
		\begin{tikzpicture}[line width=2, inner sep=0.75mm, dotted]
		\node [solid, draw, minimum size=5cm, anchor=south west, line width=3pt, rounded corners=5mm] at (0.5,0.5) {};
		\node [draw=col1, rounded corners=2mm] at (0.9,2.1) {
			\begin{tikzpicture}
			\node at (0,0) {\tikz{\draw pic {arr0}}};
			\node at (0,0.4) {\tikz{\draw pic {arr0}}};
			\node at (0,0.8) {\tikz{\draw pic {arr0}}};
			\node at (0,1.2) {\tikz{\draw pic {arr0}}};
			\end{tikzpicture}
		};
		
		\node [draw=col2, rounded corners=2mm] at (3.9,1.9) {
			\begin{tikzpicture}
			\node at (0,0) {\tikz{\draw pic {arr45}}};
			\node at (0.4,0) {\tikz{\draw pic {arr45}}};
			\end{tikzpicture}
		};
		
		\node [draw=col4, rounded corners=2mm] at (2.8,4.6) {
			\begin{tikzpicture}
			\node at (0,0) {\tikz{\draw pic {arr135}}};
			\end{tikzpicture}
		};
		
		\node [draw=col5, rounded corners=2mm] at (1,4.4) {
			\begin{tikzpicture}
			\node at (0,0) {\tikz{\draw pic {arr180}}};
			\node at (0,0.4) {\tikz{\draw pic {arr180}}};
			\node at (0,0.8) {\tikz{\draw pic {arr180}}};
			\end{tikzpicture}
		};
		
		\node [draw=col6, rounded corners=2mm] at (2.5,1.4) {
			\begin{tikzpicture}
			\node at (0,0) {\tikz{\draw pic {arr225}}};
			\end{tikzpicture}
		};
		
		\node [draw=col7, rounded corners=2mm] at (2.6,3.0) {
			\begin{tikzpicture}
			\node at (0,0) {\tikz{\draw pic {arr270}}};
			\node at (0.4,0) {\tikz{\draw pic {arr270}}};
			\end{tikzpicture}
		};
		
		\node [draw=col8, rounded corners=2mm] at (4.2,4.1) {
			\begin{tikzpicture}
			\node at (0,0) {\tikz{\draw pic {arr315}}};
			\node at (0,0.4) {\tikz{\draw pic {arr315}}};
			\node at (0,0.8) {\tikz{\draw pic {arr315}}};
			\node at (0,1.2) {\tikz{\draw pic {arr315}}};
			\end{tikzpicture}
		};
		
	\end{tikzpicture}
};

\node (c) [right=2cm of b, label={below:noisy trajectory labeling}, scale=\bagOfWordsScale] {
	\begin{tikzpicture}[inner sep=0]
	\node [draw, minimum size=1cm, fill=col8] at (1,5) {\tikz{\draw pic {arr315}}};
	\node [draw, minimum size=1cm, fill=col8] at (2,5) {\tikz{\draw pic {arr315}}};
	\node [draw, minimum size=1cm, fill=col1] at (3,5) {\tikz{\draw pic {arr0}}};
	\node [draw, minimum size=1cm, fill=col1] at (4,5) {\tikz{\draw pic {arr0}}};
	\node [draw, minimum size=1cm, fill=col7] at (5,5) {\tikz{\draw pic {arr270}}};
	\node [draw, minimum size=1cm, fill=col6] at (5,4) {\tikz{\draw pic {arr225}}};
	\node [draw, minimum size=1cm, fill=col5] at (5,3) {\tikz{\draw pic {arr180}}};
	\node [draw, minimum size=1cm, fill=col5] at (4,3) {\tikz{\draw pic {arr180}}};
	\node [draw, minimum size=1cm, fill=col4] at (3,3) {\tikz{\draw pic {arr135}}};
	\node [draw, minimum size=1cm, fill=col5] at (2,3) {\tikz{\draw pic {arr180}}};
	\node [draw, minimum size=1cm, fill=col7] at (1,3) {\tikz{\draw pic {arr270}}};
	\node [draw, minimum size=1cm, fill=col8] at (1,2) {\tikz{\draw pic {arr315}}};
	\node [draw, minimum size=1cm, fill=col2] at (1,1) {\tikz{\draw pic {arr45}}};
	\node [draw, minimum size=1cm, fill=col1] at (2,1) {\tikz{\draw pic {arr0}}};
	\node [draw, minimum size=1cm, fill=col1] at (3,1) {\tikz{\draw pic {arr0}}};
	\node [draw, minimum size=1cm, fill=col2] at (4,1) {\tikz{\draw pic {arr45}}};
	\node [draw, minimum size=1cm, fill=col8] at (5,1) {\tikz{\draw pic {arr315}}};
	\draw[step=1.0, black, line width=1pt, xshift=1cm, yshift=0.5cm] (0,0) grid (5,5);
	\end{tikzpicture}
};

\draw [-latex, line width=3pt, shorten >= 8, shorten <= 8] (a) to (b);
\draw [-latex, line width=3pt, shorten >= 8, shorten <= 8] (b) to (c);

\end{tikzpicture}

\caption{A diagram to illustrate the implications of the exchangeability assumption in~BNIRL. Similar to a bag-of-words model \citep{blei2003latent,yang2007evaluating}, the BNIRL partitioning %
mechanism ignores the spatio-temporal context of the data, %
which makes it difficult to %
discriminate demonstration noise %
from a real change of the agent's intentions. Note that the diagram illustrates the partitioning process in a simplified way as it only %
shows the effect of the prior $p(\mathbf{\tilde{z}})$ but neglects %
the impact of the likelihood model $\pi$. %
While the latter %
\textit{does} indeed consider the state context of the actions, it cannot account for 
spatial or temporal patterns in the
data as it processes all state-action pairs %
separately.}
\label{fig:noisyLabels}
\end{figure*}
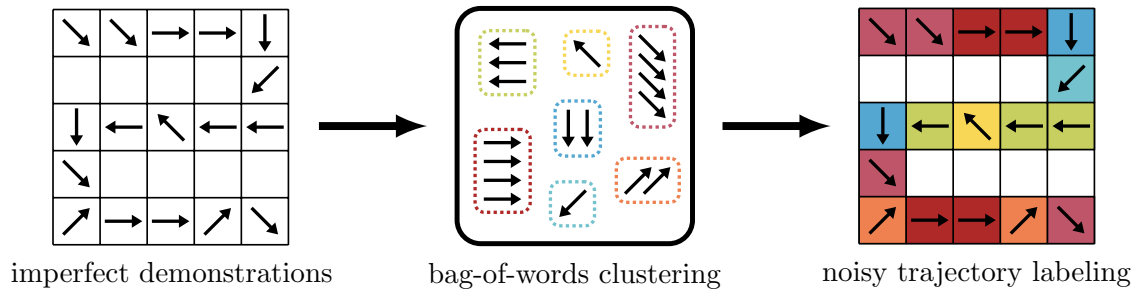

When combined with BNIRL, this method allows to synthesize a behavioral model that %
mimics the observed expert behavior. However, the strategy comes with a number of significant drawbacks: 
\begin{enumerate}
\item[\first] Using the waypoint method, the spatio-temporal %
relationships between the individual demonstrations are
explored only in a post-hoc fashion and %
are largely ignored during the actual inference procedure~(the state information enters via the likelihood model $\pi$ but is not considered by the partitioning %
model $p(\mathbf{\tilde{z}}$), as explained in Limitation~\hyperref[phantom:lim1]{1}). %
This lack of context-awareness makes the inference mechanism overly prone to demonstration noise (see Figure~\ref{fig:noisyLabels} and results in Section~\ref{sec:results}).
\item[\second] Measuring proximities to subgoals in order to determine the right visitation order requires some form of distance metric defined on the state space. If the system states correspond to physical locations, constructing such a metric is usually straightforward. However, in the general case where states encode arbitrary abstract information (see %
 example in Section~\ref{sec:randomMDP}), %
it can become difficult to design that metric by hand. Unfortunately, the BNIRL framework does not provide any %
solution to this problem. 
\item[\third] The waypoint method cannot be applied to multiple unaligned trajectories (\eg, obtained from different experts) or %
in cases where the data set does not carry any temporal information. This situation occurs, for instance, when %
the expert data is provided as separate state-action pairs with unknown timestamps and not given in form of coherent trajectories (see again example in Section~\ref{sec:randomMDP}). %
\item[\fourth] Assigning a particular visitation order to the inferred subgoals is meaningful only if the expert eventually reaches those subgoals during the demonstration phase (or if, at least, the subgoals lie ``close'' to the visited states in terms of the aforementioned distance metric). Finding subgoals with such properties can be guaranteed by constraining the support of the subgoal prior distribution $p_g$ to states that are near to the expert data (see footnote on page~\pageref{foot:conditioning}) but this %
reduces the flexibility of the model and potentially disables compact encodings of the task~(Figure~\ref{fig:globalVSlocal}).
\end{enumerate}

\begin{figure}[p]
	
	\tikzset{->-/.style={decoration={
				markings,
				mark=at position .5 with {\arrow{stealth}}},postaction={decorate}}}
	
	\tikzset{cross/.style={cross out, draw=black, fill=none, minimum size=2*(#1-\pgflinewidth), inner sep=0pt, outer sep=0pt}, cross/.default={0.75ex}}
	
	\tikzstyle{goal} = [draw, circle]
	\tikzstyle{final} = [cross, minimum size=6.5pt]
	
	\tikzset{distracted/.pic ={
			\node (S) at (0,0) {};
			\node (I) at (1,0) {};
			\node (A) at (1,1) {};
			\node [circle] (E) at (1.5,0.5) {};
			\node [final] (G) at (2,0) {};
	}}
	
	\tikzset{collide/.pic ={
			\node (L) at (0,0) {};
			\node (L1) at (1,0.25) {};
			\node [final] (M) at (2,0.5) {};
			\node (R1) at (1,0.75) {};
			\node (R) at (0,1) {};
	}}
	\centering
	
	\begin{tikzpicture}
	\node (A1) at (0,0) {
		\begin{tikzpicture} [line width=1.5pt, minimum size=3ex, inner sep=0mm]
		\pic[scale=1.5]{collide};
		\draw [-{Stealth[line width=3pt]}] (L.center) to (L1);
		\draw [-{Stealth[line width=3pt]}] (R.center) to (R1);
		\draw [dots, shorten <=26, shorten >=2] (L.center) to (M.center);
		\draw [dots, shorten <=26, shorten >=2] (R.center) to (M.center);
		\end{tikzpicture}};
	
	\node (A2) at (5,0) {
		\begin{tikzpicture} [line width=1.5pt, minimum size=3ex, inner sep=0mm]
		\pic[scale=1.5]{collide};
		\draw [dots, shorten <=26, shorten >=2] (L.center) to (M.center);
		\draw [dots, shorten <=26, shorten >=2] (R.center) to (M.center);
		\draw [col3, arrows={-{Stealth[black, line width=3pt]}}] (L.center) to (L1);
		\draw [col1, arrows={-{Stealth[black, line width=3pt]}}] (R.center) to (R1);
		\node [goal, col3] at (L1) {};
		\node [goal, col1] at (R1) {};
		\end{tikzpicture}};
	
	\node (A3) at (10,0) {
		\begin{tikzpicture} [line width=1.5pt, minimum size=3ex, inner sep=0mm]
		\pic[scale=1.5]{collide};
		\draw [dots, shorten <=26, shorten >=2] (L.center) to (M.center);
		\draw [dots, shorten <=26, shorten >=2] (R.center) to (M.center);
		\draw [col2, arrows={-{Stealth[black, line width=3pt]}}] (L.center) to (L1);
		\draw [col2, arrows={-{Stealth[black, line width=3pt]}}] (R.center) to (R1);
		\node [goal, col2] at (M) {};
		\end{tikzpicture}};
	
	\node [label={below:demonstrations}] (B1) at (0,-2.5) {
		\begin{tikzpicture} [line width=1.5pt, minimum size=3ex, inner sep=0mm]
		\pic[scale=1.5]{distracted};
		\draw [-{Stealth[line width=3pt]}] (S.center) to (I.center) to (A.center) to (E);
		\draw [dots, shorten <= 17] (A.center) to (G.center);
		\end{tikzpicture}};
	
	\node [label={below:local search}] (B2) at (5,-2.5) {
		\begin{tikzpicture} [line width=1.5pt, minimum size=3ex, inner sep=0mm]
		\pic[scale=1.5]{distracted};
		\draw [dots, shorten <= 17] (A.center) to (G.center);
		\node [goal, col1] at (I) {};
		\node [goal, col2] at (A) {};
		\node [goal, col3] at (E) {};
		\draw [col1, line cap=round] (S.center) to (I.center);
		\draw [col2, line cap=round] (I.center) to (A.center);
		\draw [col3, line cap=round, arrows={-{Stealth[black, line width=3pt]}}] (A.center) to (E);
		\end{tikzpicture}};
	
	\node [label={below:global search}] (B3) at (10, -2.5) {
		\begin{tikzpicture} [line width=1.5pt, minimum size=3ex, inner sep=0mm]
		\pic[scale=1.5]{distracted};
		\draw [dots, shorten <= 17] (A.center) to (G.center);
		\node [goal, col1] at (G) {};
		\node [goal, col3] at (A) {};
		\draw [col1, line cap=round] (S.center) to (I.center);
		\draw [col3, line cap=round] (I.center) to (A.center);
		\draw [col1, line cap=round, arrows={-{Stealth[black, line width=3pt]}}] (A.center) to (E);
		\end{tikzpicture}};
	
	\node at (5,-5) {
		\begin{tikzpicture}[line width=1.5pt]
		\node (z) at (0,0) {};
		\node (arr) [inner sep=0cm, outer sep=0mm] at (1,0) {};
		\draw  [-{Stealth[line width=3pt]}] (z.center) -- (arr);
		\node (t) [anchor=west, right=0cm of arr, text depth=0.1ex] {trajectory};
		\node (g1) [goal, right=0.5cm of t, col1] {};
		\node (g2) [goal, right=0.05cm of g1, col2] {};
		\node (g3) [goal, right=0.05cm of g2, col3] {};
		\node (s) [anchor=west, right=0cm of g3, text depth=0.1ex] {subgoals};
		\node (f) [final, right=0.5cm of s] {};
		\node (g) [anchor=west, right=0cm of f, text depth=0.1ex] {global goal};
		
		\node [fit=(z) (g), draw, line width=1pt, gray, rounded corners=2mm, inner sep=0.2ex] {};
		\end{tikzpicture}};
	
\end{tikzpicture}

\caption{Difference between local (constrained) and global (unconstrained) subgoal search. The top and the bottom row depict two different sets of demonstration data (solid lines), together with potential goal/subgoal locations (crosses/circles) that explain the observed behavior. Color indicates the corresponding subgoal assignment of each trajectory segment. \textbf{Top:}~two trajectories approaching the same goal. \textbf{Bottom:}~the agent is heading toward a global goal, gets temporarily distracted, and then follows up on its original plan. \textbf{Left:}~observed trajectories. \textbf{Center:}~example partitioning under the assumption that the expert reached all subgoals during the demonstration. \textbf{Right:}~example partitioning without restriction on the subgoal locations, yielding a more compact encoding of the task.\looseness-1}
\label{fig:globalVSlocal}
\end{figure}
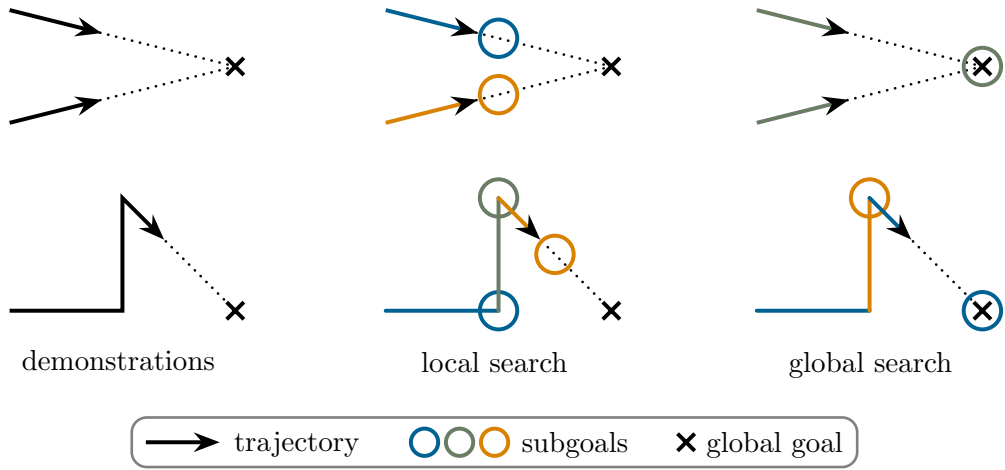

\begin{figure}[p]
	\centering
	\subcaptionbox{time-varying intentions\label{subfig:crossing}}[5cm]{
		\tikzsetnextfilename{crossing}
		\begin{tikzpicture}[scale=3]
		\path [gradientline=1pt] (0,0) -- (1,1) -- (1,0) -- (0,1);
		\node [draw=black, circle, line width=1pt, minimum size=7mm, dotted] at (0.5,0.5) {};
		\end{tikzpicture}
	}
	\hspace{0.5cm}
	\subcaptionbox{time-invariant intentions\label{subfig:nocrossing}}[5cm]{
		\tikzsetnextfilename{nocrossing}
		\begin{tikzpicture}[scale=3]
		\path [gradientline=1pt] (0,0) -- (1,1) -- (1,0) -- (0.5,0.5) -- (0.85,0.85);
		\node [draw=black, circle, line width=1pt, minimum size=7mm, dotted] at (0.5,0.5) {};
		\end{tikzpicture}
	}
	\caption{Schematic comparison of the two basic behavior types, illustrated using two different agent trajectories. Color indicates the temporal progress. 	%
		(a)~Time-varying intentions may cause the agent to perform a different action when revisiting a %
		state (dotted circle). %
		(b)~By contrast, time-invariant intentions imply a simple state-to-action policy: the agent has no incentive to perform a different action at an already visited state since\einschub{by definition}the underlying %
		objective has remained unchained. %
		Diverging actions, as observed at the crossing point in~%
		the left subfigure, can therefore only be explained as a result of suboptimal behavior.\looseness-1}
	\label{fig:crossing}
\end{figure}
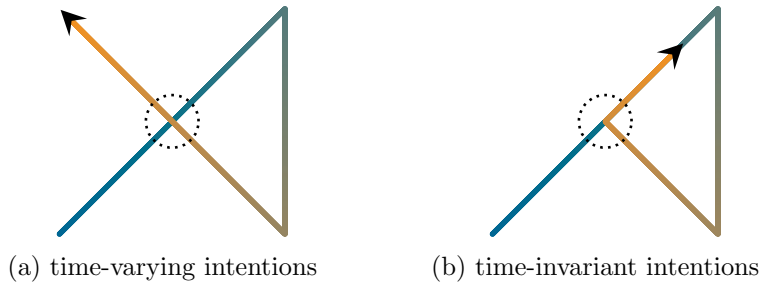
\tikzexternaldisable

\phantomsection
\subsubsection*{Limitation 3: Inconsistency under Time-Invariance}
\label{phantom:lim3}
Reasoning about the intentions of an %
agent, there are two basic types of behavior one may encounter:
\begin{enumerate}[topsep=-\parskip+1.5ex, itemsep=1ex, parsep=0mm]
\item[$\bullet$] either the agent follows a static strategy %
to optimize a fixed objective \citep[as assumed in the standard MDP formalism,][]{sutton1998reinforcement}, or
\item[$\bullet$] the intentions of the agent change over time.
\end{enumerate}
The latter is clearly the more general case but also poses a more difficult inference problem in that it requires us both, to identify the %
intentions of the agent and to understand their temporal relationship. The static scenario, in contrast, implies that there exists an optimal policy for the task %
in %
form of a simple state-to-action mapping $\pi:\mathcal{S}\rightarrow\mathcal{A}$ \citep{puterman1994}, which from the very beginning %
imprints a specific structure on the inference problem.

The BNIRL model generally falls into the second category since it freely allocates its subgoals per \textit{decision instant} and not per state, allowing a flexible change of the agent's objective. Yet, it is important to understand that 
the model does not actually distinguish  %
 between the two described scenarios. %
As explained in Limitation~\hyperref[phantom:lim2]{2}, the temporal aspect of the data is not explicitly modeled by the BNIRL framework, even though the waypoint method %
subsequently tries to capture the overall %
chronological %
order of  events.
As a consequence, the model is not tailored to either of the two scenarios: on the one hand, it ignores the valuable temporal context %
that is needed  in the time-varying case to  %
reliably discriminate demonstration noise from a real change of the agent's intention. %
On the other hand, the model is agnostic about the predefined time-invariant nature of the optimal policy in the static scenario. This lack of structure not only makes the inference problem harder than necessary in both cases; it also allows the model to learn inconsistent data representations in the static case %
since the same state can be potentially assigned to more than one subgoal, violating the above-mentioned %
state-to-action rule~(Figure~\ref{fig:crossing}).

\phantomsection
\subsubsection*{Limitation 4: Subgoal Likelihood Model}
\label{phantom:lim4}
Apart from the discussed limitations of the BNIRL partitioning model, %
it turns out there are %
two problematic issues %
concerning the %
softmax %
likelihood model in Equation~\eqref{eq:softmaxPolicy}. %
On the following pages, we demonstrate that %
the specific form of the model encodes a number of properties that %
are indeed contradictory to our intuitive understanding of subgoals. While these properties %
are less critical for the final 
prediction of the expert behavior, %
it turns out they drastically affect %
the %
\textit{localization} of subgoals. Since the cause of these effects is somewhat hidden in the model equation, %
we defer the detailed explanation %
to %
Section~\ref{sec:actionLikelihood}.

\phantomsection
\subsubsection*{Limitation 5: State-Action Demonstrations} 
\label{phantom:lim5}
Lastly, a minor problem of the original BNIRL %
framework 
is that the inference algorithm expects the demonstration data to be provided in the form of state-action pairs, which %
requires full access to the expert's action record.
This %
assumption is restrictive from a practical point of view as it %
confines the application of the model to settings with laboratory-like conditions that allow a complete monitoring of the expert. %
For this reason, it is important to note that an estimate of the expert's action sequence %
can be %
recovered through BNIRL with the help of an additional sampling stage (omitted in the original paper), %
provided that we know the %
successor state %
reached by the expert %
after each decision. %
For the marginalized %
inference scheme described in this paper, we present the corresponding sampling stage in Section~\ref{sec:actionInference}.\looseness-1

\section{Nonparametric Spatio-Temporal Subgoal Modeling}
\label{sec:ddBNIRL}
In this section, we introduce 
a redesigned inference framework, which, in analogy to BNIRL, we refer to as \textit{distance-dependent Bayesian nonparametric IRL} (ddBNIRL). %
We derive the model by making a series of modifications to the original BNIRL framework %
that address the previously described shortcomings %
on the conceptual level. %
Rethinking each part of the original framework, we begin with a discussion of the commonly used softmax action selection strategy (Equation~\ref{eq:softmaxPolicy}) in the context of subgoal %
inference, which finally leads to a redesign of the subgoal likelihood model %
(Limitation~\hyperref[phantom:lim4]{4}). Next, we focus on the subgoal allocation mechanism itself and introduce two %
closely related model formulations, %
each targeting one of the basic behavior types described in Figure~\ref{fig:crossing}, thereby addressing Limitations~\hyperref[phantom:lim1]{1}, \hyperref[phantom:lim2]{2} and~\hyperref[phantom:lim3]{3}. 
For the time-invariant case, we begin with an intermediate model that introduces a subtle yet important structural modification to the BNIRL framework. In a second step, we generalize that new model to account for the spatial structure of the %
control problem, which finally allows us to extrapolate the expert behavior to unseen situations. As part of this generalization, we present a new state space metric that arises naturally in the context of subgoal inference (see Limitation~\hyperref[phantom:lim2]{2}, second point). Lastly, we tackle the time-varying case and present a variant of the model that %
explicitly considers the temporal aspect of the subgoal problem. A solution to Limitation~\hyperref[phantom:lim5]{5} is discussed later in Section~\ref{sec:predAndInf}.

 In contrast to BNIRL, %
 both presented models can be used likewise for \textit{subgoal extraction} and \textit{action prediction}. Moreover, %
 sticking with the Bayesian methodology, the presented approach provides complete posterior information at all levels.

\subsection{The Subgoal Likelihood Model}
\label{sec:actionLikelihood}
Like many other approaches found in the (I)RL literature, BNIRL exploits a softmax weighting (Equation~\ref{eq:softmaxPolicy}) to transform the Q-values of an optimal policy into a valid subgoal likelihood model.
The softmax action rule has its origin in RL where it is known as the Boltzmann exploration strategy \citep{cesa2017boltzmann,sutton1998reinforcement}, which is commonly applied %
to %
cope with the %
\textit{exploration-exploitation dilemma} \citep{ghavamzadeh2015bayesian}. 
In recent years, however, it has %
also become the de facto standard for describing the %
(imperfect) decision-making strategy of %
an observed demonstrator \citep[see, for example,][]{dimitrakakis2011bayesian,ramachandran2007bayesian,rothkopf2011preference,NIPS2012_4737,neu2007,babes2011apprenticeship}.\looseness-1

In the following paragraphs, 
we focus on the implications of this %
model on the subgoal extraction problem 
and show that it contradicts our intuitive understanding of what characteristics a reasonable subgoal model should have. %
In particular, we argue that the subgoal posterior distribution arising from the %
BNIRL softmax model is of limited use %
for inferring the latent intention of the agent, due to subgoal artifacts caused by the system dynamics that %
cannot be reconciled with the evidence %
provided by the demonstrations. Based on these insights, we  propose an alternative transformation scheme that is more consistent with the subgoal principle. %

\subsubsection{Scale of the Reward Function}
\label{sec:choiceOfBeta}
The first implication of the softmax likelihood model %
concerns the choice of the uncertainty coefficient~$\beta$.  To explain the problem, we consider the thought experiment of an agent located at some state $s$ targeting a particular subgoal~$g$. The likelihood $\softmax(a \given s, g)$ in Equation~\eqref{eq:softmaxPolicy} quantifies the probability that the agent decides for a specific action $a$, based on the corresponding state-action values $Q(\cdot, s \given g)$. Since %
those values are linear in the underlying reward function~$R_g$ (Equation~\ref{eq:Qfunction}), %
the softmax likelihood model %
implies that the expert's %
ability to maximize the long-term reward, reflected by the %
spread of the probability mass in $\softmax(\cdot \given s, g)$, rises with the magnitude $C$ of the assumed subgoal reward (more concentrated probability mass signifies a higher confidence in the action choice). In other words, assuming a higher goal reward virtually increases our level of confidence in the expert, even though the difficulty of the underlying task and the optimal %
policy remain unchanged. Nonetheless, the BNIRL model requires us to readjust the uncertainty coefficient~$\beta$ %
in order to keep both models %
consistent. %
However, as the model provides no %
reference level for the expert's uncertainty across different scenarios, the choice of $\beta$ becomes nontrivial. Yet, the parameter has a significant impact on the granularity of the learned subgoal model as it trades off purposeful goal-oriented behavior against random decisions.

Note that the described effect is not specific to the subgoal reward model in Equation~\eqref{eq:zeroOneReward} but is really a consequence of the softmax transformation in Equation~\eqref{eq:softmaxPolicy}. In fact, the same problem occurs when the model is applied in a regular MDP environment with arbitrary reward function, for example, when the agent is provided an additional constant reward at all states. Clearly, such a %
constant reward provides no further information about the underlying task and should hence not affect the agent's belief about the optimal choice of actions \citep[compare discussion on constant reward functions and transformations of rewards,][]{ng2000algorithms,ng1999policy}. %
Based on these two observations, our intuition tells us that %
we %
seek for a rationality model that is \textit{invariant to affine transformations of the reward signal}, meaning that any two reward functions $R:\mathcal{S}\rightarrow\mathbb{R}$ and $\bar{R}\defeq xR+y$ with $x\in(0,\infty),y\in\mathbb{R}$, should %
give rise to the same intentional representation.
As we shall see in Section~\ref{sec:normalizedLikelihood}, %
this can be achieved by modeling the behavior of %
an agent %
based on the \textit{relative advantages} of actions rather than on their absolute expected returns. %

\subsubsection{Impact of the Transition Dynamics}
\label{sec:impactOfTransitionModel}
The second implication of the softmax likelihood model is %
less immediate and %
inherently %
tied to the dynamics of the system. To explain the problem, we consider a %
scenario where we have a %
precise idea about the potential goals of the expert. 
For our example, we adopt the grid world dynamics described in Section~\ref{sec:proofConcept} and consider a simple upward-directed trajectory of state-action pairs, which we aspire to explain using a single \mbox{(sub-)goal}. The complete setting is depicted in Figure~\ref{fig:Qnormalization}. 

Intuitively, %
the shown demonstration set should %
lead to goals that are located in the upper region of the state space and concentrated around the vertical center line. Moreover, as we move %
away from that %
center line, we expect to observe a smooth decrease in the %
subgoal likelihood, while the rate of the decay should reflect our assumed level of confidence in the expert. As it turns out, the induced BNIRL subgoal posterior distribution, shown in the top row for different values of $\beta$, contradicts this intuition. In particular, we observe that the %
model yields unreasonably high posterior values at the upper border states and corners of the state space, which, according to our intuitive understanding of the problem, cannot be justified by the given demonstration set.\looseness-1

To pin down the cause of this effect, we recall from Equation~\eqref{eq:softmaxPolicy} that the likelihood of an action %
grows with the corresponding Q-value. %
Hence, we need to ask what causes the Q-values of the demonstrated actions to be large when the subgoal is assumed to be located at one of the upper corner/border states of the space. Using Bellman's principle, we can express the optimal Q-function for any subgoal $g$ %
as
\begin{align}
	Q^{*}(s,a \given g) &= R_g(s) + \gamma\,\E_T\big[V^{*}(s' \given g) \given s, a\big] \nonumber \\
	&= R_g(s) + \gamma\,\E_T\big[\E_{\rho^{\pi_g}}[R_g(s'') \given s'] \given s, a \big] \label{eq:improperBellman} \\
	&= R_g(s) + \gamma\,\E_T\big[C \rho^{\pi_g}(g \given s') \given s, a\big] \nonumber \eqkomma
\end{align}
where $V^{*}(s \given g) \defeq \max_{a\in\mathcal{A}} Q^{*}(s ,a \given g) $, $\pi_g (s) \defeq \argmax_{a\in\mathcal{A}} Q^*(s,a \given g)$ is the optimal policy for subgoal $g$, and $C$ is the subgoal reward from Equation~\eqref{eq:zeroOneReward}. Lastly, $\rho^{\pi_g}(s' \given s) \defeq \sum_{t=0}^\infty \gamma^t p_t(s' \given s, \pi_g)$ denotes the (improper) discounted state distribution generated by executing policy $\pi_g$ from the considered initial state~$s$, where $p_t(s' \given s, \pi_g)$ refers to the probability of reaching state $s'$ from state $s$ under policy $\pi_g$ after exactly $t$ steps, which is defined implicitly via the transition model~$T$. 

The outer expectation in Equation~\eqref{eq:improperBellman} accounts for the stochastic transition to the successor state $s'$, while the inner expectation evaluates the expected cumulative reward over all states $s''$ that are reachable from $s'$.  
It is important to note that\einschub{by the construction of the Q-function}only the first move of the agent to state $s'$ %
depends on the choice of action~$a$ whereas all remaining moves (i.e., the argument of the expectation in the last line) are purely determined by the system dynamics and the subgoal policy $\pi_g$. Focusing on that inner part, %
we conclude that, \textit{regardless of the chosen action a}, the Q-values will be large whenever the assumed subgoal induces %
a high state visitation frequency~$\rho^{\pi_g}$ at its own location~$g$. The latter is fulfilled if 
\begin{enumerate}[topsep=-\parskip+1.5ex, itemsep=1ex, parsep=0mm]
\item[\first] the chance of reaching the goal in a small number of steps is high so that the effect of discounting is small and/or
\item[\second] the controlled transition dynamics	%
 $T(s'\given s, \pi_g(s))$ that are induced by the subgoal %
lead to a high chance of hitting the goal frequently.
\end{enumerate}
Note that the first condition implies that the model generally prefers subgoals that are close to the demonstration set\nachschub{a~\mbox{property} that cannot be justified in all cases}. For example, the recording of the demonstrations could have simply ended before the expert was able to reach the goal (Figure~\ref{fig:crossing}). Yet, if desired, this proximity property %
should be more naturally attributed to the subgoal prior model $p_g(g\given \mathbf{s})$. 

Moreover, we observe that the second condition depends primarily on %
the system dynamics~$T$, which can be more or less strongly influenced by the actions of the agent, depending on the scenario. In fact, in a pathological example, $T$ could be even independent of the agent's decisions, meaning that the agent has no control over its state. An example illustrating this extreme case would be a scenario where the agent gets always driven to the same terminal state, regardless of the executed policy. Although it is somewhat pointless %
speak of ``subgoals'' in this context, that terminal state would exhibit a high subgoal likelihood according to the softmax model because the corresponding visitation frequency  would be inevitably large. A %
softened variant of this condition %
can occur at corner/border states (\ie, states in which the agent %
experiences fewer degrees of freedom and which are hence more difficult to leave than others) 
and transition states (\ie, states that must be passed in order to get from certain regions of the space to others), which naturally exhibit an increased visitation frequency due to the characteristics of the environment.

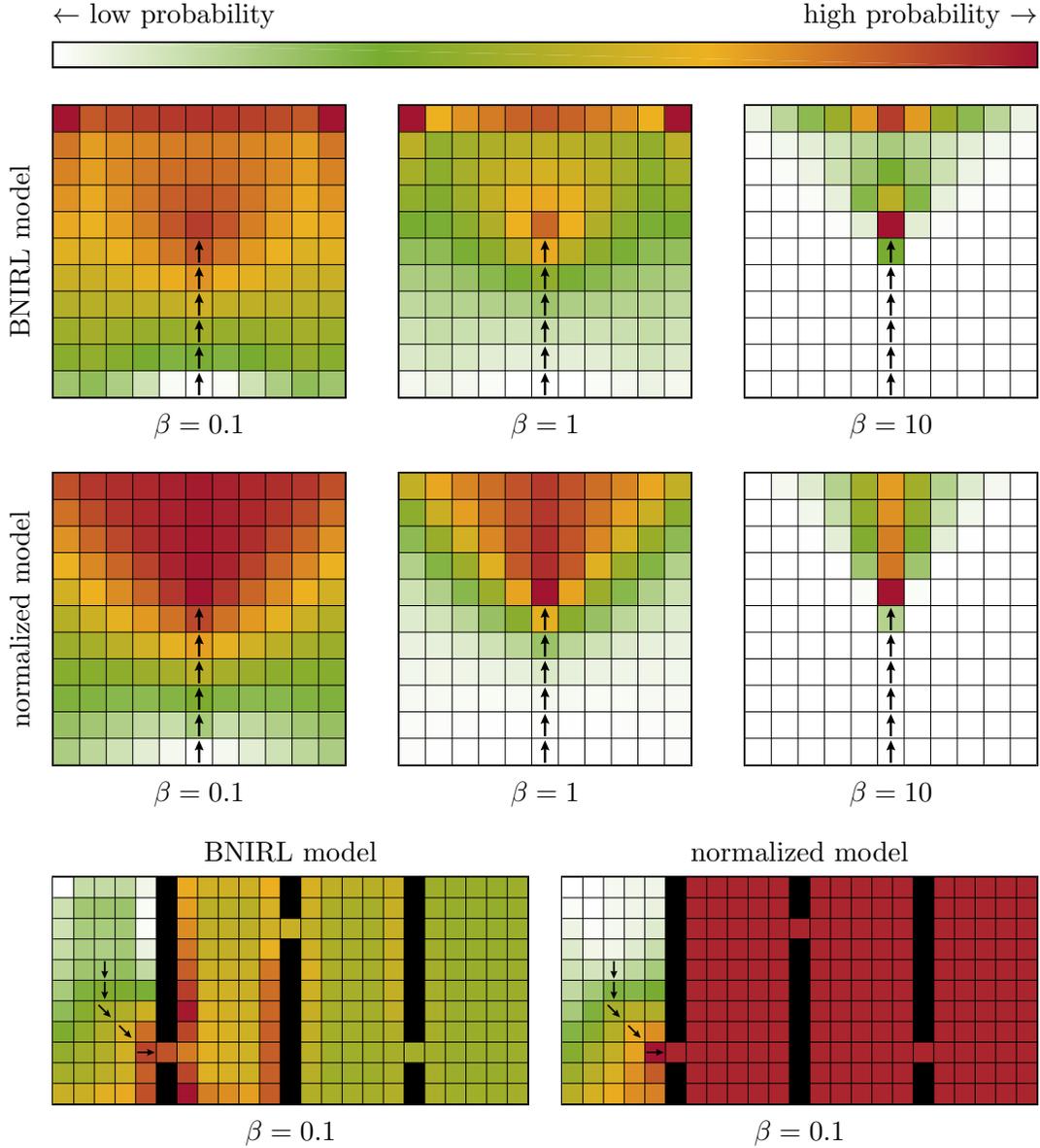
\begin{figure}[p!]
	\centering
	\newlength{\gridworldboxQnorm}%
	\newlength{\lenJMLRqnorm}%
	\newlength{\lenJMLRsepH}%
	\setlength{\gridworldboxQnorm}{1pt}%
	\setlength{\lenJMLRqnorm}{4cm}%
	\setlength{\lenJMLRsepH}{0.7cm}%
	\begin{tikzpicture}[inner sep=0mm]
	
	\node (a1) [draw, line width=\gridworldboxQnorm, label={[label distance=0.2cm]below:{$\beta=0.1$}}] {\includegraphics[width=\lenJMLRqnorm]{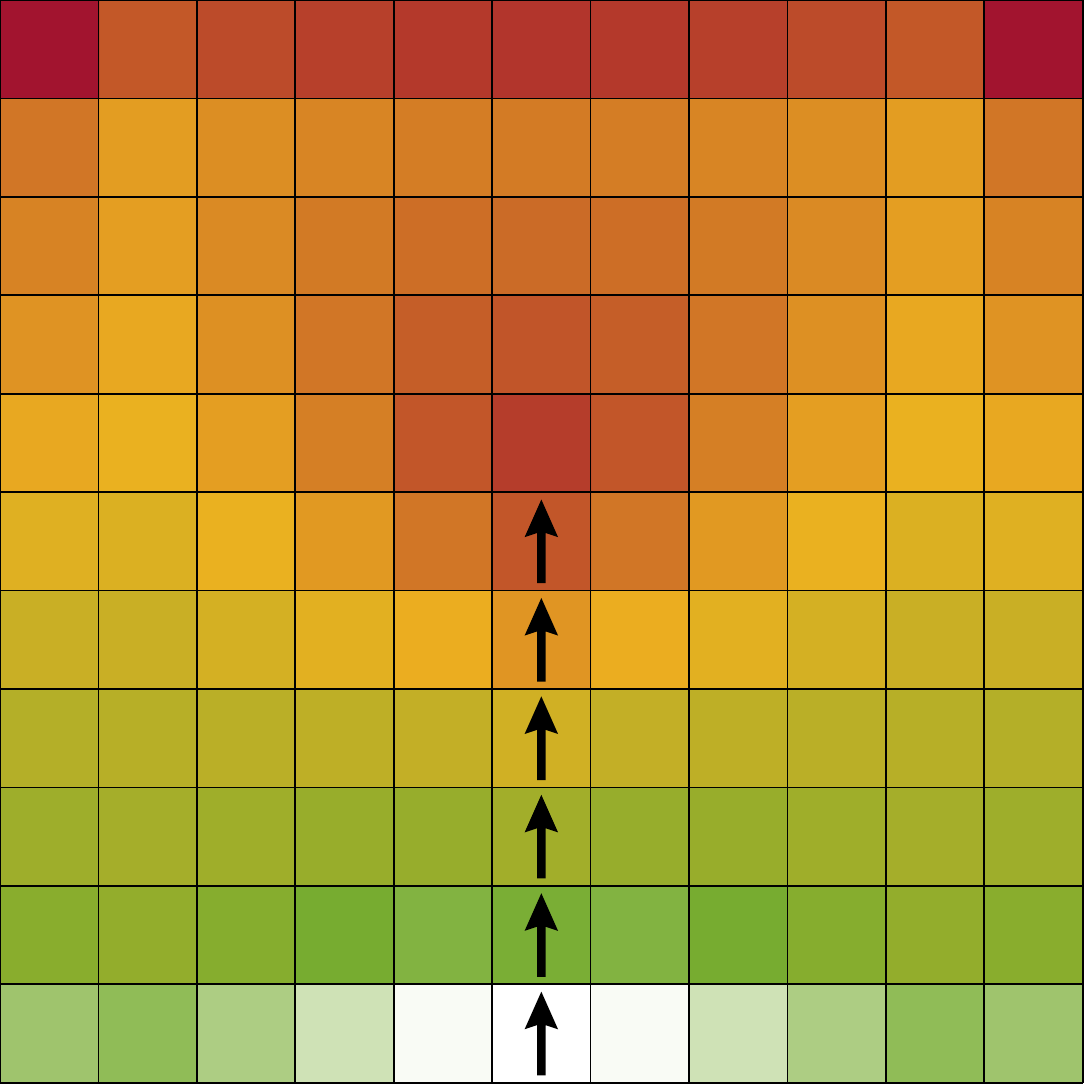}};
	\node (a2) [draw, line width=\gridworldboxQnorm, label={[label distance=0.2cm]below:{$\beta=1$}}, right=\lenJMLRsepH of a1] {\includegraphics[width=\lenJMLRqnorm]{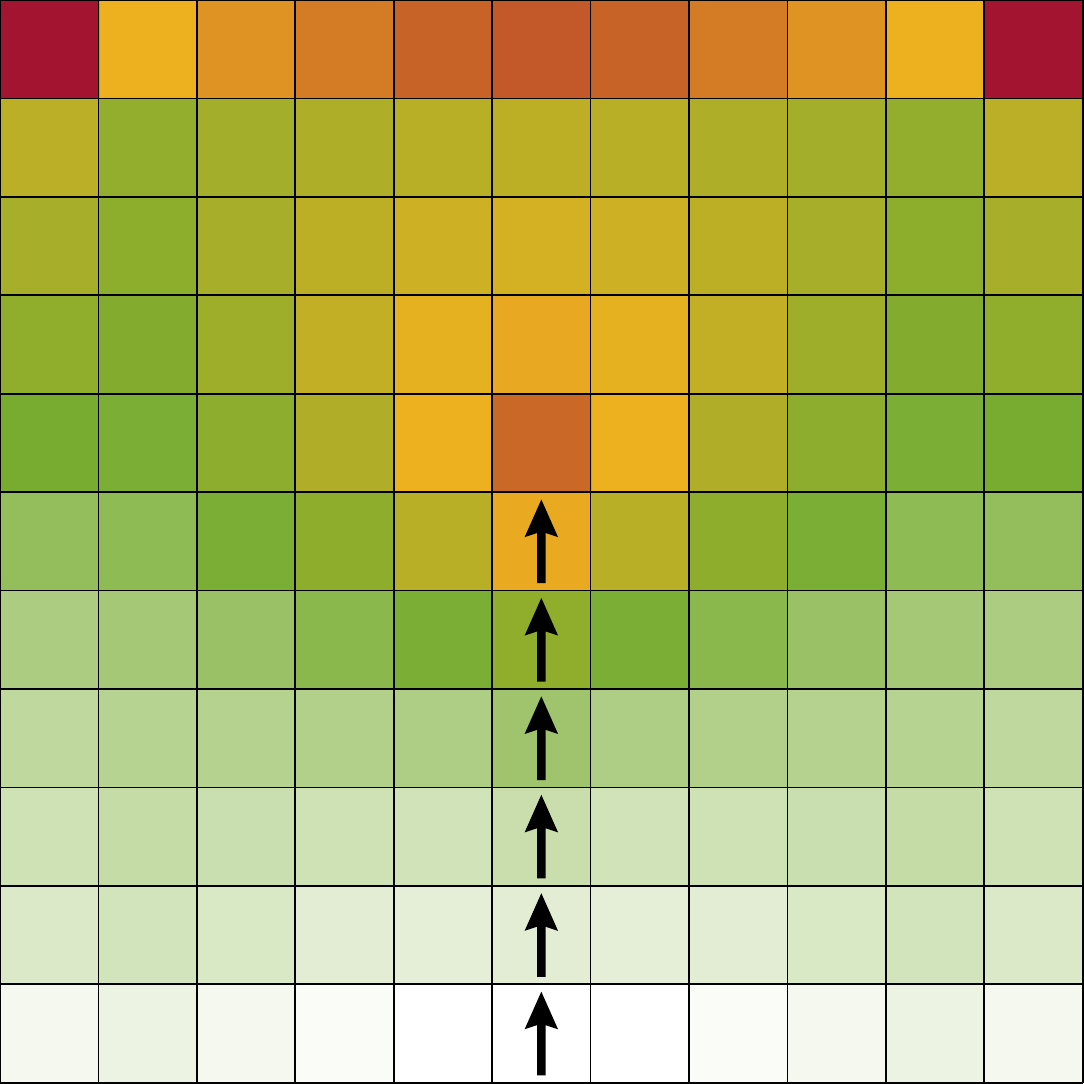}};
	\node (a3) [draw, line width=\gridworldboxQnorm, label={[label distance=0.2cm]below:{$\beta=10$}}, right=\lenJMLRsepH of a2] {\includegraphics[width=\lenJMLRqnorm]{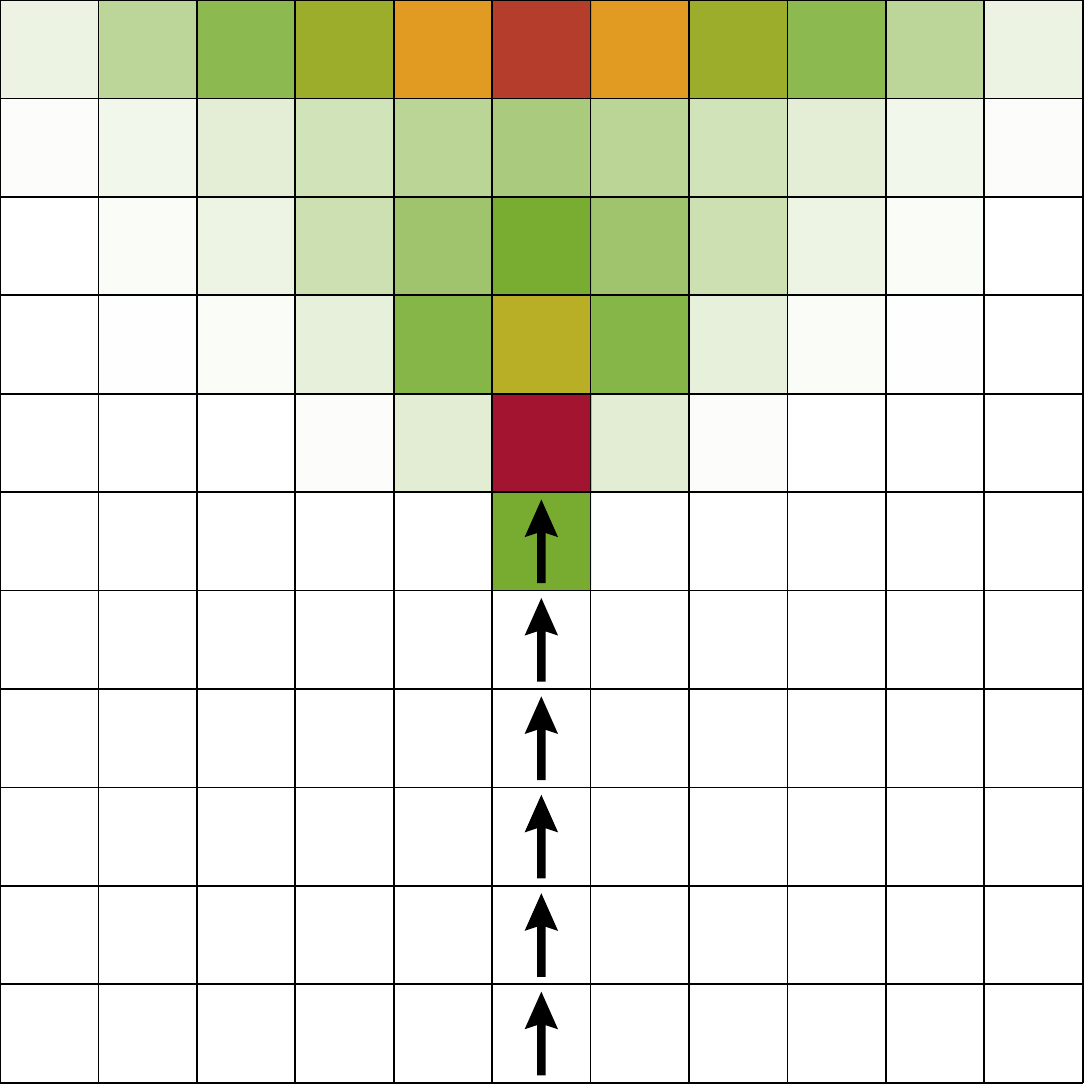}};
	\node [rotate=90, left=0.3cm of a1, anchor=south] {BNIRL model};
	
	\node (b1) [draw, line width=\gridworldboxQnorm, label={[label distance=0.2cm]below:{$\beta=0.1$}}, below=1cm of a1] {\includegraphics[height=\lenJMLRqnorm]{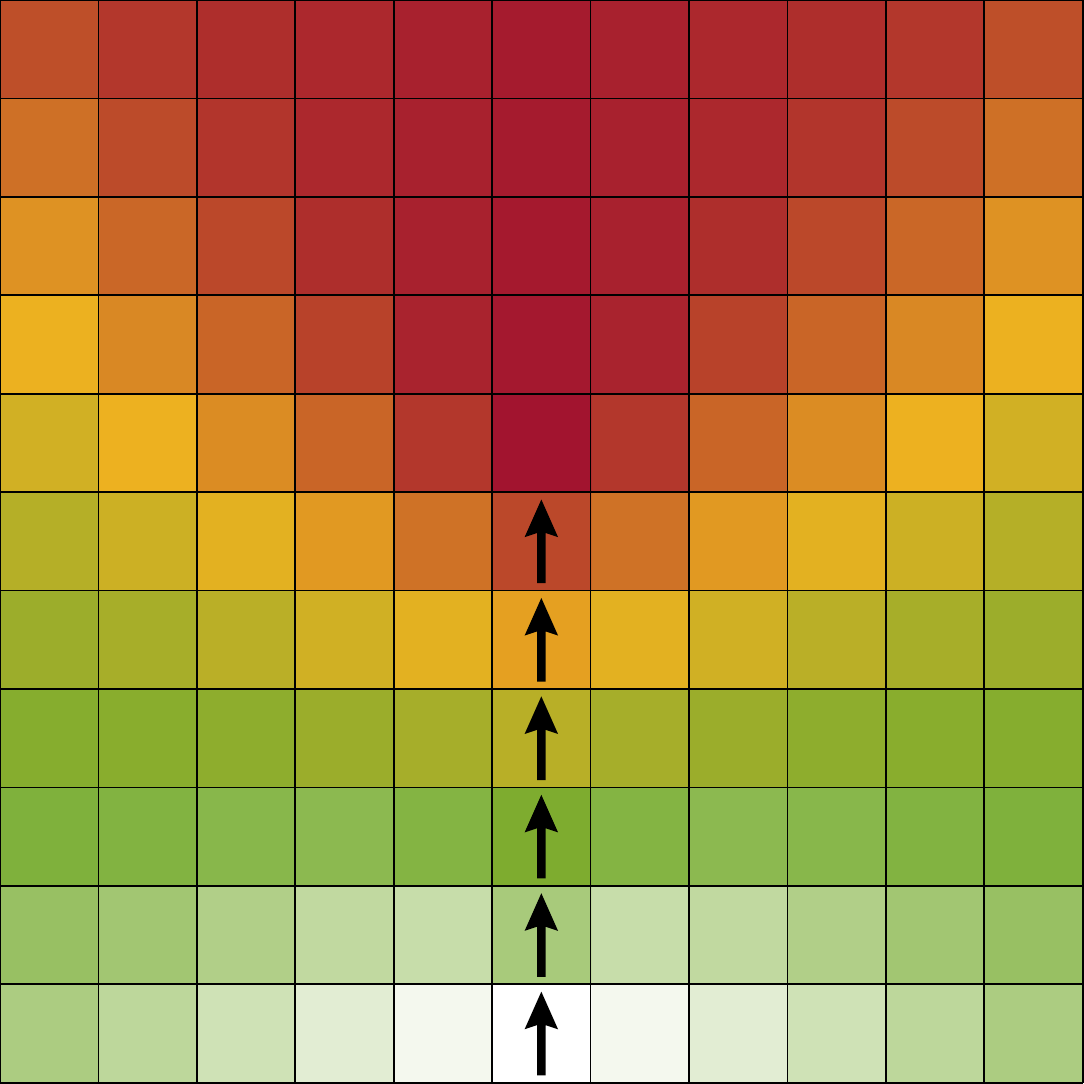}};
	\node (b2) [draw, line width=\gridworldboxQnorm, label={[label distance=0.2cm]below:{$\beta=1$}}, right=\lenJMLRsepH of b1] {\includegraphics[height=\lenJMLRqnorm]{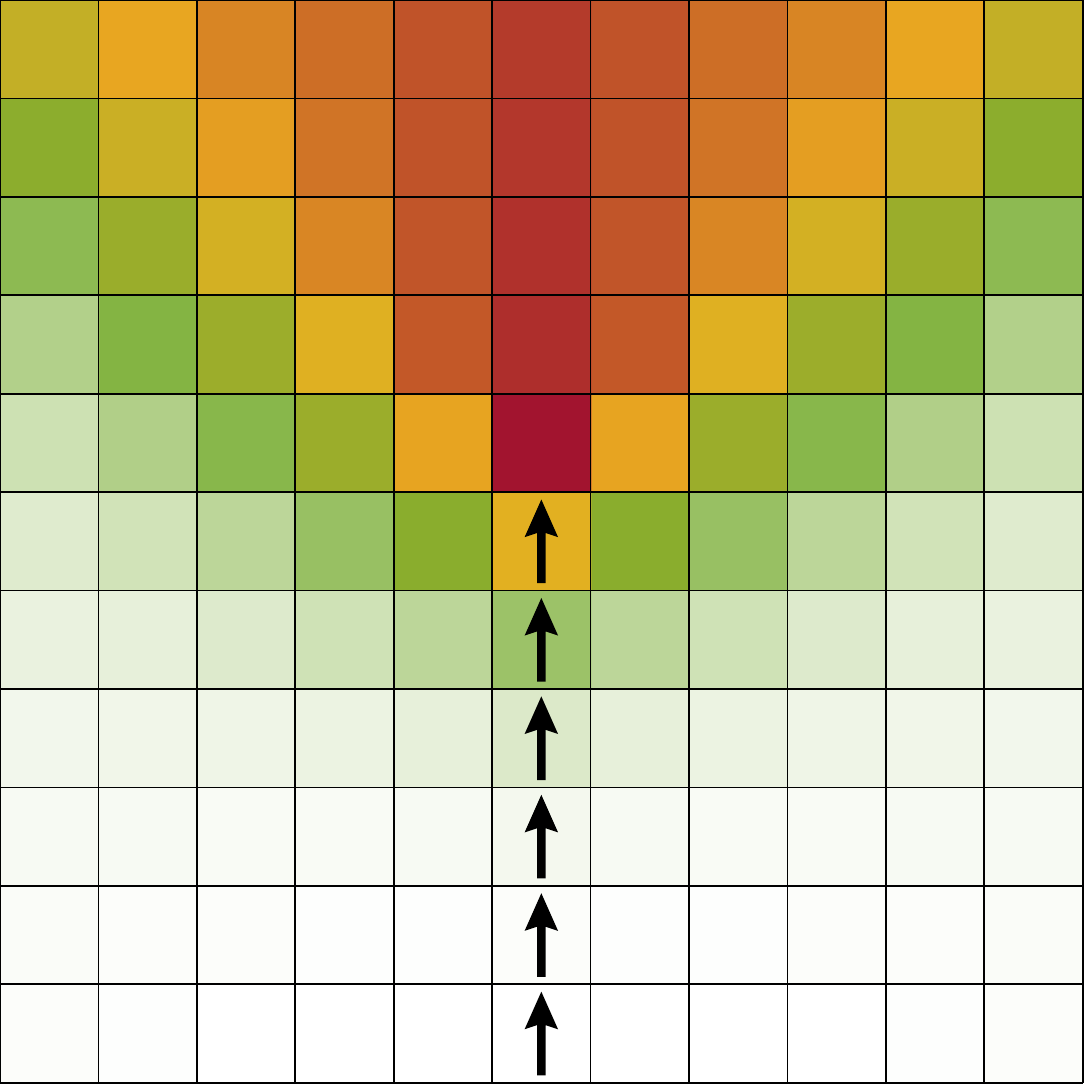}};
	\node (b3) [draw, line width=\gridworldboxQnorm, label={[label distance=0.2cm]below:{$\beta=10$}}, right=\lenJMLRsepH of b2] {\includegraphics[height=\lenJMLRqnorm]{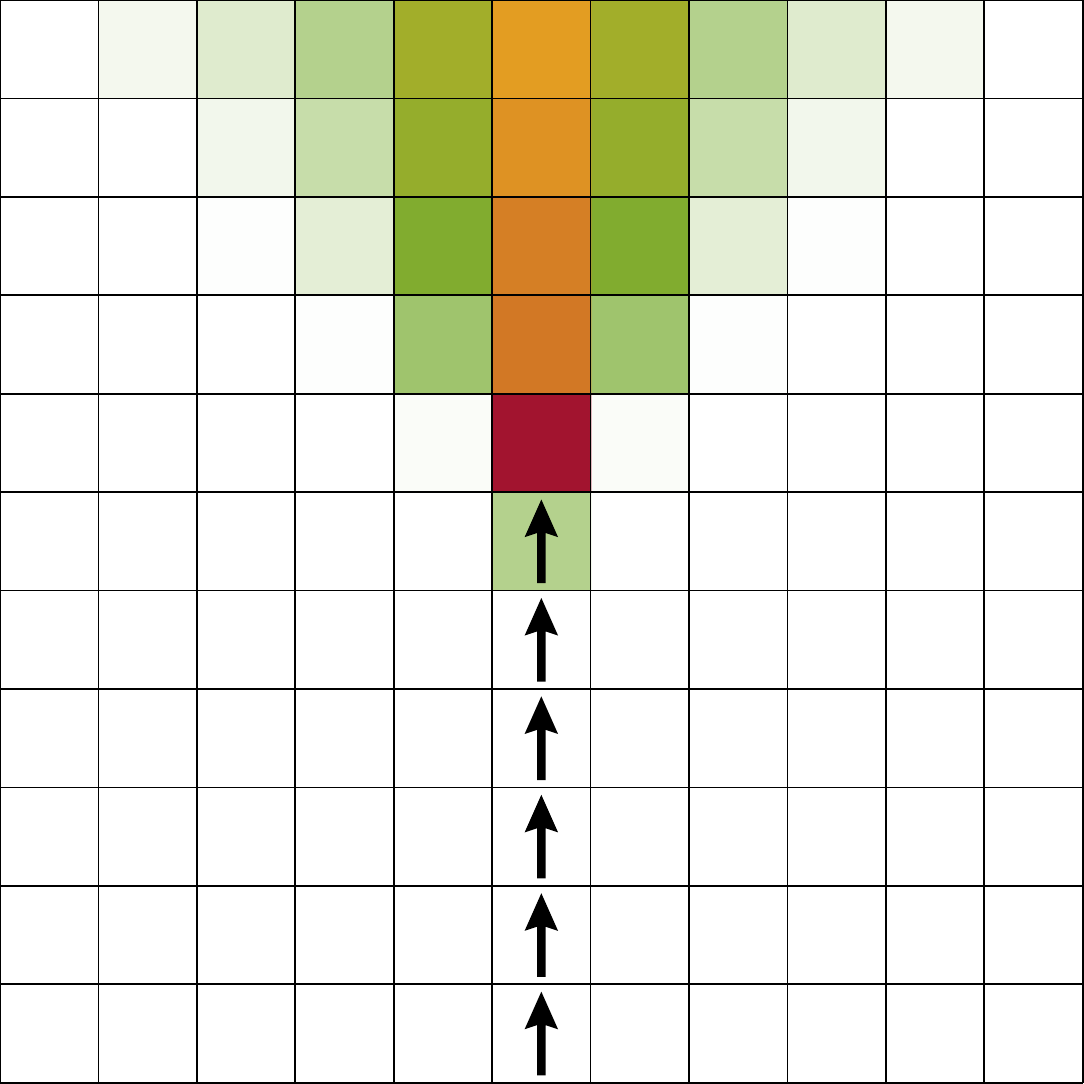}};
	\node [rotate=90, left=0.3cm of b1, anchor=south] {normalized model};
	
	\draw let \p1 = ($(a1.west) - (a3.east)$), \n1 = {veclen(\x1,\y1)} in ([yshift=0.85cm]a1.north west) to node (leg) [draw, line width=2pt, anchor=center, rotate=-90] {\includegraphics[height=\n1-2pt, width=0.3cm]{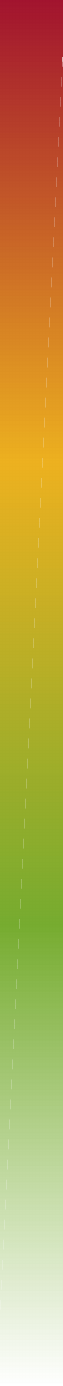}} ([yshift=0.5cm]a3.north east) ; 
	
	\node [above =1ex of leg.south west, anchor=south west] {$\leftarrow$ low probability};
	\node [above =1ex of leg.north west, anchor=south east] {high probability $\rightarrow$};
	
	\node (w1) [draw, line width=\gridworldboxQnorm, label={[label distance=0.2cm]below:{$\beta=0.1$}}, below=1.5cm of b1.south west, anchor=north west] {\includegraphics[width=6.5cm]{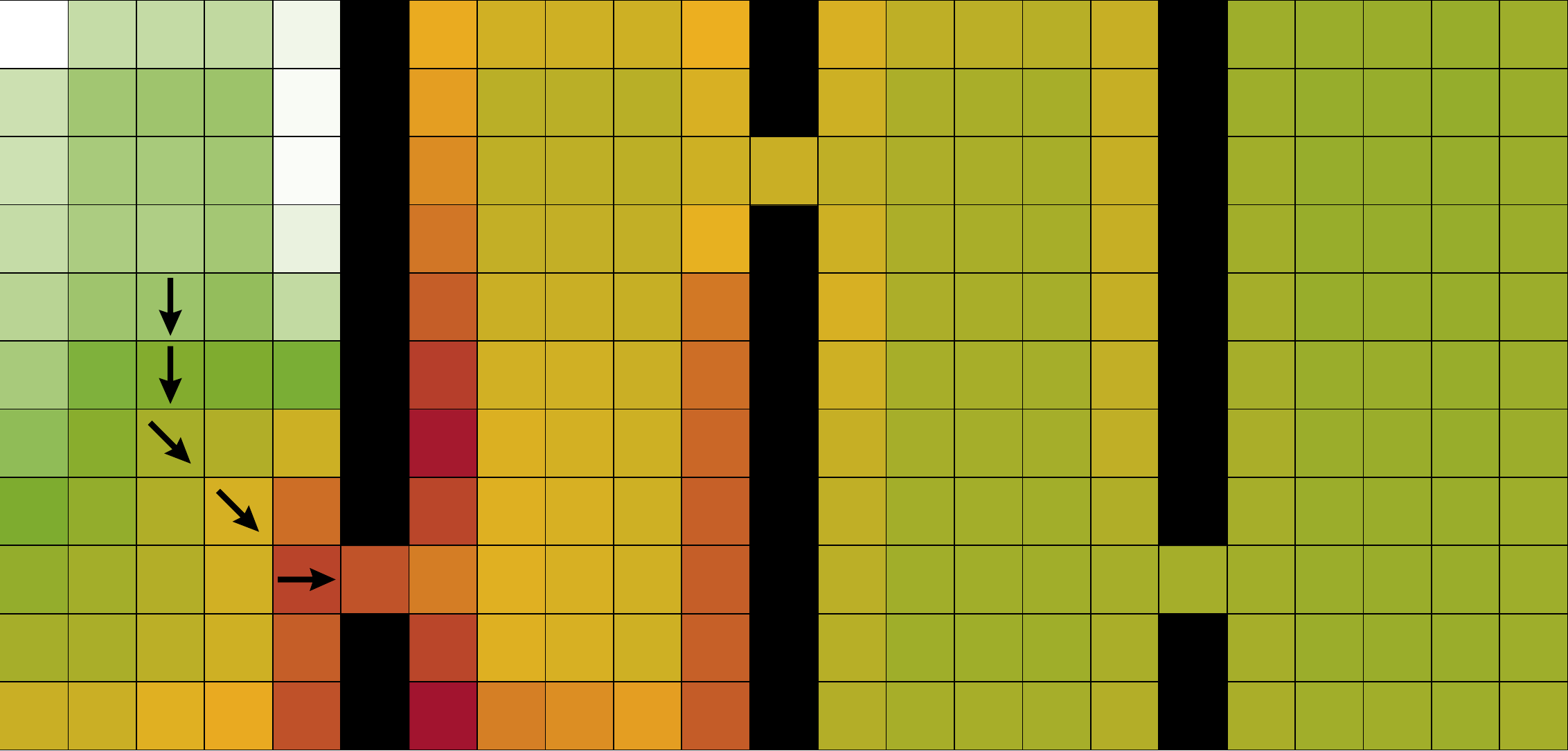}};
	\node (w2) [draw, line width=\gridworldboxQnorm, label={[label distance=0.2cm]below:{$\beta=0.1$}}, below=1.5cm of b3.south east, anchor=north east] {\includegraphics[width=6.5cm]{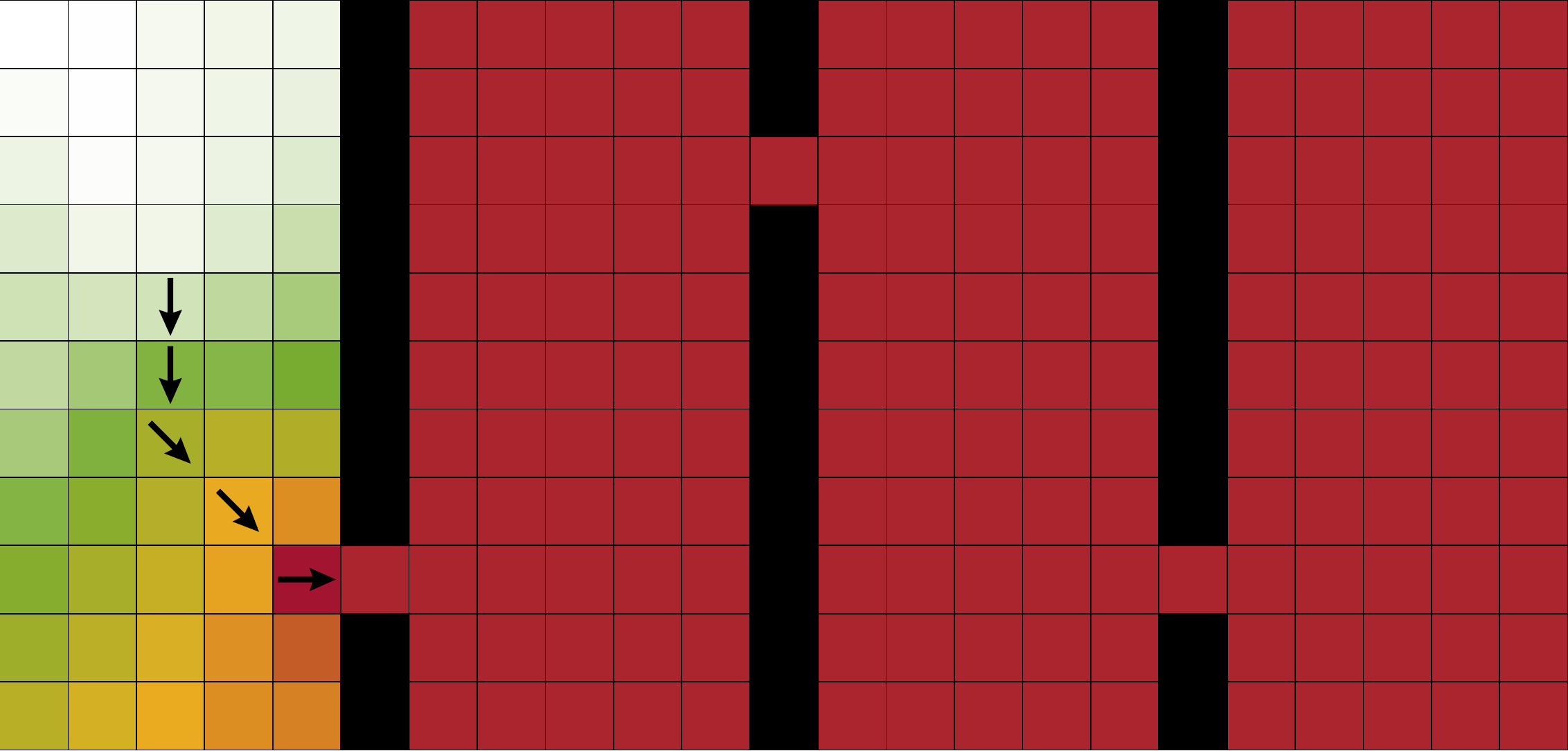}};
	\node [above=0.2cm of w1] {BNIRL model};
	\node [above=0.2cm of w2] {normalized model};

	\end{tikzpicture}
	\caption{Comparison of the subgoal posterior distributions induced by the original BNIRL likelihood model and by the proposed normalized model, based on the grid world dynamics described in Section~\ref{sec:proofConcept} and a uniform subgoal prior distribution $p_g$. The range of the shown color scheme is to be understood per subfigure. Black squares indicate wall states. The BNIRL likelihood model yields unreasonably high subgoal posterior mass at the border states and corners of the state space (due to locally increased state visitation probabilities arising from wall reflections) as well as at trajectory endings (caused by the implicit proximity property of the model)\nachschub{see Section~\ref{sec:impactOfTransitionModel} for details}. %
		Both effects are mitigated by the proposed normalized likelihood model, which %
		describes the action-selection process of the agent using {relative advantages} of actions instead of absolute returns.}
	\label{fig:Qnormalization}
\end{figure}

In our example in Figure~\ref{fig:Qnormalization}, we can observe the %
symptoms of both described conditions clearly. In particular, for an upward-directed policy as it is implied by the shown demonstration set, %
the induced state visitation distribution %
exhibits increased values at exactly the aforementioned border and corner states (due to the %
reflections occurring to the agent when hitting the state space boundary) as well as close to the trajectory ending (caused by the proximity condition). %

\subsubsection{The Normalized Likelihood Model}
\label{sec:normalizedLikelihood}
To address these problems, we modify the likelihood model using a rescaling of the involved Q-values. Let ${Q}^\wedge(s \given g)$ and ${Q}^\vee(s \given g)$ denote the maximum and minimum Q-values at state~$s$ for subgoal $g$, \ie, ${Q}^\wedge(s \given g) \defeq \max_{a\in\mathcal{A}} Q^*(s,a \given g)$ and ${Q}^\vee(s \given g) \defeq \min_{a\in\mathcal{A}}Q^*(s,a \given g)$. We then define the normalized state-action value function $Q^\bullet:\mathcal{S}\times\mathcal{A}\times\mathcal{S}\rightarrow[0,1]$ as
\begin{equation}
Q^\bullet(s, a \given  g) \defeq
\begin{cases}
\frac{Q^*(s,a \given g) - {Q}^\vee(s \given g)}{{Q}^\wedge(s \given g) - {Q}^\vee(s \given g)} & \text{if} \ \ {Q}^\wedge(s \given g) \neq {Q}^\vee(s \given g) \eqkomma \\
\epsilon & \text{otherwise} \eqkomma
\end{cases}
\label{eq:normalizedQ}
\end{equation}
where $\epsilon\in(0,1]$ is an arbitrary constant that is canceled out in Equation~\eqref{eq:normalizedLikelihood}.
In contrast to the Bellman state-action value function $Q^*$, %
which %
quantifies the expected return of an action, the normalized function $Q^\bullet$ assesses the return of that action in relation to the returns of all other actions. This concept is similar to that of the advantage function~\citep{baird1994reinforcement} with the important difference that the values returned by $Q^\bullet$ are normalized to the range $[0,1]$ and thus serve as an indicator for the \textit{relative} %
quality of actions. Accordingly, the values can be interpreted as \textit{relative advantages} (\ie, relative to the maximum possible advantage among all actions).
The normalized subgoal likelihood model is then constructed analogously to the BNIRL likelihood model,
\begin{equation}
\softmax^\bullet(a_d \given s_d, g_{\tilde{z}_d}) \propto \exp\big\{\beta Q^\bullet(s_d, a_d \given g_{\tilde{z}_d})\big\} \eqpunkt %
\label{eq:normalizedLikelihood}
\end{equation}
The key property of %
this model is that it is invariant to affine transformations of the reward function, as summarized by the following proposition.

\vspace{0.5\baselineskip}
\begin{proposition}[Affine Invariance]
Consider an MDP with reward function $R:\mathcal{S}\rightarrow\mathbb{R}$ %
and let $Q^*(s,a \given R)$ denote the corresponding optimal state-action value function. For the corresponding normalized function $Q^\bullet$ it holds that $Q^\bullet(s,a \given R) = Q^\bullet(s,a \given xR+y) \ \forall x\in(0,\infty),y\in\mathbb{R}, s\in\mathcal{S}, a\in\mathcal{A}$. Hence, the subgoal likelihood model in Equation~\eqref{eq:normalizedLikelihood} is invariant to affine transformations of $R$.
\end{proposition}

\begin{proof}
Due to the linear dependence of $Q^*$ on the reward function $R$ (Equation~\ref{eq:Qfunction}) it holds that $Q^*(s,a \given xR+y) = xQ^*(s,a \given R) + \frac{y}{1-\gamma}$. Using this relationship in Equation~\eqref{eq:normalizedQ}, it follows immediately that $Q^\bullet(s,a \given R) = Q^\bullet(s,a \given xR+y)$.
\end{proof}

\noindent Using the proposed likelihood model offers several advantages. First of all, it enables a %
more generic choice of the uncertainty coefficient~$\beta$ (Section~\ref{sec:choiceOfBeta}). This is because the returned Q$^\bullet$-values lie in the fixed range [0,1], where 0 always indicates the lowest and~1 indicates the highest confidence. For example, setting $\beta=\log(\beta')$ for some $\beta'\in(0,\infty)$ always %
corresponds to the assumption that the expert chooses the optimal action with a probability that is $\beta'$~times higher than the probability of choosing the least favorable action, irrespective of the underlying system model. 

Moreover, as the results in Figure~\ref{fig:Qnormalization} reveal, the induced subgoal posterior distribution is notably closer to our expectation. %
The reason for this is twofold: first, a likelihood computation based on relative advantages mitigates the %
influence of the transition dynamics discussed in Section~\ref{sec:impactOfTransitionModel}. This is because the described cumulation effect of the state visitation distribution $\rho^{\pi_g}$ (Equation~\ref{eq:improperBellman}) %
is present %
in the returns of all actions and is thus %
reduced through the proposed normalization. For instance, if the agent in our grid world follows a policy that is all upward directed (as shown in the example), the induced state visitation distribution exhibits increased values at the upper border states of the world, even if we manipulated the \textit{first} action of the agent (as considered in the Bellman Q-function). Accordingly, the original model would indicate an increased subgoal likelihood at those states. %
The normalized model, by contrast, is less affected as it constructs the likelihood by considering the increased visitation frequencies \textit{relative} to each other. 

Second, since the normalization %
diminishes the effect of the discounting, the subgoal posterior distribution is less concentrated around the trajectory ending and shows significant mass along the extrapolated path of the agent. %
This property allows us to identify %
far located states as potential goal locations, which adds more flexibility to the inferred subgoal constellation (compare Figure~\ref{fig:globalVSlocal}). 
As an illustrating example, consider the scenario shown in the bottom part of Figure~\ref{fig:Qnormalization}. We observe that the normalized model assigns high posterior mass to all states in the %
right three corridors since any subgoal located in those corridors explains the demonstration set equally well. Here, the difference between the two models is even more pronounced because the transition dynamics have a strong impact on the agent behavior due to the added wall states. %
For further details, we refer to Section~\ref{sec:proofConcept}, where we provide additional insights into the subgoal inference mechanism.

\subsection{Modeling Time-Invariant Intentions}
\label{sec:ddBNIRL-S}
With our redesigned likelihood model, we now focus on the partitioning structure of the model. Herein, we first consider the case where the intentions of the agent are constant with respect to time. As explained in Limitation~\hyperref[phantom:lim3]{3}, this setting is consistent with the standard MDP formalism in the sense that the optimal policy for the considered task can be described in the form of a state-to-action mapping. 

As a first step, to account for this relation, we establish a link between the model partitioning structure and the underlying system state space by replacing
the demonstration-based indicators $\mathbf{\tilde{z}}=\{\tilde{z}_d\in\mathbb{N}\}_{d=1}^{D}$ with a new set of variables $\mathbf{z}\defeq\{{z}_i\in\mathbb{N}\}_{i=1}^{|\mathcal{S}|}$. Unlike $\mathbf{\tilde{z}}$, %
these new indicators do not operate directly on the data 
but %
are instead tied to the elements in~$\mathcal{S}$. %
Although they formally represent a new type of variable, %
we can still imagine that their distribution follows a CRP. This yields an intermediate model of the form
\begin{equation*}
\p(\mathbf{a},\mathbf{z},\mathcal{G} \given \mathbf{s}) = \p(\mathbf{z}) \prod_{k=1}^{\infty} \p_g(g_k \given \mathbf{s}) \prod_{d=1}^D \softmax^\bullet(a_d \given s_d, g_{\mathbf{z}_{s_d}}) \eqkomma %
\end{equation*}
whose structure is illustrated in Figure~\ref{fig:graphicalModelB}. To see the difference to Equation~\eqref{eq:BNIRLjoint}, notice the way %
the subgoals are indexed in this model. %

The intermediate model makes it possible to reason about the policy (or, more suggestively, the underlying state-to-action rule %
approximated by the expert) at visited %
parts of the state space. %
Yet, the model is unable to extrapolate %
the gathered information to unvisited states, for the reasons explained in Section~\ref{sec:limitations}.
This problem can be solved by replacing the exchangeable prior distribution over subgoal assignments induced by the CRP with a non-exchangeable one, %
in order to account explicitly for the covariate state information %
contained in the demonstration set.
Based on our insights from 
Bayesian policy recognition \citep[][]{sosic2018pami}, we use the distance-dependent Chinese restaurant process \citep[ddCRP,][]{blei2011distance} for this purpose, which allows a very intuitive handling of %
the state context, as explained below. %
For alternatives, %
we point to the survey paper by \cite{foti2015survey}.

In contrast to the CRP, which assigns states to partitions, the ddCRP assigns states to other states, based on their pairwise distances. 
These ``to-state'' assignments are described by a set of indicators $\mathbf{c}\defeq\{c_i\in\mathcal{S}\}_{i=1}^{|\mathcal{S}|}$ %
with prior distribution $\p(\mathbf{c})=\prod_{i=1}^{|\mathcal{S}|}\p(c_i)$,
\begin{equation}
	\p(c_i=j) = \begin{cases} \nu & \text{if } i=j\eqkomma \\ f(\Delta_{i,j}) &  \text{otherwise}\eqkomma \end{cases}
	\label{eq:ddCRPprior}
\end{equation}
for $i,j\in\mathcal{S}$. Herein, $\nu\in[0,\infty)$ is called the self-link parameter of the process, $\Delta_{i,j}$ denotes the distance from state~$i$ to state $j$, and $f:[0,\infty) \rightarrow [0, \infty)$ is a monotone decreasing score function. Note that the distances $\{\Delta_{i,j}\}$ can be obtained via a suitable metric defined on the state space, which may be furthermore used for calibrating the %
score function~$f$ (see subsequent section).
The state partitioning structure itself is then determined by the connected components of the induced ddCRP graph (Figure~\ref{fig:ddCRPgraph}). %
Our  joint distribution, visualized in Figure~\ref{fig:graphicalModelC}, thus reads as
\begin{equation}
\p(\mathbf{a},\mathbf{c},\mathcal{G} \given \mathbf{s}) = \p(\mathbf{c}) \prod_{k=1}^{\infty}\p_g(g_k \given \mathbf{s}) \prod_{d=1}^D \softmax^\bullet(a_d \given s_d, g_{\mathbf{z}(\mathbf{c})|_{s_d}}) \eqkomma \hspace{-0.5ex}
\label{eq:ddBNIRLSjoint}
\end{equation}
where $\mathbf{z}(\mathbf{c})|_s$ denotes the subgoal label of state $s$ arising from the considered indicator set~$\mathbf{c}$. %
In order to highlight the state dependence of the underlying subgoal mechanism, we refer to this model as \ddBNIRLS.

\subsubsection{The Canonical State Metric for Spatial Subgoal Modeling}
\label{sec:canonicalStateMetric}
The use of the ddCRP as a prior model for the state partitioning in Equation~\eqref{eq:ddBNIRLSjoint} %
inevitably requires %
some notion of distance %
between any two states of the system, in order to compute the involved function scores $\{f(\Delta_{i,j})\}$. When no such distances %
are provided by the problem setting (see Limitation~\hyperref[phantom:lim2]{2}, second point), a suitable \mbox{(quasi-)metric} can be derived from the transition dynamics of the system, which turns out to be %
the canonical choice %
for the \ddBNIRLS\ model.
Consider the Markov chain governing the state %
process $\{s_{t=n}\}_{n=1}^\infty$ of an agent for some specific policy~$\pi$. For any ordered pair of states $(i,j)$, the chain naturally induces a value 
$\hitting^\pi_{i\rightarrow j}$, called a \textit{hitting time} \citep{taylor2014introduction,tewari2008optimistic}, which represents the expected number of steps required until the state process, initialized at $i$, eventually reaches state $j$ for the first time,
\begin{equation*}
\hitting^\pi_{i\rightarrow j} \defeq \E \big[ \min \{n\in\mathbb{N} : s_{t=n} = j \} \given s_0 = i, \pi \big] \eqpunkt
\end{equation*}
In the context of our subgoal problem, the natural quasi-metric to measure the directed distance between two states~$i$ and $j$ is thus given by the time it takes to reach the goal state $j$ from the starting state $i$ under the corresponding optimal subgoal policy $\pi_j (s) = \argmax_{a\in\mathcal{A}} Q^*(s,a \given j)$, \ie, $\Delta_{i,j} \defeq \hitting^{\pi_j}_{i\rightarrow j}$.
For \ddBNIRLS\  (as well as for the waypoint method in BNIRL), %
this choice is particularly appealing since the subgoal policies $\{\pi_j\}$ are %
already available within the inference procedure after the state-action values have been computed for the likelihood model (more on this in Section~\ref{sec:complexity}). The corresponding distances $\{\Delta_{i,j}\}$ can be obtained efficiently in a single policy evaluation step since $\Delta_{i,j}$ corresponds to the optimal (negative) expected return at the starting state~$i$ %
for the special setting where 
the respective target state $j$ is made absorbing with zero reward while all other states are assigned a reward of~$-1$.\looseness-1

\subsubsection{Choice of the Score Function}
From Equation~\eqref{eq:ddCRPprior} it is evident that the ddCRP model %
favors partitioning structures %
that result from the connection of nearby states. In the context of the subgoal problem, this property translates to the prior assumption that, most likely, each subgoal is approached by the expert from only one specific localized region in the system state space. While this assumption may be reasonable for some tasks, other tasks require that certain target states be approached %
more than one time, from different regions in the system state space. In such cases, it is beneficial if the model can reuse the same subgoal in various contexts, in order to obtain a more efficient task encoding %
(Figure~\ref{fig:globalVSlocal}). 

From a mathematical point of view, the prerequisite for learning such encodings is that the score function $f$ does not shrink to zero at large distance values, so that there remains a non-zero probability of connecting states %
that are far apart from each other. This can be achieved, for example, by representing %
$f$ as a convex combination of a monotone decreasing zero-approaching function $\bar{f}:[0,\infty)\rightarrow[0,\infty)$  and some constant offset~$\kappa\in(0,1]$,
\begin{equation*}
	f(\Delta) = (1-\kappa)\bar{f}(\Delta) + \kappa\eqkomma
\end{equation*}
where $\bar{f}$ is chosen, \eg, as a radial basis function \citep{sosic2018pami}. Note that, in order to implement a desired degree of locality in the model, the scale of the decay function $f$ (or $\bar{f}$, respectively) can be further calibrated %
based on the quantiles of 
the distribution of the given distances~$\{\Delta_{i,j}\}$.

\subsection{Modeling Time-Varying Intentions}
\label{sec:ddBNIRL-T}
For the case of changing expert intentions, we %
need to keep the flexibility of BNIRL to %
select a new subgoal at each decision instant, %
instead of restricting our policy %
to target a unique subgoal per state (Figure~\ref{fig:crossing}). Hence, we %
retain the basic BNIRL structure in this case
and %
define the subgoal allocation mechanism using a set of data-related indicator variables. However, in contrast to BNIRL, which makes 
no assumptions %
about the temporal relationship of the subgoals %
and %
thus allows %
arbitrary changes of the expert's intentions (Section~\ref{sec:BNIRLrefresher}), we design our joint distribution in a way that favors smooth action plans in which the expert persistently follows a subgoal over an extended period of time. Again, we can %
make use of the ddCRP properties to encode the underlying smoothness assumption, but this time %
using a %
score function defined on the %
\textit{temporal distance} %
between demonstration pairs. %
For this purpose, we require an additional piece of information, namely the unique timestamp of each demonstration example. %
Accordingly, we need to assume that our data set is of the form $\widetilde{\mathcal{D}}\defeq\{(s_d,a_d,t_d)\}_{d=1}^D$, where $t_d$ denotes the recording time of the $d$th demonstration pair~$(s_d,a_d)$.\footnote{Note that %
the timestamps $\{t_d\}$ are naturally available if the demonstrations are recorded in trajectory form, where we observe several consecutive state-action pairs. In fact, the temporal information of the data is also required for the waypoint method to work (Limitation~\hyperref[phantom:lim2]{2}), even though the authors %
of BNIRL formally assume to have access to the reduced data set of state-action pairs only.} %

The %
prior distribution over data partitionings can then be written as
$\p(\mathbf{\tilde{c}})=\prod_{d=1}^{D}\p(\tilde{c}_d)$,
\begin{equation*}
\p(\tilde{c}_d=d') \propto \begin{cases} \nu & \text{if } d=d'\eqkomma \\ f(\widetilde{\Delta}_{d,d'}) &  \text{otherwise}\eqkomma \end{cases}
\end{equation*}
where the indices $d,d'\in\{1,\ldots,D\}$ range over the size of the demonstration set. Herein, %
$\widetilde{\Delta}_{d,d'} \defeq \lvert t_d - t_{d'} \rvert$ denotes the temporal distance between the data points $d$ and $d'$. %
As before, we use the \mbox{``$\sim$''-notation} to distinguish the data-related partitioning variables $\mathbf{\tilde{c}}$,~$\mathbf{\tilde{z}}$ and distances $\{\widetilde{\Delta}_{d,d'}\}$ from their state-space-related counterparts $\mathbf{c}$, $\mathbf{z}$ and $\{{\Delta}_{i,j}\}$ used in %
\ddBNIRLS. %
Note, however, that the score function $f$ is independent of the underlying model type and may be chosen as described in Section~\ref{sec:canonicalStateMetric}, %
with a scale calibrated to the duration of the demonstrated task.

With that, we obtain our temporal subgoal model as
\begin{equation}
\p(\mathbf{a},\mathbf{\tilde{c}},\mathcal{G} \given \mathbf{s}) = \p(\mathbf{\tilde{c}}) \prod_{k=1}^{\infty}\p_g(g_k \given \mathbf{s}) \prod_{d=1}^D \softmax^\bullet(a_d \given s_d, g_{\mathbf{\tilde{z}}(\mathbf{\tilde{c}})|_{d}}) \eqkomma \hspace{-0.5ex}
\label{eq:ddBNIRLTjoint}
\end{equation}
where $\mathbf{\tilde{z}}(\mathbf{\tilde{c}})|_{d}$ refers to the subgoal label of the $d$th demonstration pair induced by the given assignment $\mathbf{\tilde{c}}$. Analogous to our spatial subgoal model, we refer to this model as \ddBNIRLT. The structural differences between all models can be seen from Figure~\ref{fig:graphicalModels}.

\subsubsection{Relationship to BNIRL}
\label{sec:relationToBNIRL}
Since the distance-dependent CRP contains the classical CRP as a special case for a specific choice of distance metric and score function \citep{blei2011distance}, the \ddBNIRLT\  model can be considered a strict generalization of the original BNIRL framework %
(neglecting the likelihood normalization in Section~\ref{sec:actionLikelihood}). In the same way, \ddBNIRLS\ generalizes the intermediate model presented in Section~\ref{sec:ddBNIRL-S} (Figure~\ref{fig:graphicalModels}). However, although the BNIRL model can be recovered from ddBNIRL, it is important to note that the sampling mechanisms of both frameworks are fundamentally different. Whereas in BNIRL the subgoal assignments are sampled %
directly, the clustering structure in ddBNIRL is defined \textit{implicitly} via the assignment variables $\mathbf{c}$ and $\mathbf{\tilde{c}}$, respectively. As explained by \cite{blei2011distance}, this has the effect that the Markov chain governing the Gibbs sampler mixes significantly faster because %
several cluster assignments %
can be altered in a single step, %
which effectively realizes a blocked Gibbs %
sampler %
\citep{roberts1997updating}. %

\subsection{Static versus Dynamic Subgoal Allocation}
\label{sec:staticVSdynamic}
With the model structures described in Sections~\ref{sec:ddBNIRL-S} and \ref{sec:ddBNIRL-T}, we have presented two alternative views on the subgoal problem. Naturally, the question arises which of the two approaches %
is better suited for a particular application scenario. 
As explained in the previous paragraphs, the main difference between the two models lies in their structure, \ie, in the way subgoals are allocated. While \ddBNIRLS\ relies on a static assignment mechanism that consistently links the individual states of a system to their corresponding subgoals, \ddBNIRLT\ allocates its subgoals per demonstration pair. The latter means that different state-action pairs observed at the \textit{same} state can be explained %
using \textit{different} intentional %
settings (Figure~\ref{fig:crossing}). 
To answer the above question, we hence need to ask %
under which conditions an observed decision-making process can be described via a static assignment rule that \textit{uniquely} characterizes each system state, %
and in which situations %
we require a more flexible model that allows to take into account additional %
side information. %

From decision-making theory, we know that the optimal solutions for time-invariant MDPs can be formulated as a deterministic \textit{time-invariant} Markov policies \citep{puterman1994}, %
the class of which is fully covered by the static \ddBNIRLS\ framework.\footnote{While we omit a rigorous proof here, this can be seen intuitively by noticing that any state-to-action rule that is optimal for a given MDP reward function can be synthesized via \ddBNIRLS\ by %
assuming an individual subgoal for each state in the extreme case.} Therefore, assuming that the transition dynamics of our system are constant with respect to time and that the agent acts rationally while having complete knowledge of the environment, %
there exist only two plausible reasons why %
we would potentially observe the agent execute a time-variant policy:\looseness-1
\begin{enumerate}[topsep=-\parskip+1ex, itemsep=0ex, parsep=0mm]
\item[$\bullet$] either, the reward model of the agent changes over time,
\item[$\bullet$] or, the observed decision-making process %
is not Markovian with respect to the assumed state space model (\ie, the agent's decisions depend on additional context information that is not explicitly captured in our state representation).
\end{enumerate}
Accordingly, if we assume that the Markov property holds (meaning that the chosen state representation is sufficiently rich to capture the decision-making strategy of the agent), %
the only theoretical justification to prefer a dynamic subgoal model like \ddBNIRLT\ over a static one such as \ddBNIRLT\ would be if %
we assume that the intentions of the agent are truly 
time-dependent.

Practically speaking, 
however, there can %
be several reasons why a given state representation might not fulfill the Markov requirement. One obvious explanation would be that the actual state space of the demonstrator is not perfectly known. This situation occurs, for example, if not all state context %
available to the agent is observable by the modeler. Another potential situation is when the strategy of the agent %
depends on information that is %
independent of the system dynamics and hence deliberately excluded from the state variable~(\ie, parameters that are unaffected by the actions of the agent, such as the preselection of a specific high-level strategy). %
A generic framework for such settings is described by \citet{daniel2016hierarchical}, where the agent learns multiple sub-policies %
that are triggered depending on %
context information that is 
treated separately from the state.

To an external observer who is unaware of that context information, the resulting policy of the agent would potentially appear time-dependent, %
in which case the only chance to disentangle the individual sub-policies %
would be to resort to a dynamic subgoal encoding, such as provided by \ddBNIRLT. 
However, if the %
context is known (like the temporal information in Section~\ref{sec:ddBNIRL-T} as a particular example), %
both approaches can be used equivalently and %
will only differ in the resulting state representation. More specifically, %
we can either fall back on the static \ddBNIRLS\ model by augmenting the state variable with the context information accordingly, or we can resort to the dynamic subgoal allocation scheme of \ddBNIRLT, using a distance metric that %
accounts for %
the context.
Conversely, when considered in a purely time-invariant setting (where the context is described by some other known quantity), 
\ddBNIRLS\ and \ddBNIRLT\ can be regarded as two sides of the same coin, \ie, both can be used to describe the time-invariant policy of an observed demonstrator but they differ %
in the way the side information %
is represented.

\section{Prediction and Inference}
\label{sec:predAndInf}
Having introduced the ddBNIRL framework,
we now explain how it can be used to %
generalize a given
expert %
behavior. To this end, we first focus on the %
task of %
action prediction at a given query state, and then explain in a second step how to %
extract the necessary %
information from the %
demonstration data. Along the way, we also give insights into the implicit intentional model learned %
through the framework.

\vspace{0.5\baselineskip}
\noindent
\textbf{Note:} In order to keep the level of redundancy at a minimum, the following considerations are based on the %
\ddBNIRLS\  model. The %
results for \ddBNIRLT\  follow straightforwardly; the only %
change in the equations is the way the subgoals are referenced. To obtain the corresponding expressions, we simply replace the assignment variables $\mathbf{c}$ with $\mathbf{\tilde{c}}$ and change the cluster definition in Equation~\eqref{eq:clusterDefinition} to $\mathcal{C}_k\defeq\{d\in\{1,\ldots,D\}:\mathbf{\tilde{z}}(\mathbf{\tilde{c}})|_d=k\}$. Accordingly, all occurrences of $\mathbf{z}(\mathbf{c})\vert_{s^*}$ change to $\mathbf{\tilde{z}}(\mathbf{\tilde{c}})\vert_{d^*}$, $\mathbf{z}(\mathbf{c})\vert_{s_d}$ becomes $\mathbf{\tilde{z}}(\mathbf{\tilde{c}})\vert_{d}$, and $s_d\in\mathcal{C}_k$ is replaced with $d\in\mathcal{C}_k$.

\subsection{Action Prediction}
\label{sec:actionPrediction}
Similar to the work by \citet{abbeel2004apprenticeship}, we consider the task of predicting an action $a^*\in\mathcal{A}$ at some query state $s^*\in\mathcal{S}$ that is optimal with respect to the expert's \textit{unknown} reward model. However, %
in contrast to most existing IRL methods, %
our approach is not based on point estimates of the
expert's reward function but
takes into account the entire hypothesis space of %
reward models. %
This allows us to obtain the full posterior predictive policy from the expert data. %
Mathematically, %
the task %
is formulated as computing the predictive action distribution $\p(a^* \given s^*, \mathcal{D})$, which captures %
the full information about the expert behavior contained in the demonstration set $\mathcal{D}$. 
We start by expanding that distribution with the help of 
the latent state assignments~$\mathbf{c}$,
\begin{equation*}
	\p(a^* \given s^*,\mathcal{D}) = \sum_{\mathbf{c}\in\mathcal{S}^{|\mathcal{S}|}} \p(a^* \given s^*,\mathcal{D},\mathbf{c}) \p(\mathbf{c} \given \mathcal{D}) \eqpunkt
\end{equation*} 
The conditional distribution $\p(a^* \given s^*,\mathcal{D},\mathbf{c})$ %
can be expressed in terms of the %
posterior distribution of the subgoal targeted at the query state~$s^*$,
\begin{align*}
	\p(a^* \given s^*,\mathcal{D}) &= \sum_{\mathbf{c}\in\mathcal{S}^{|\mathcal{S}|}} \p(\mathbf{c} \given \mathcal{D}) \sum_{i\in\mathcal{S}}^{} \p(a^* \given s^*,\mathbf{c},g_{\mathbf{z}(\mathbf{c})\vert_{s^*}}=i) \p(g_{\mathbf{z}(\mathbf{c})\vert_{s^*}}=i \given \mathcal{D},\mathbf{c}) \eqkomma
\end{align*}
where we used the fact that the prediction $a^*$ is conditionally independent of the demonstration set $\mathcal{D}$ given the state partitioning structure and the corresponding subgoal assigned to $s^*$ (that is, given $\mathbf{c}$ and $g_{\mathbf{z}(\mathbf{c})\vert_{s^*}}$).
From the joint distribution in Equation~\eqref{eq:ddBNIRLSjoint}, it follows that
\begin{equation}
\p(g_k \given \mathcal{D},\mathbf{c}) = \frac{1}{Z_k(\mathcal{D},\mathbf{c})} \p_g(g_k \given \mathbf{s}) \!\!\! \prod_{d:\mathbf{z}(\mathbf{c})\vert_{s_d}=k} \!\!\! \softmax(a_d \given s_d,  g_k) \eqkomma %
\label{eq:subgoalPosterior}
\end{equation}
where $Z_k(\mathcal{D},\mathbf{c})$ is the corresponding normalizing constant,
\begin{equation}
Z_k(\mathcal{D},\mathbf{c}) \defeq \sum_{i\in\supp(p_g)} \p_g(g_k = i \given \mathbf{s}) \!\!\! \prod_{d:\mathbf{z}(\mathbf{c})\vert_{s_d}=k} \!\!\! \softmax(a_d \given s_d,  g_k = i) \eqpunkt
\label{eq:normalizingConstant}
\end{equation}
Using this relationship, we get
\begin{align*}
	\p(a^* \given s^*,\mathcal{D}) &= \sum_{\mathbf{c}\in\mathcal{S}^{|\mathcal{S}|}} \frac{1}{Z_k(\mathcal{D},\mathbf{c})} \p(\mathbf{c} \given \mathcal{D}) \hspace{-0.3cm} \sum_{i\in\supp(p_g)}^{} \hspace{-0.3cm} \p_g(g_{\mathbf{z}(\mathbf{c})\vert_{s^*}}=i  \given \mathbf{s}) \ \ldots
	\\ &\phantom{=}  \  \ldots \ \times \hspace{-0.75cm} \prod_{d:\mathbf{z}(\mathbf{c})\vert_{s_d}=\mathbf{z}(\mathbf{c})\vert_{s^*}} \hspace{-0.75cm} \softmax(a_d \given s_d, g_{\mathbf{z}(\mathbf{c})\vert_{s^*}}=i) \p(a^* \given s^*,\mathbf{c},g_{\mathbf{z}(\mathbf{c})\vert_{s^*}}=i) \eqpunkt
\end{align*}
In contrast to the summation over %
subgoal locations $i$, whose computational complexity is determined by the support of the subgoal prior distribution $\p_g$ and which grows at most linearly with the size of %
$\mathcal{S}$, the marginalization with respect to the indicator variables $\mathbf{c}$ involves the summation of $|\mathcal{S}|^{|\mathcal{S}|}$ terms and %
becomes quickly intractable even for small state spaces. %
Therefore, we %
approximate this operation via Monte Carlo integration, which yields %
\begin{align*}
\p(a^* \given s^*,\mathcal{D}) &\approx \frac{1}{N} \sum_{n=1}^{N} \sum_{i\in\supp(p_g)}^{} \hspace{-0.3cm} \p(g_{\mathbf{z}(\mathbf{c}^{\{n\}})\vert_{s^*}}=i \given \mathcal{D}, \mathbf{c}^{\{n\}})  \p(a^* \given s^*,\mathbf{c}^{\{n\}},g_{\mathbf{z}(\mathbf{c}^{\{n\}})\vert_{s^*}}=i) \eqkomma
\end{align*}
where $\mathbf{c}^{\{n\}} \sim \p(\mathbf{c} \given \mathcal{D})$. 
The final prediction step can then be performed, for example, via the {maximum a posteriori}~(MAP) policy estimate, %
\begin{equation}
	\hat{\pi}(s^*) \defeq \argmax_{a^*\in\mathcal{A}}\p(a^* \given s^*, \mathcal{D}) \eqpunkt
	\label{eq:MAPpolicy}
\end{equation}
The inference task, hence, reduces to the computation of 
the posterior samples~$\{\mathbf{c}^{\{n\}}\}$, %
which is described in the next section.

\subsection{Partition Inference}
\label{sec:partitionInference}
Based on the joint model in Equation~\eqref{eq:ddBNIRLSjoint},
we %
obtain the posterior distribution $\p(\mathbf{c} \given \mathcal{D})$ in factorized form as
\begin{align}
	\p(\mathbf{c} \given \mathcal{D}) &= \p(\mathbf{c}) \prod_{k=1}^\infty \; \sum_{g_k\in\supp(p_g)} \hspace{-0.3cm} \p_g(g_k \given \mathbf{s}) \prod_{d=1}^D \softmax(a_d \given s_d, g_{\mathbf{z}(\mathbf{c})\vert_{s_d}}) \nonumber
	\\ & = \p(\mathbf{c}) \prod_{k=1}^{|\mathbf{z}(\mathbf{c})|} \sum_{g_k\in\supp(p_g)} \hspace{-0.3cm} \p_g(g_k \given \mathbf{s}) \prod_{d:s_d\in\mathcal{C}_k} \softmax(a_d \given s_d, g_k) \eqkomma
	\label{eq:indicatorPosterior}
\end{align}
where $\mathcal{C}_k$ denotes the $k$th state cluster induced by the assignment $\mathbf{c}$, \begin{equation}
\mathcal{C}_k\defeq\{s\in\mathcal{S}:\mathbf{z}(\mathbf{c})|_s=k\} \eqkomma
\label{eq:clusterDefinition}
\end{equation}and $|\mathbf{z}(\mathbf{c})|$ is the total number of clusters defined by $\mathbf{c}$. %
As explained by~\citet{blei2011distance}, the indicator samples $\{\mathbf{c}^{\{n\}}\}$ can be efficiently generated using a fast-mixing Gibbs chain. %
Starting from a given ddCRP graph defined by the subset of indicators~$\mathbf{c}_{\without i}\defeq\{c_j\}\setminus c_i$, the insertion of an additional edge $c_i$ will result in one of three possible outcomes, as illustrated in Figure~\ref{fig:ddCRPgraph}: in the case of adding a self-loop ($c_i=i$), %
the underlying partitioning structure %
stays unaffected. Setting $c_i \neq i$ either leaves the structure unchanged (if the target state is already in the same cluster as state $i$) or creates a new link between two clusters. In the latter case, the involved clusters are merged, which corresponds to a merging of the associated sums in Equation~\eqref{eq:indicatorPosterior}. %
According to these three cases, the conditional distribution for the Gibbs procedure is obtained as
\begin{equation}
	\p(c_i=j \given \mathbf{c}_{\without i}, \mathcal{D}) \ \propto \
	\begin{cases}
		\nu  & \text{if } i=j \eqkomma\\
		f(d_{i,j}) & \text{if no clusters are merged} \eqkomma \\
		f(d_{i,j}) \frac{\mathcal{L}(\mathcal{C}_{z_i}\cup\,\mathcal{C}_{z_j})}{\vphantom{\displaystyle |}\mathcal{L}(\mathcal{C}_{z_i})\cdot\mathcal{L}(\mathcal{C}_{z_j})} & \text{if clusters } \mathcal{C}_{z_i} \text{ and } \mathcal{C}_{z_j} \text{ are merged} \eqpunkt
	\end{cases}
	\label{eq:ddCRPposterior}
\end{equation}
Herein, $\mathcal{L}(\mathcal{C})$ denotes the marginal action likelihood of all demonstrations %
accumulated in cluster $\mathcal{C}$,
\begin{equation}
\mathcal{L}(\mathcal{C}) = \displaystyle \sum_{g\in\supp(p_g)} \hspace{-0.3cm} \p_g(g  \given \mathbf{s}) \prod\limits_{d:s_d\in\mathcal{C}} \softmax(a_d \given s_d, g) \eqkomma
\label{eq:marginalActionLikelihood}
\end{equation}
which %
further %
represents the normalizing constant for the posterior distribution of the cluster subgoal (Equation~\ref{eq:normalizingConstant}).
Accordingly, the fraction in Equation~\eqref{eq:ddCRPposterior} can be interpreted as the likelihood ratio of the %
partitioning defined by $\mathbf{c}_{\without i}$ and the merged structure after inserting the new edge~$c_i$.

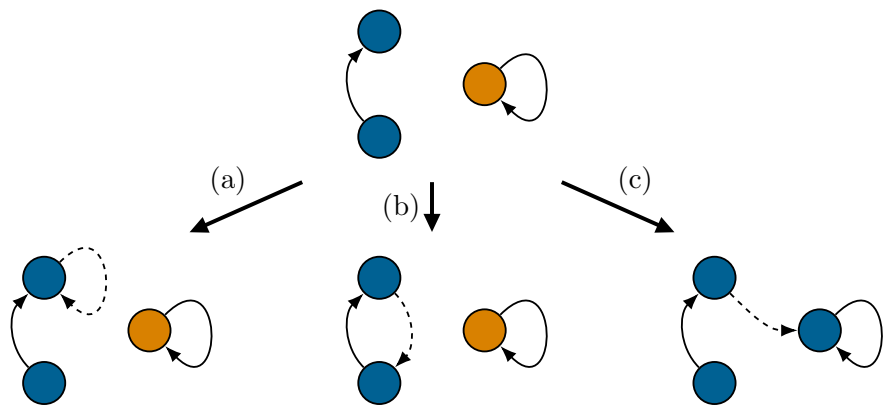
\begin{figure*}
	\centering
	\tikzstyle{node}=[circle, line width=1pt, draw=black!100, fill=black!5, minimum size=0.8cm]
	\begin{tikzpicture}[line width = 1pt, auto]
		\node (A) [] {	
			\scalebox{0.7}{
			\begin{tikzpicture}[line width=1pt]
			\node (a) [node, fill=col1] at (0,0) {};
			\node (b) [node, fill=col3] at (2,-1) {};
			\node (c) [node, fill=col1] at (0,-2) {};
			\draw [overlay, -{Latex[scale=1.2]}, out=135, in=225] (c) to (a);
			\draw [overlay, -{Latex[scale=1.2]}, loop right, out=45, in=315, looseness=7] (b) to (b);
			\end{tikzpicture}
			}
		};
		\node (B) [below = 1cm of A]{	
			\scalebox{0.7}{
				\begin{tikzpicture}[line width=1pt]
				\node (a) [node, fill=col1] at (0,0) {};
				\node (b) [node, fill=col3] at (2,-1) {};
				\node (c) [node, fill=col1] at (0,-2) {};
				\draw [overlay, -{Latex[scale=1.2]}, out=135, in=225] (c) to (a);
				\draw [overlay, -{Latex[scale=1.2]}, out=315, in=45, dashed] (a) to (c);
				\draw [overlay, -{Latex[scale=1.2]}, loop right, out=45, in=315, looseness=7] (b) to (b);
				\end{tikzpicture}
			}
		};
		\node (C) [left = 2cm of B] {	
			\scalebox{0.7}{
				\begin{tikzpicture}[line width=1pt]
				\node (a) [node, fill=col1] at (0,0) {};
				\node (b) [node, fill=col3] at (2,-1) {};
				\node (c) [node, fill=col1] at (0,-2) {};
				\draw [overlay, -{Latex[scale=1.2]}, out=135, in=225] (c) to (a);
				\draw [overlay, -{Latex[scale=1.2]}, loop right, out=45, in=315, looseness=7, dashed] (a) to (a);
				\draw [overlay, -{Latex[scale=1.2]}, loop right, out=45, in=315, looseness=7] (b) to (b);
				\end{tikzpicture}
			}
		};
		\node (D) [right = 2cm of B]{	
			\scalebox{0.7}{
				\begin{tikzpicture}[line width=1pt]
				\node (a) [node, fill=col1] at (0,0) {};
				\node (b) [node, fill=col1] at (2,-1) {};
				\node (c) [node, fill=col1] at (0,-2) {};
				\draw [overlay, -{Latex[scale=1.2]}, out=135, in=225] (c) to (a);
				\draw [overlay, -{Latex[scale=1.2]}, out=315, in=180, dashed] (a) to (b);
				\draw [overlay, -{Latex[scale=1.2]}, loop right, out=45, in=315, looseness=7] (b) to (b);
				\end{tikzpicture}
			}
		};
		
		\draw [line width=1.5pt, -{Triangle[width=6pt]}] ([yshift=-1ex]A.south) to node [swap] {(b)} ([yshift=1ex]B.north);
		\draw [line width=1.5pt, -{Triangle[width=6pt]}] ([xshift=-3ex, yshift=-1ex]A.south west) to node [swap, xshift=1ex] {(a)} ([yshift=1ex]C.north east);
		\draw [line width=1.5pt, -{Triangle[width=6pt]}] ([xshift=3ex, yshift=-1ex]A.south east) to node [xshift=-1ex] {(c)} ([yshift=1ex]D.north west);
	\end{tikzpicture}
\caption{Insertion of an edge (dashed arrow) to the ddCRP graph. Colors indicate the cluster memberships of the nodes, which are defined implicitly via the connected components of the graph. (a)~ Adding a self-loop or (b)~inserting an edge between two already connected nodes does not alter the clustering structure. (c)~Adding an edge between two unconnected components merges the associated %
	clusters.}
\label{fig:ddCRPgraph}
\end{figure*}

\subsection{Subgoal Inference}%
It is important to note that the inference method described in %
Sections~\ref{sec:actionPrediction} and \ref{sec:partitionInference} 
is based on a collapsed sampling scheme where all subgoals of our model are marginalized out. %
In fact, the ddBNIRL framework %
differs from BNIRL and other IRL methods %
in that the reward model of the expert is never made explicit for predicting new actions. Nonetheless, if desired (\eg, for the purpose of %
analyzing the expert's intentions), an estimate of the %
subgoal locations can be obtained in a post-hoc fashion from the subgoal 
posterior distribution in Equation~\eqref{eq:subgoalPosterior} for any given assignment $\mathbf{c}$. Examples are provided in Figure~\ref{fig:policySynthesis}.
\subsection{Action Inference}
\label{sec:actionInference}

As mentioned in Section~\ref{sec:limitations}, the original BNIRL algorithm requires complete knowledge of the expert's action record~$\mathbf{a}$, which limits the range of potential application scenarios. For this reason, we generalize our inference scheme to the case where we have access to  %
state information only, provided in the form of an alternative data set
$\mathcal{\overline{D}} \defeq \{ (s_d,\bar{s}_d)\}_{d=1}^D$, %
where $\bar{s}_d$ refers to the state visited by the expert immediately after $s_d$. In this setting, inference can be performed by extending the Gibbs procedure with an additional collapsed sampling %
stage,
\begin{align}
	\p(a_d \given \mathbf{a}_{\without d}, \overline{\mathcal{D}}, \mathbf{c}) &\propto T(\bar{s}_d \given s_d, a_d) \!\!\! \sum_{i\in\supp(p_g)} \!\!\! \p_g(g_{\mathbf{z}(\mathbf{c})|_{s_d}}=i)  \hspace{-0.4cm} \prod_{d': \mathbf{z}(\mathbf{c})|_{s_{d'}} = \mathbf{z}(\mathbf{c})|_{s_{d}}} \hspace{-4ex} \softmax(a_{d'} \given s_{d'}, g_{\mathbf{z}(\mathbf{c})|_{s_d}}=i) \eqkomma
	\label{eq:actionSampling}
\end{align}
which, for a fixed assignment $\mathbf{c}$, recovers an estimate of the latent action %
set $\mathbf{a}$ from the observed state transitions. %
Note that knowledge of the transition model $T$ is required for this step as it provides the necessary link between the %
expert's actions and the observed successor states. The same extension is possible for the \ddBNIRLT\ model, provided that the transition timestamps $\{t_d\}$ are known (Section~\ref{sec:ddBNIRL-T}).

\subsection{Computational Complexity}
\label{sec:complexity}
As a last point in this section, we would like to discuss the computational complexity of our approach. %
For this purpose, here a quick reminder on the used notation: we write $|\mathcal{S}|$ and $|\mathcal{A}|$ for the cardinalities of the state and action space, respectively, and use the letter $D$ for the size of the demonstration set. Further, we write $\mathcal{C}_k$ to refer to the $k$th state cluster (ddBNIRL-S) or data cluster (ddBNIRL-T). In the subsequent paragraphs, we additionally use the notation $N_{D}(\mathcal{C}_k)$ to access the number of demonstration data points associated with cluster~$\mathcal{C}_k$, $K$ to indicate the number of clusters in the current iteration, $N_g\defeq|\supp(p_g)|$ as a shorthand for the size of the support of the subgoal prior distribution, and $N_c$ for the number of indicator variables, \ie, $N_c\defeq|S|$ for \ddBNIRLS\  and $N_c\defeq D$ for ddBNIRL-T.\\

\noindent \textit{Initialization Phase:}
Common to all discussed models (including BNIRL) is that they depend on a preceding planning phase, %
where we compute, potentially in parallel, the state-action value functions~(Equation~\ref{eq:Qfunction}) for all $N_g$ considered subgoals, %
which allows us to construct the subgoal likelihood model (Equation~\ref{eq:softmaxPolicy} or \ref{eq:normalizedLikelihood}). The overall computational complexity of this procedure is of order $\mathcal{O}(N_g\mathbb{C}_\text{MDP}(|\mathcal{S}|,|\mathcal{A}|))$, where $\mathbb{C}_\text{MDP}(x,y)$ denotes the complexity of the used planning routine to (approximately) solve an MDP of size $x$ with a total number of $y$ actions. Using a value iteration algorithm, for instance, this can be achieved %
in $\mathcal{O}(\mathbb{C}_\text{MDP}(|\mathcal{S}|,|\mathcal{A}|))=\mathcal{O}(|\mathcal{S}|^2|\mathcal{A}|)$ steps \citep{littman1995complexity}.
If we assume that the expert reaches all subgoals during the demonstration phase \citep{michini2012bayesian}, we can restrict the support of the subgoal prior to the visited states, so that $N_g$ is upper-bounded by $\min(|\mathcal{S}|,D)$. Note that there %
exist approximation techniques that make the computation tractable in large/continuous state spaces (see discussion in Section~\ref{sec:conclusion}).

Before we start the sampling procedure, we %
compute all single-cluster likelihoods $\{\mathcal{L}(\mathcal{C}_k)\}$ and pairwise likelihoods $\{\mathcal{L}(\mathcal{C}_k\cup\mathcal{C}_{k'})\}$ according to Equation~\eqref{eq:marginalActionLikelihood}, based on some (random) initial cluster structure. 
The likelihood computation for the $k$th cluster $\mathcal{C}_k$ involves a product over $N_{D}(\mathcal{C}_k)$ data points, which needs to be calculated for each of the $N_g$ subgoals before taking their weighted average. This step has to be executed (potentially in parallel) for all clusters. However, because each %
demonstration is associated with exactly one cluster (either directly as in \ddBNIRLT\  or via the corresponding state variable as in ddBNIRL-S) and hence $\sum_kN_{D}(\mathcal{C}_k)=D$, the total %
complexity for computing all single-cluster likelihoods is of order $\mathcal{O}(N_gD)$, irrespective of the actual cluster structure. A similar line of reasoning applies to the computation of the pairwise likelihoods, yielding the same complexity order. %
Yet, for the latter we need to consider all possible cluster combinations.
Assuming an initial number of $K$ clusters, there are in total $K(K-1)/2$ pairwise likelihoods to be computed. Hence, the overall complexity of the initialization phase can be summarized as $\mathcal{O}(N_gDK^2)$.\\

\noindent \textit{Partition Inference:}
For the partition inference, the bulk of the computation lies in the repeated construction of the likelihood term in Equation~\eqref{eq:ddCRPposterior}, which needs to be updated whenever the cluster structure changes.
To analyze the complexity, we consider the sampling step of an individual assignment variable $c_i$ (or likewise $\tilde{c}_i$). In the worst case, removing the edge that belongs to $c_i$ from the ddCRP graph divides the associated cluster into two parts (Figure~\ref{fig:ddCRPgraph}), so that two new single-cluster likelihoods need to be computed. With the %
upper bound $D$ on the number of data points associated with the cluster before the division, this operation is of worst-case complexity $\mathcal{O}(N_gD)$ (see initialization phase). Irrespective of whether a division occurs, we then need to compute all pairwise cluster likelihoods %
with the (new) cluster connected via $c_i$. For a total of $K-1$ possible %
choices, this is done in $\mathcal{O}(N_gDK)$ operations (see initialization phase). After assigning the indicator, we move on to the next variable where the process repeats. If we assume, for simplicity, that the number of clusters stays constant during a full Gibbs cycle, the total complexity of updating all cluster assignments is hence of order $\mathcal{O}(N_gDKN_c)$. A (pessimistic) upper bound for the %
general case can be obtained by assuming that each data point defines its own cluster, in which case the complexity increases to $\mathcal{O}(N_gD^2N_c)$. Note that, in order to identify the new %
cluster structure after changing an assignment, we additionally need to track the connected components of the underlying ddCRP graph. As explained by \citet{kapron2013dynamic}, this can be done in polylogarithmic worst-case time. \\

\noindent\textit{Action Sampling:}
In order compute the conditional probability distribution of a particular action $a_d$, we need to evaluate a product involving all actions that belong to the same cluster as action $a_d$ (Equation~\ref{eq:actionSampling}). First, we can compute the product over all actions except $a_d$ itself, where the number of involved terms is again upper-bounded by $D$. %
Appending the term that belongs to $a_d$ for all possible action choices requires another $|\mathcal{A}|$ operations. These two steps need to be repeated for all possible subgoals, yielding an upper bound on the complexity of order $\mathcal{O}(N_g(D+|\mathcal{A}|))$. For a full Gibbs cycle, which involves sampling all $D$ action variables, the overall (worst-case) complexity is hence of order $\mathcal{O}(N_g(D+|\mathcal{A}|)D)$.

\section{Experimental Results}
\label{sec:results}
In this section, we present experimental results for our framework. The evaluation is separated into four parts: 
\begin{enumerate}[topsep=-\parskip+1ex, itemsep=0ex, parsep=0mm]
\item[\first] a proof of concept and conceptual comparison to BNIRL (Section~\ref{sec:proofConcept}), 
\item[\second] a performance comparison with related algorithms (Section~\ref{sec:randomMDP}), 
\item[\third] a real data experiment conducted on a KUKA robot (Section~\ref{sec:robotExperiment}) and 
\item[\fourth] an active learning task (Section~\ref{sec:activeLearning}).
\end{enumerate}

\begin{figure}[p]
	\newlength{\figsize}
	\newlength{\hsep}
	\newlength{\vsep}
	\setlength{\figsize}{3.4cm}
	\setlength{\hsep}{0.33cm}
	\setlength{\vsep}{0.75cm}
	
	\centering
	\begin{tikzpicture}[inner sep=0mm, line width=1pt, -latex]
	
	\node (post1) [draw] {\includegraphics[width=\figsize]{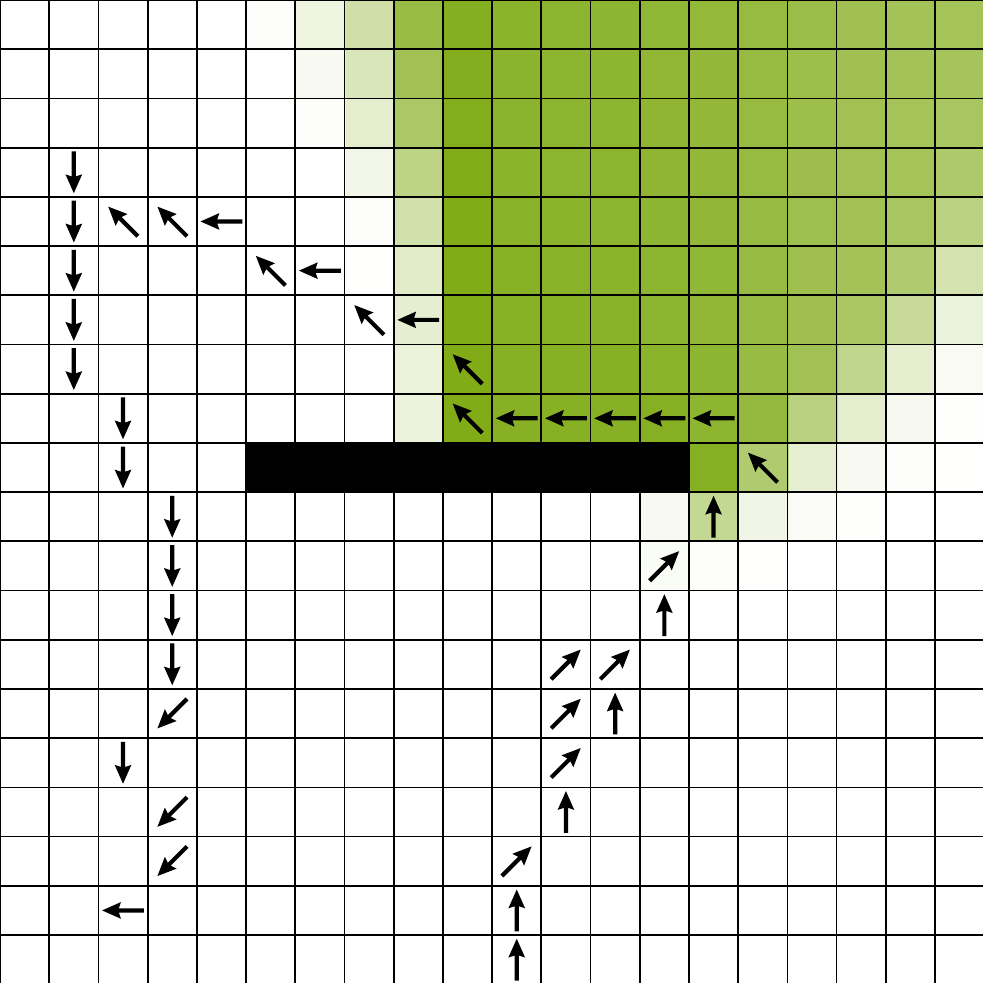}};
	\node (post2) [draw, right=\hsep of post1] {\includegraphics[width=\figsize]{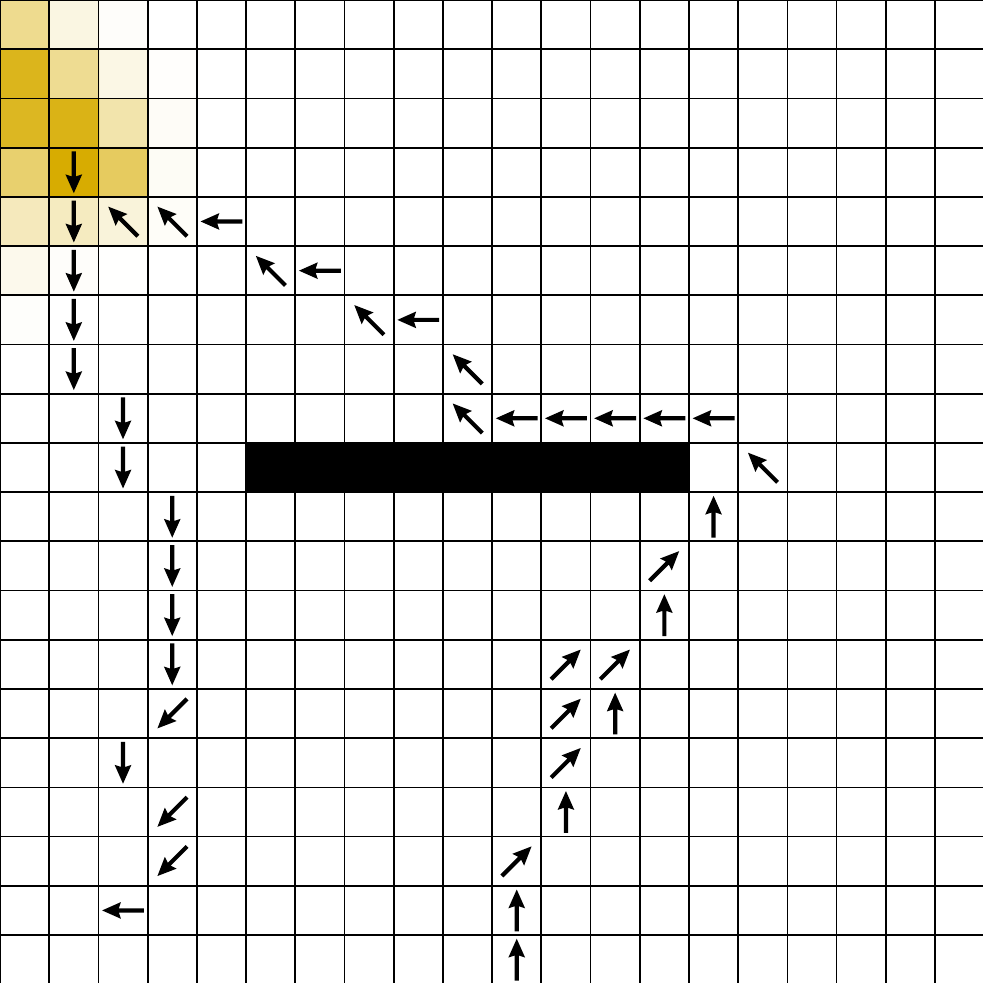}};
	\node (post3) [draw, right=\hsep of post2] {\includegraphics[width=\figsize]{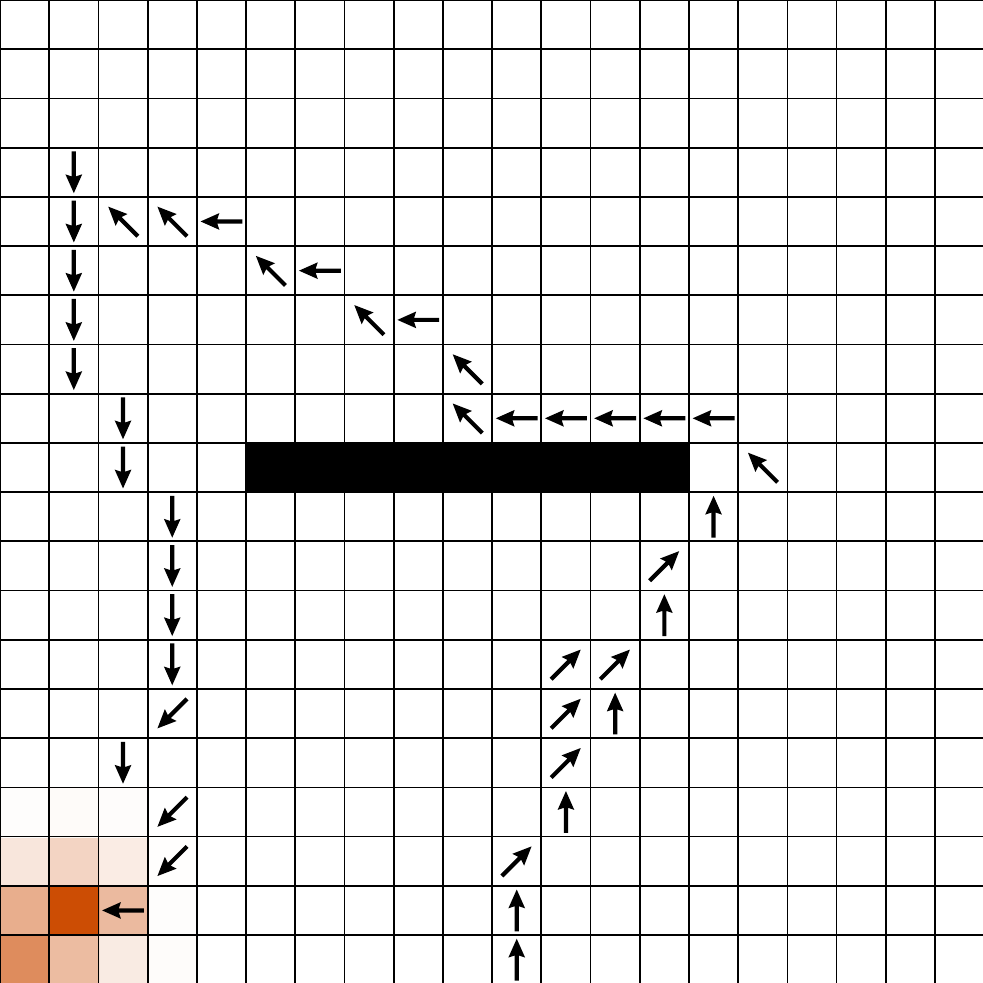}};
	\node (plate) [draw, dotted, rounded corners=5mm, inner sep = 0.3cm, fit=(post1) (post2) (post3)] {};
	\node (dummy) [left=\hsep of post1, minimum size=\figsize] {};
	\node (posteriors) [inner sep=0.1cm, left=-2.8cm of dummy] {\scriptsize subgoal posteriors};
	\node (partitionings) [inner sep=0.1cm, above=0.5cm of posteriors] {\scriptsize sample partitionings};
	\node (policies) [inner sep=0.1cm, below=0.5cm of posteriors] {\scriptsize predictive policies};
	\draw (partitionings.north) -- +(0,0.5cm);
	\draw (posteriors.east) -- +(0.5cm,0);
	\draw (policies.south) -- +(0,-0.5cm);
	
	\node (spatial) [draw, above=\vsep of post1, label={[shift={(0,1ex)}] ddBNIRL-S}] {\includegraphics[width=\figsize]{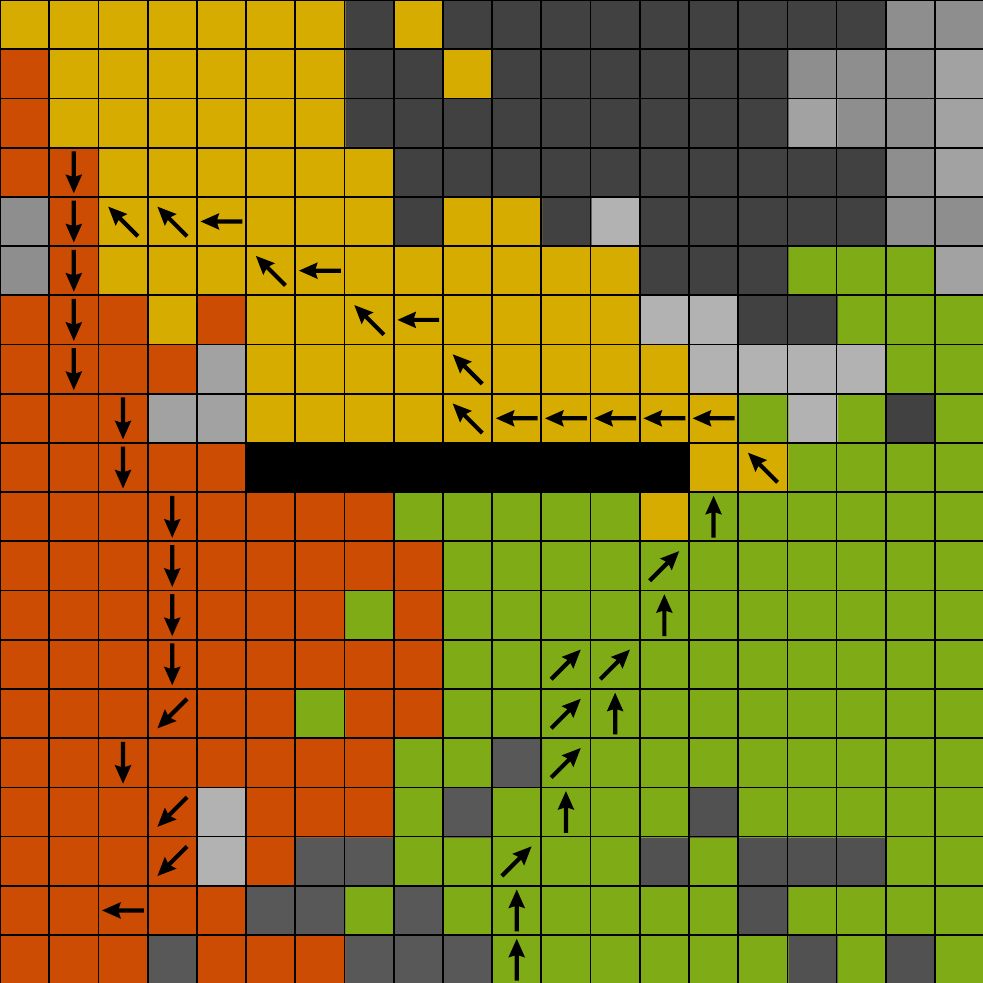}};
	\node (BNIRL) [draw, left=\hsep of spatial, label={[shift={(0,1ex)}] BNIRL}] {\includegraphics[width=\figsize]{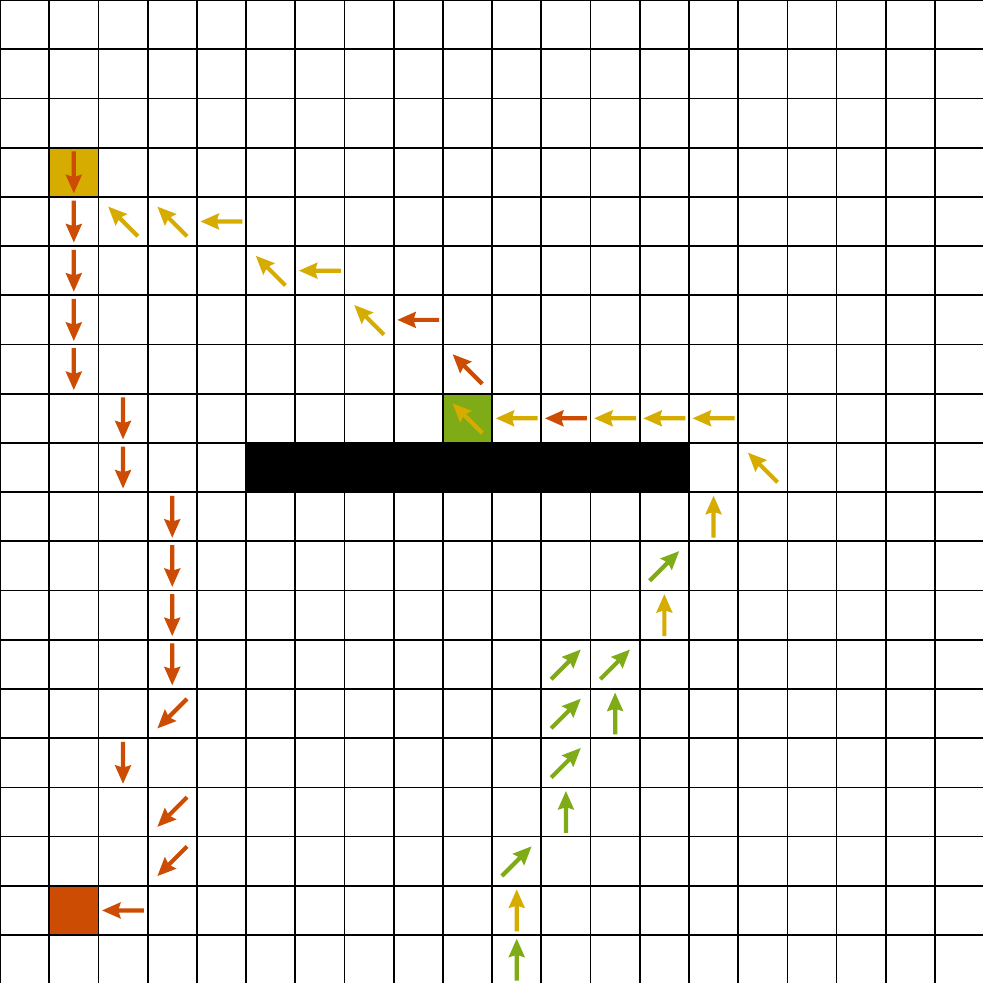}};
	\node (temporal) [draw, above=\vsep of post3, label={[shift={(0,1ex)}] ddBNIRL-T}] {\includegraphics[width=\figsize]{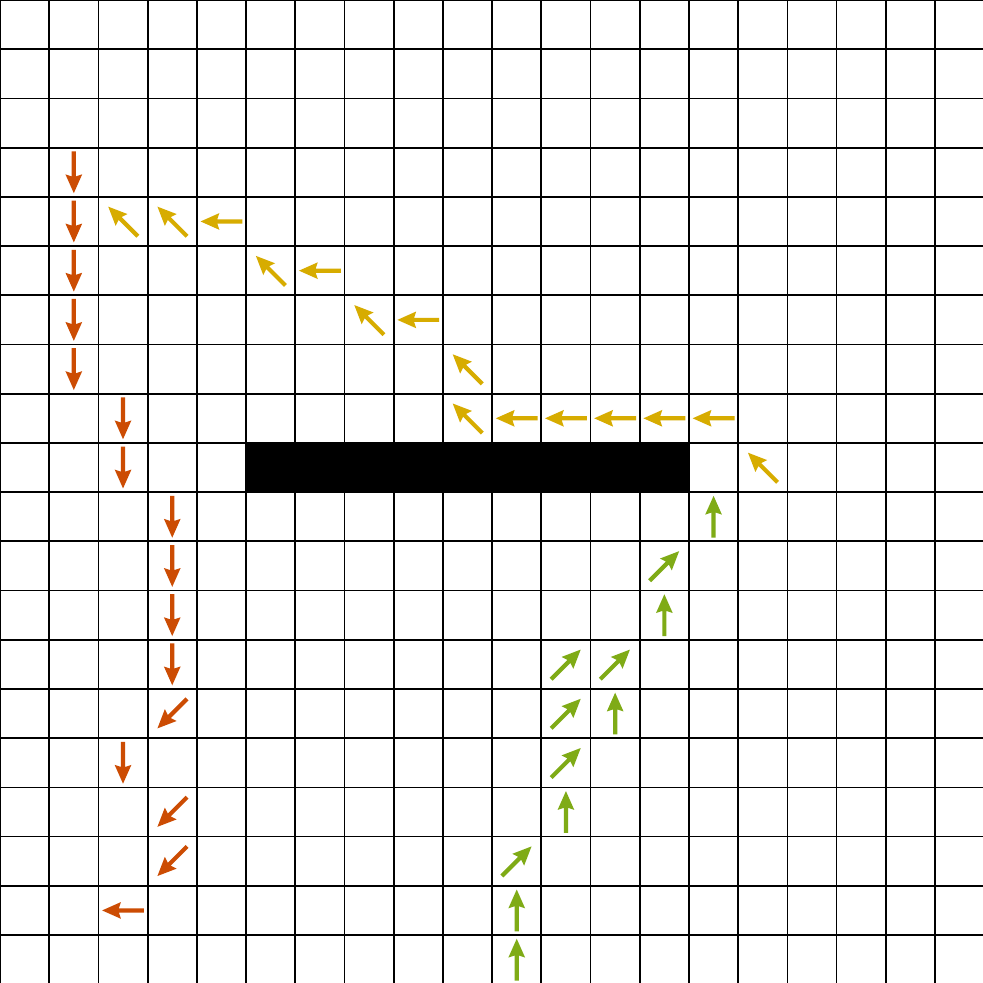}};
	
	\node (t1) [draw, below=\vsep of post1, label={[shift={(0,-1ex)}, font=\scriptsize]below:{phase 1: green subgoal}}] {\includegraphics[width=\figsize]{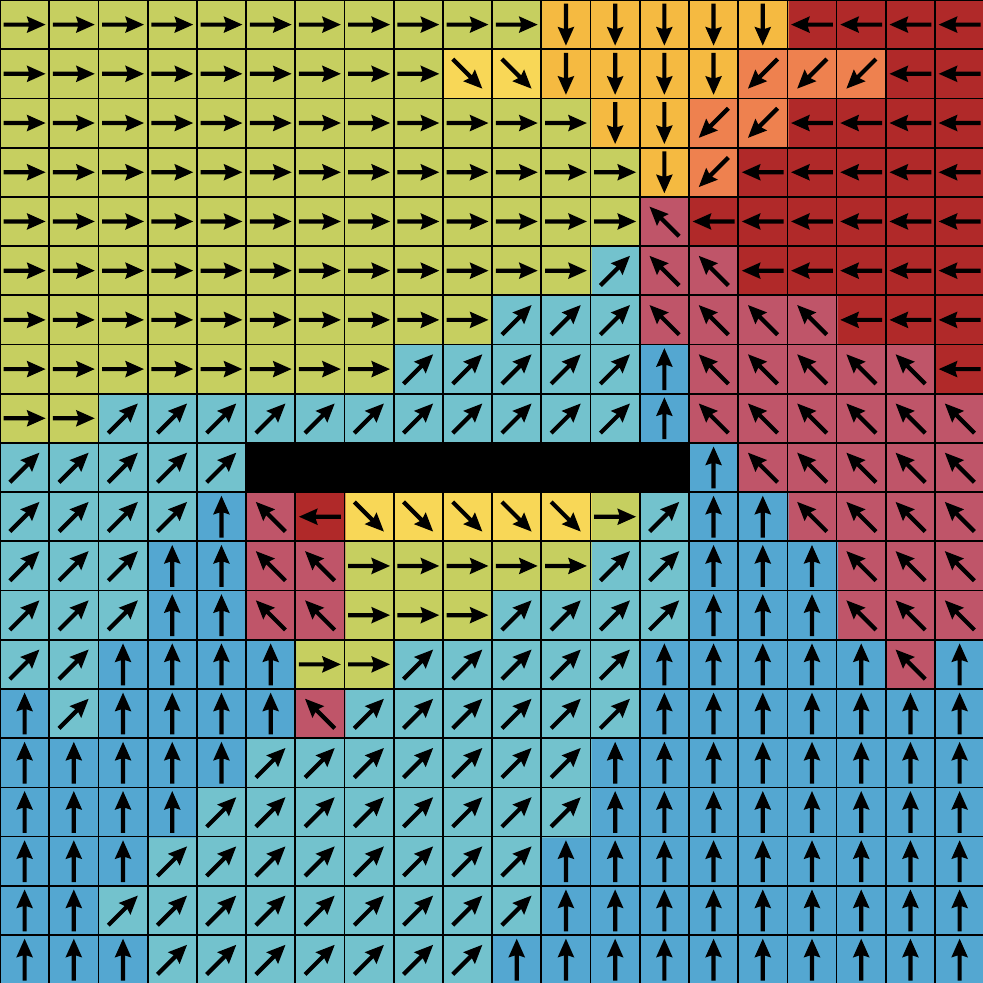}};
	\node (t2) [draw, right=\hsep of t1, label={[shift={(0,-1ex)}, font=\scriptsize]below:{ phase 2: yellow subgoal}}] {\includegraphics[width=\figsize]{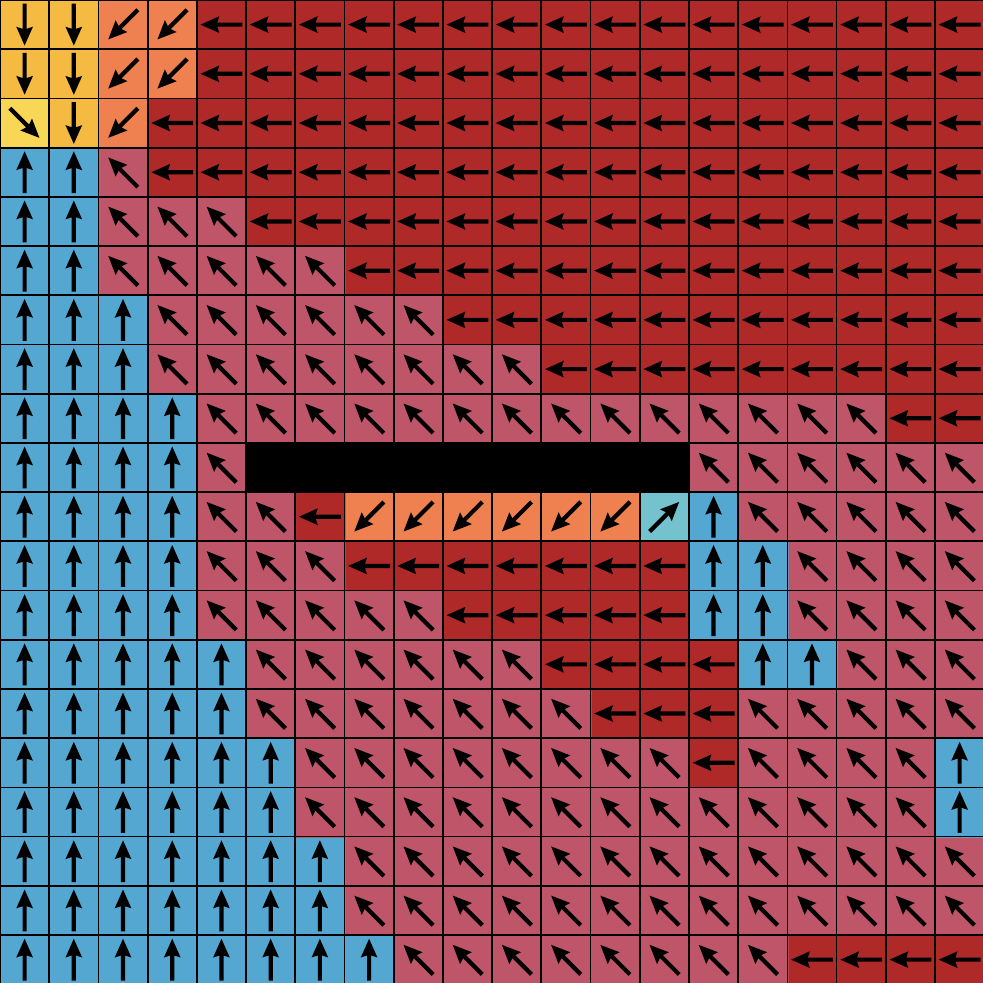}};
	\node (t3) [draw, right=\hsep of t2, label={[shift={(0,-1ex)},font=\scriptsize]below:{phase 3: red subgoal}}] {\includegraphics[width=\figsize]{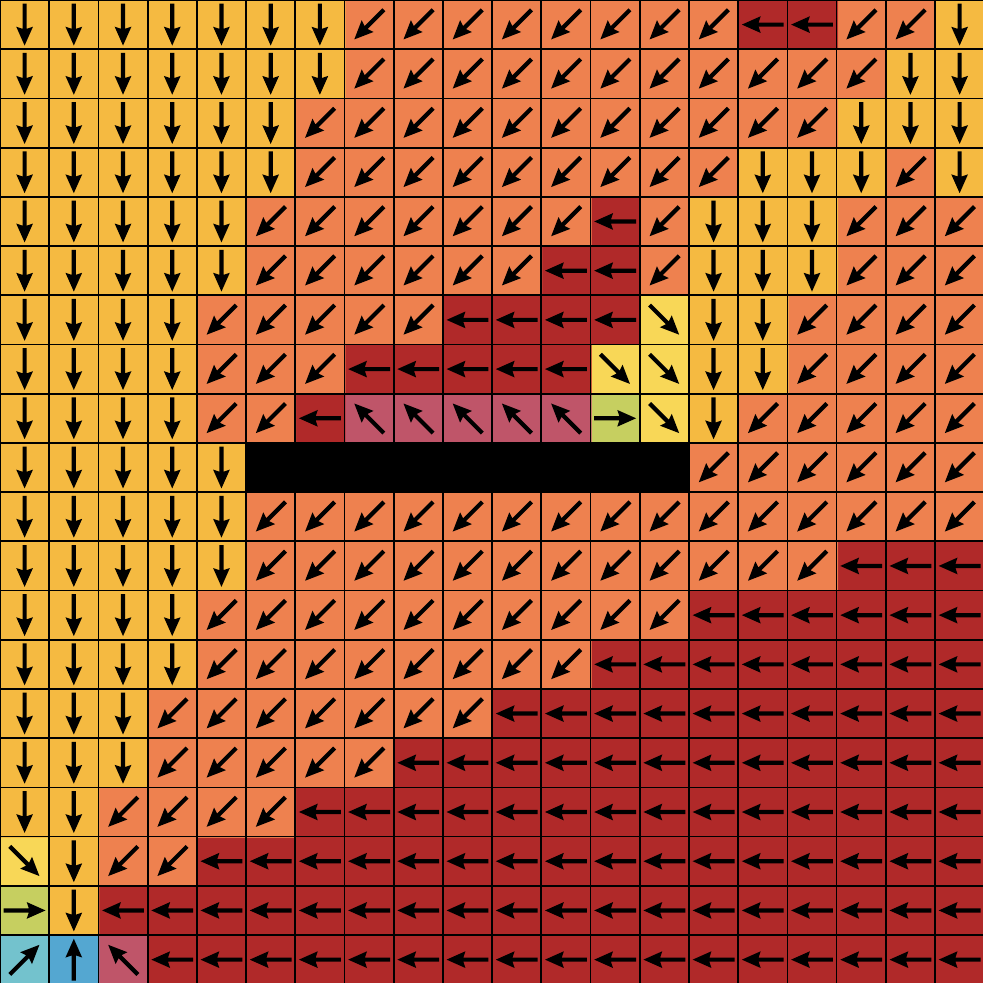}};
	\node (single) [draw, left=\hsep of t1, label={[shift={(0,-1ex)}, align=center, font=\scriptsize]below:{spatial policy\\(all subgoals combined)}}] {\includegraphics[width=\figsize]{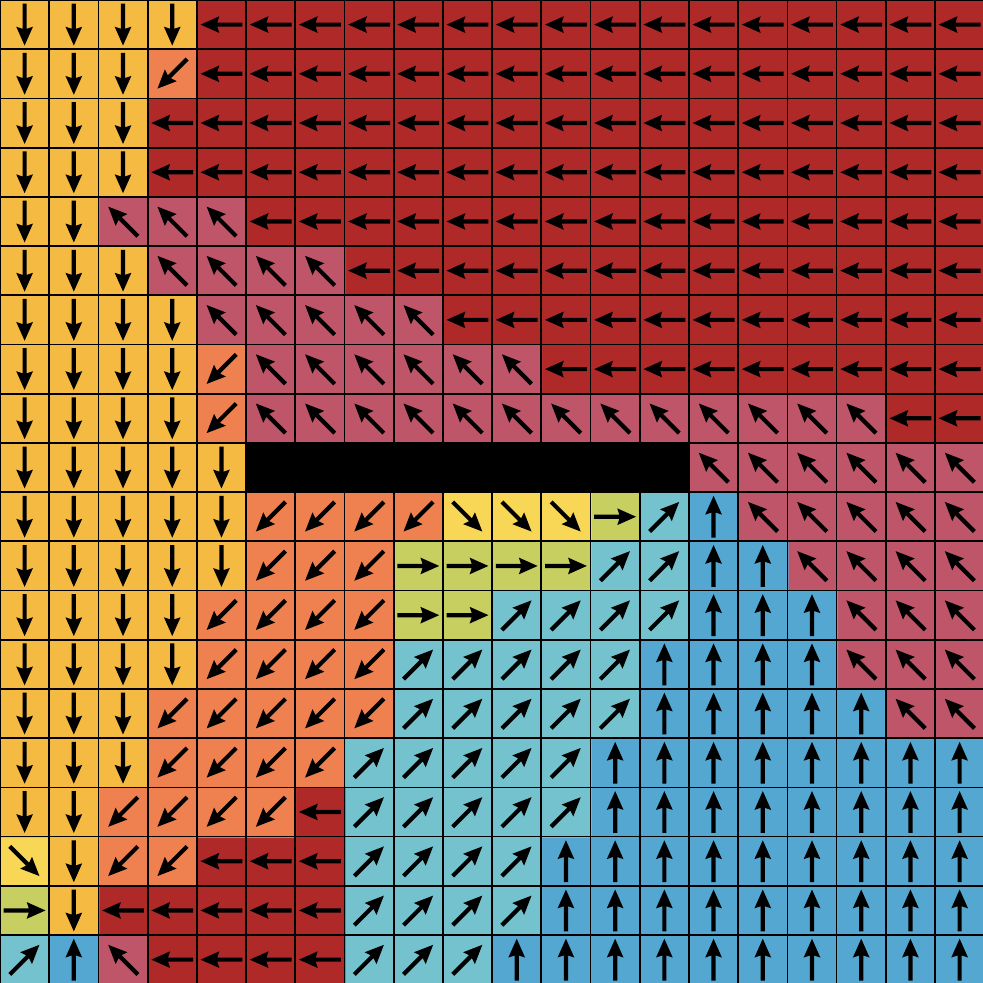}};

	\draw [-, line width=1.5pt, decoration={brace, mirror, raise=0.5cm, amplitude=10}, decorate] (t1.south west) to node [font=\scriptsize, below, yshift=-1cm] {temporal policy} (t3.south east) ;
	
	\draw [bend left] (spatial) to (plate);
	\draw [bend right] (temporal) to (plate);
	\draw (post1) to (t1);
	\draw (post2) to (t2);
	\draw (post3) to (t3);
	\draw [dotted, out=200, in=70] ([yshift=-0.75cm]plate.west) to ([xshift=0.5cm]single.north);
	\end{tikzpicture}
	\caption{Results on the BNIRL data set \citep{michini2012bayesian}. \textbf{Top row:} demonstration data and sample partitionings generated by the different inference algorithms. \textbf{Center row:} subgoal posterior distributions associated with the partitions found by \ddBNIRLS\  and \ddBNIRLT. For a clearer overview, the corresponding BNIRL distributions are omitted (see Figure~\ref{fig:Qnormalization} for a comparison). %
		\textbf{Bottom row:} time-invariant \ddBNIRLS\  policy model synthesized from all three detected subgoals (left) and %
		temporal %
		 phases identified by \ddBNIRLT~(right). The background colors have no particular meaning and were added only to highlight the structures of the policies. Because of its missing generalization mechanism, BNIRL does not itself provide a reasonable predictive policy model (Limitation~\hyperref[phantom:lim1]{1}).}
	\label{fig:policySynthesis}
\end{figure}

\subsection{Proof of Concept}
\label{sec:proofConcept}
To illustrate the conceptual differences to BNIRL and provide %
additional insights into the %
latent intentional model learned through our framework, we %
begin with the motivating data set from Figure~\ref{subfig:JMLR:illustrationSubgoal}, which had been originally %
presented %
by \citet{michini2012bayesian}. The considered system environment, defined by $|\mathcal{S}|=20\times20=400$ grid positions, is again shown in the top left corner of 
Figure~\ref{fig:policySynthesis}. Nine of those positions correspond to inaccessible wall states, marked by the horizontal black bar. %
At %
 the valid states, the expert can choose from an action set comprising a total of eight actions, each %
initiating
a noisy state transition %
toward %
one of the \mbox{(inter-)cardinal} directions. The observed state-action pairs are depicted in the form of arrows, whose colors indicate the MAP partitioning %
learned through BNIRL. %
The remaining subfigures show the results of the ddBNIRL framework, which were 
obtained from a posterior sample returned by 
the respective algorithm (\mbox{ddBNIRL-S/T}) at a low temperature in a simulated annealing schedule~\citep{kirkpatrick1983optimization}.

Comparing %
the obtained results, we observe the following main differences to the original approach: 
\begin{enumerate}[topsep=-\parskip+1.5ex, itemsep=1ex, parsep=0mm]
\item[\first] Unlike BNIRL, the proposed framework allows to choose between a spatial and a temporal encoding of the observed task, providing the possibility to 
account explicitly
for the type of demonstrated behavior %
(static/dynamic).
As explained in Section~\ref{sec:relationToBNIRL}, the context-unaware (yet in principle %
dynamic) vanilla BNIRL inference scheme is still included as a special case. 
\item[\second] Exploiting the spatial/temporal context of the data, the \mbox{ddBNIRL} solution is inherently robust to demonstration noise, %
giving rise to notably smoother partitioning structures (top row). This effect is particularly pronounced in the case of real data, as we shall see later in Section~\ref{sec:KUKAtemporal}.
\item[\third] For each state partition or trajectory segment, we obtain an implicit representation of the associated subgoal in the form of a posterior distribution, without the need of assigning point estimates (center row). %
It is striking that the posterior distribution corresponding to the green state partition has a comparably large spread on the upper side of the wall. %
This can be explained intuitively by the fact that any subgoal located in this high posterior region could have potentially caused the %
green state sequence,
which %
circumvents the wall from the right. At the same time, the %
green area %
of high posterior values
exhibits a sharp boundary on the left side since a subgoal located %
in the upper left region of the state space would have more likely resulted in a trajectory %
approaching from the left. %
\item[\fourth] In contrast to BNIRL, which has no built-in generalization mechanism (Limitation~\hyperref[phantom:lim1]{1}), our method returns %
a predictive policy model comprising the full posterior action information at all states. Note that %
we only show the resulting MAP policy estimates here (bottom row), computed according to Equation~\eqref{eq:MAPpolicy}. Additional results %
concerning the posterior uncertainty are provided in Sections~\ref{sec:robotExperiment} and~\ref{sec:activeLearning}.
\end{enumerate}
The example %
illustrates how the synthesis of the predictive policy differs between \ddBNIRLS\ (bottom left) and \ddBNIRLT\  (bottom row, rightmost three subfigures). While \ddBNIRLT\  uses a set of (conditionally) independent policy models to describe the different identified behavioral phases, \ddBNIRLS\  maps the entire subgoal schedule onto a single time-invariant policy representation. Looking closer at the learned models, we recognize that the \ddBNIRLS\  solution in fact realizes a spatial combination of the three temporal \ddBNIRLT\  components, where each component is activated in the corresponding cluster region of the state space. This gives us two alternative %
interpretations of the same behavior.

\subsection{Random MDP Scenario}
\label{sec:randomMDP}
Our next experiment is designed to provide insights into the generalization abilities of the framework. For this purpose, we consider a class of randomly generated MDPs similar to the Garnet problems \citep{bhatnagar2009natural}. The transition dynamics $\{T(\cdot \given s, a)\}$  are sampled independently from a symmetric Dirichlet distribution with a concentration parameter of $0.01$, where we choose $|\mathcal{S}|=100$ and $|\mathcal{A}|=10$. 
For each repetition of the experiment, $N_R$ states are selected uniformly at random and assigned rewards that are, in turn, sampled uniformly from the interval $[0,1]$. All other states %
contain zero reward. Next, we compute an optimal deterministic MDP policy $\pi^*$ with respect to a discount factor %
of $\gamma=0.9$  
and generate a number of expert trajectories of length 10. %
Herein, we let the expert select the optimal action with probability $0.9$ and %
a random, suboptimal action with probability $0.1$. The obtained state sequences %
are passed to the %
algorithms and we compute the normalized value loss of the reconstructed policies according to
\begin{equation}
\mathrm{L}(\pi^*, \hat{\pi}) \defeq \frac{\lVert \mathbf{V}^{*}-\mathbf{V}^{\hat{\pi}} \rVert_2}{\lVert\mathbf{V}^{*}\rVert_2} \eqkomma
\label{eq:valueLoss}
\end{equation}
where $\mathbf{V}^{*}$ and $\mathbf{V}^{\hat{\pi}}$ represent, respectively, the vectorized value functions of the optimal policy $\pi^*$ and the reconstruction $\hat{\pi}$. %

\begin{figure*}[]
	\centering
	\hspace*{0.7cm}\includegraphics[scale=1]{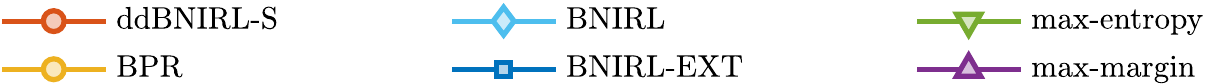}\\[0.5\baselineskip]
	
	\captionsetup[subfigure]{oneside,margin={1.1cm,0cm}}
	\subcaptionbox{$N_R=1 \ (\widehat{=}\, |\mathcal{S}|/100)$}{
		\includegraphics[scale=1]{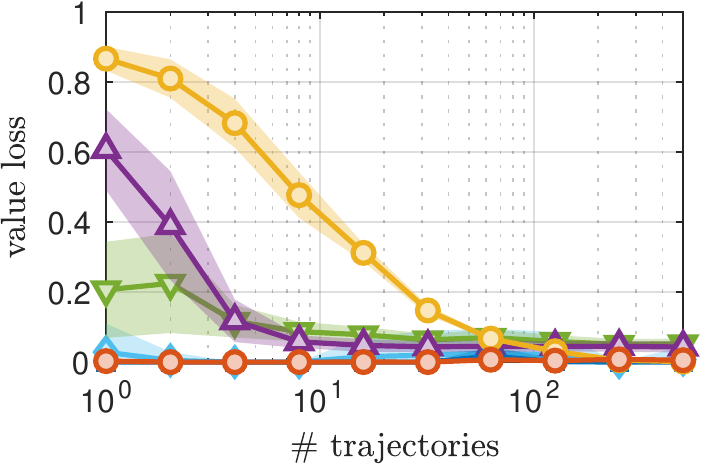}
	}
	\hfill
	\captionsetup[subfigure]{oneside,margin={1.1cm,0cm}}
	\subcaptionbox{$N_R=10 \ (\widehat{=}\, |\mathcal{S}|/10)$}{
		\includegraphics[scale=1]{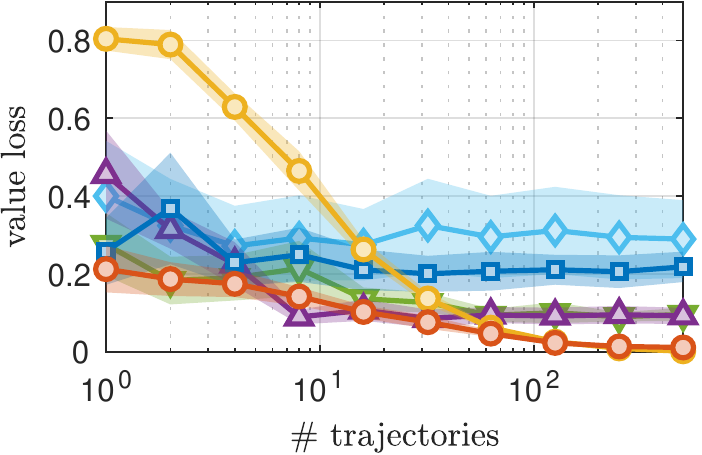}
	}
	
	\vspace{0.3cm}

	\captionsetup[subfigure]{oneside,margin={1.1cm,0cm}}
	\subcaptionbox{$N_R=50 \ (\widehat{=}\, |\mathcal{S}|/2)$}{
		\includegraphics[scale=1]{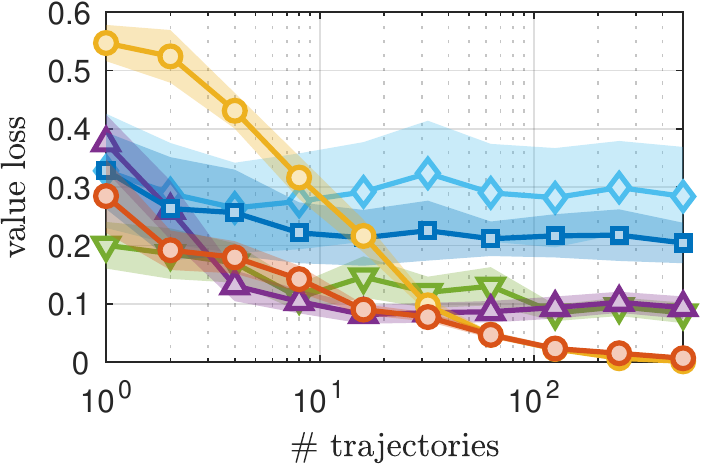}
	}
	\hfill
	\captionsetup[subfigure]{oneside,margin={1.1cm,0cm}}
	\subcaptionbox{$N_R=100 \ (\widehat{=}\, |\mathcal{S}|)$}{
		\includegraphics[scale=1]{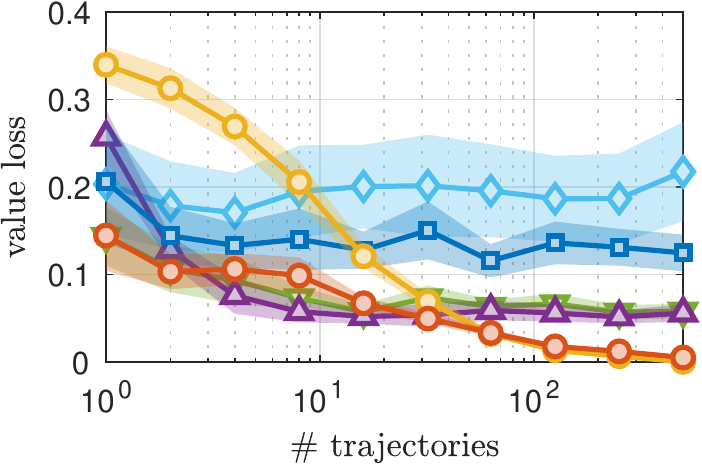}
	}
	\caption{Comparison of all inference methods in the random MDP scenario for different reward densities. Shown are the empirical mean values and standard deviations of the resulting value losses, obtained from 100 Monte Carlo runs. The graphs show a clear difference between BNIRL, BNIRL-EXT and ddBNIRL-S, which illustrates the importance of considering the spatial context for subgoal extraction.\looseness-1}
	\label{fig:randomWorldPolicyLoss}
\end{figure*}

Since the considered system belongs to the class of time-invariant MDPs,
\ddBNIRLS\  lends itself as the natural choice to model the expert behavior.
As baseline methods, we %
adopt our %
subintentional Bayesian policy recognition framework %
 \citep[BPR,][]{sosic2018pami}, as well as maximum-margin IRL \citep{abbeel2004apprenticeship}, maximum-entropy IRL \citep{ziebart2008maximum}, and vanilla BNIRL. %
Due to the missing generalization abilities of BNIRL (Limitation~\hyperref[phantom:lim1]{1}) and because the waypoint method (Section~\ref{sec:limitations}) does not straightforwardly apply to the considered scenario of multiple unaligned trajectories, we further compare our algorithm to an extension of BNIRL, which we refer to as BNIRL-EXT. Mimicking the \ddBNIRLS\ principle, %
the method %
accounts for the spatial context of the demonstrations by assigning each state to the BNIRL subgoal %
that is targeted by the closest %
(see metric in Section~\ref{sec:canonicalStateMetric}) state-action pair\nachschub{however,  these assignments are made 
\textit{after} the actual subgoal inference.} When compared to ddBNIRL-S, this provides a reference of how much can be %
gained by %
considering the spatial relationship of the data %
\textit{during} the inference. For the experiment, both \ddBNIRLS\  and BNIRL(-EXT) are augmented with their corresponding action sampling stages (Section~\ref{sec:actionInference}) since the action sequences of the expert are discarded from the data set, in order to enable a fair comparison to the remaining algorithms.

Figure~\ref{fig:randomWorldPolicyLoss} shows the value loss over %
the size of the demonstration set for different reward settings.
For small $N_R$, both BNIRL(-EXT) and \ddBNIRLS\ significantly outperform the reference methods. This is because the sparse reward structure allows for 
an efficient subgoal-based encoding of the 
expert behavior, %
which enables the algorithms to reconstruct %
the %
policy even from minimal amounts of demonstration data. However, the BNIRL(-EXT) solutions drastically deteriorate for denser reward structures. 
In particular, we observe a clear difference in performance between the cases where 
\begin{enumerate}[topsep=-\parskip+1.5ex, itemsep=1ex, parsep=0mm]
\item[\first] we do not account for the spatial information in the partitioning model~(BNIRL), 
\item[\second] include it in a post-processing step~(BNIRL-EXT), and
\item[\third] exploit it during the inference itself~(ddBNIRL-S), 
\end{enumerate}
which demonstrates the importance of %
processing the context information.
Most %
tellingly, \ddBNIRLS\  outperforms the baseline methods even in the dense reward regimes, although the subgoal-based encoding loses its efficiency here.
In fact, the results reveal that the proposed approach combines the %
merits of both model types, \ie, the sample efficiency of the intentional models (max-margin/max-entropy) required for small data set sizes, as well as the asymptotic accuracy and fully probabilistic nature of the subintentional Bayesian framework (BPR).\footnote{The comparably large loss of BPR for small data set sizes can be explained by the fact that the framework %
is based on a more general policy model in which the expert behavior is assumed to be \textit{inherently} stochastic, in contrast to the here considered setting where stochasticity arises merely a consequence of suboptimal decision-making.} %

\subsection{Robot Experiment}
\label{sec:robotExperiment}
\begin{figure*}
	\centering
	\includegraphics[width=0.6\textwidth]{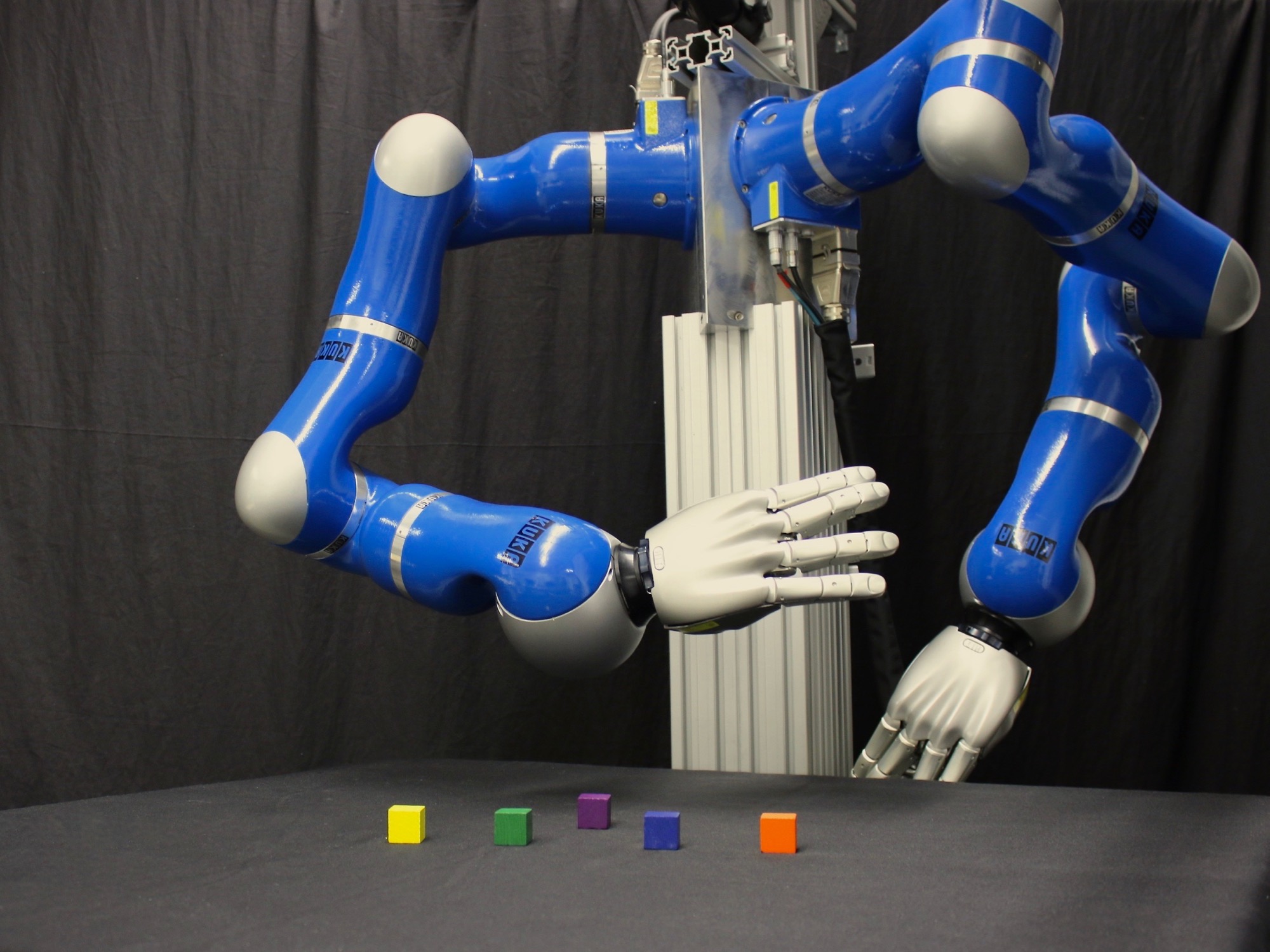}
	\caption{KUKA lightweight robotic arm.}
	\label{fig:KUKA}
\end{figure*}

In the next experiment, we test the ddBNIRL framework on various real data sets, which we
recorded on a KUKA lightweight robotic arm (Figure~\ref{fig:KUKA}) via kinesthetic teaching. Videos of all demonstrated 
tasks 
can be found at \url{http://www.spg.tu-darmstadt.de/jmlr2018}.

The system has seven degrees of freedom, corresponding to the seven joints of the arm. Each joint is equipped with a torque sensor and an angle encoder, providing recordings of joint angles, velocities and accelerations. For our experiments, 
we only consider the xy-Cartesian coordinates spanning the transverse plane, %
which we computed from the raw measurements using a forward kinematic model. The data was recorded at a sampling rate of \SI{50}{\hertz} %
and %
further downsampled by a factor of 10, yielding an effective sample rate of \SI{5}{\hertz}, which provided a sufficient temporal resolution for the considered scenario. %

The goal of %
the experiment is to learn a set of high-level intentional models for the %
recorded behavior types by partitioning the %
data sets into %
meaningful %
parts that can be used to predict the desired motion direction of the expert. 
For simplicity and to demonstrate the algorithm's robustness to modeling errors, we adopt the simplistic transition model from Section~\ref{sec:proofConcept} with the same action set %
containing 
the eight %
(inter-)cardinal motion directions. The high measurement accuracy of the end-effector position allows us to extract these high-level actions %
directly from the %
raw data, \ie, by selecting %
the directions with the smallest %
angular deviations from the %
ground truth
(see example in Figure~\ref{fig:robotData}). %
The underlying 
state space is obtained by discretizing the %
part of the coordinate %
range that is 
covered by the measurements into blocks of %
predefined size (see next sections for details). Apart from this discretization step and the aforementioned data downsampling, no preprocessing %
is applied. \looseness-1

\subsubsection{Spatial Partitioning}
First, we consider a case where the expert behavior can be described using a time-invariant policy model, which we aspire to capture via ddBNIRL-S. For our example, we %
consider the ``Cycle'' task shown in the video and in Figure~\ref{fig:wholeDataset}. The same setting is %
analyzed using the time-variant \ddBNIRLT\  model in Section~\ref{sec:KUKAtemporal}, which allows a direct comparison of the two approaches. %
The %
task consists in approaching a number of target positions, indicated by a set of markers (see video), %
before eventually returning to the initial state. The setting can be regarded as a real-world version of the ``Loop'' problem described by \citet{michini2012bayesian}. As explained in their paper, classical IRL algorithms that rely on a global state-based reward model (such as max-margin IRL and max-entropy IRL) completely fail on this problem, due to the periodic nature of the task.

\begin{figure}[h!]
	\centering
	\subcaptionbox{expert data \& sample partitioning\label{fig:robotData}}{
		\includegraphics[scale=1, angle=0]{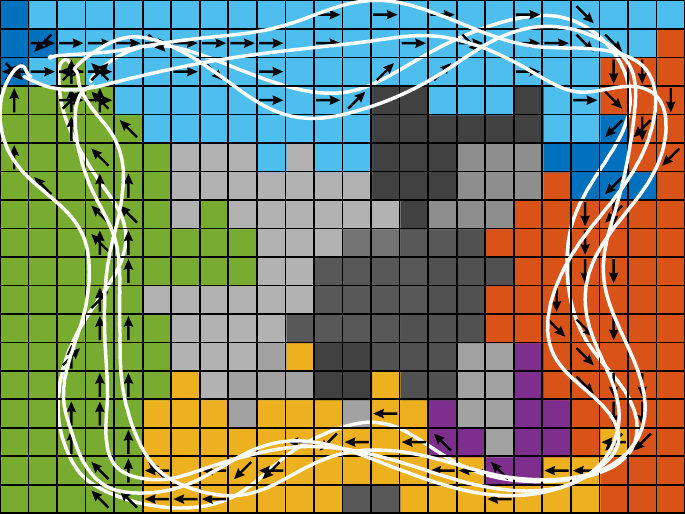}
	}
	\hspace{3ex}
	\subcaptionbox{MAP policy estimate\label{fig:cyclicMAPpolicy}}{
		\includegraphics[scale=1, angle=0]{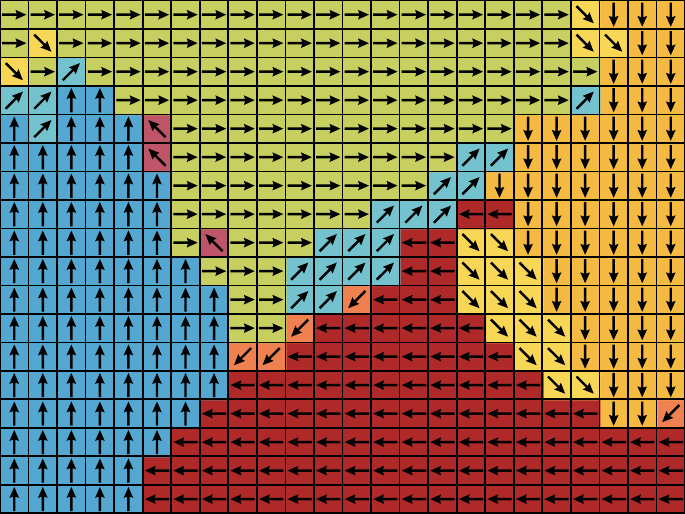}
	}
	
	\vspace{0.5\baselineskip}
	
	\subcaptionbox{uncertainty estimate\label{fig:entropyMap}}{
		\includegraphics[scale=1, angle=0]{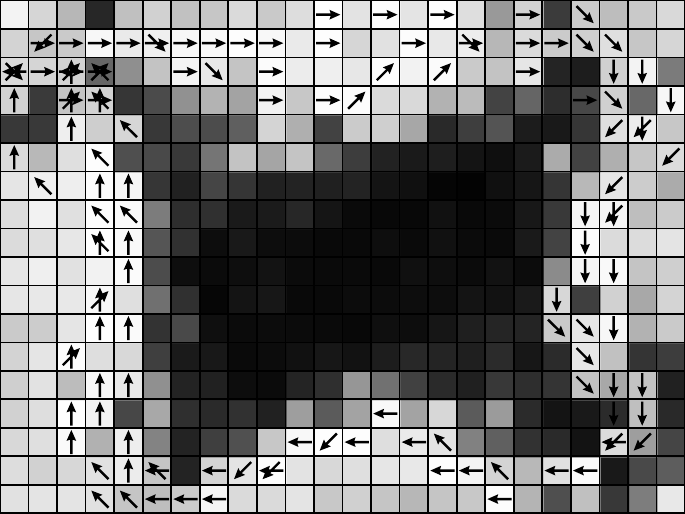}
	}
	\hspace{3ex}
	\subcaptionbox{final predictive model\label{fig:robotPolicy}}{
		\includegraphics[scale=1, angle=0]{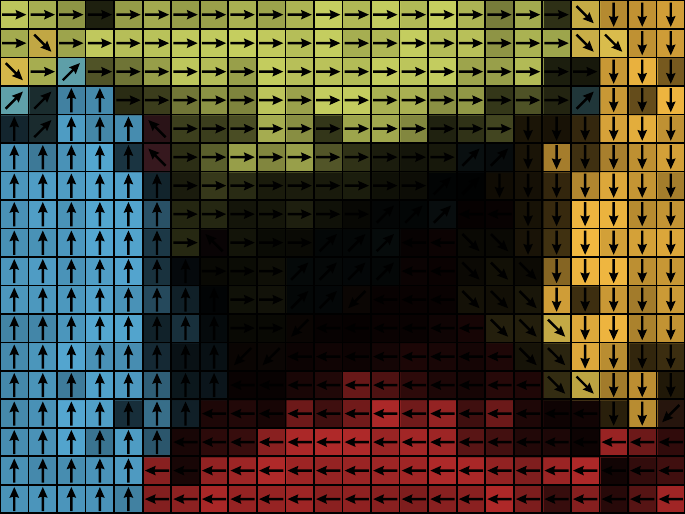}
	}
	\caption{Results of \ddBNIRLS\  on the ``Cycle'' task. (a)~Raw measurements (white lines) and discretized demonstration data (black arrows). The coloring of the background indicates a partitioning structure obtained from a low-temperature posterior sample. (b)~Maximum a posteriori policy estimate. (c)~Visualization of the model's prediction uncertainty at all system states, represented by the entropies of the corresponding posterior predictive action distributions. %
Dark background indicates high uncertainty. (d)~Illustration of the final predictive model, comprising both the action information and the prediction uncertainty.}
	\label{fig:robot}
\end{figure}

Figure~\ref{fig:robotData} shows the %
downsampled and discretized data set (black arrows) obtained from four expert trajectories (white lines). For visualization purposes, the discretization block size is chosen as \SI{2}{\centi\metre}$\times$\SI{2}{\centi\metre}, giving rise to a total of $18\times24=432$ states. %
As in the top row of Figure~\ref{fig:policySynthesis} (\ddBNIRLS), the coloring of the background indicates the learned partitioning structure, computed from a low-temperature posterior sample. We observe that the found state clusters clearly reveal the modular structure of the %
task, providing an intuitive and interpretable explanation of the data. 
However, although the induced %
policy model (Figure~\ref{fig:cyclicMAPpolicy}) %
smoothly captures the cyclic nature of the task, %
we cannot expect %
to obtain trustworthy predictions %
in the center region of the state space, due to the lack of additional 
demonstration data %
that %
would be required to unveil the expert's true intention in that region. Clearly, a point estimate %
 such as the shown MAP policy cannot reflect this prediction uncertainty since it does not carry any confidence information. Yet, following a Bayesian approach, %
we %
can naturally quantify the %
prediction 
uncertainty at any query state~$s^*$ based on  the %
shape of the corresponding posterior predictive action
distribution $\p(a^* \given s^*, \mathcal{D})$. %
A straightforward option %
is, for example, to consider the prediction entropy, defined as
\begin{equation*}
	H(s^*) \defeq \sum_{a^*\in\mathcal{A}} \p(a^* \given s^*, \mathcal{D}) \log\p(a^* \given s^*, \mathcal{D}) \eqpunkt
\end{equation*}
In order to obtain an unbiased approximation of the true \textit{non-tempered} predictive distribution $\p(a^* \given s^*, \mathcal{D})$, we run a second Gibbs chain with unaltered temperature in parallel to the tempered chain. The %
resulting entropy %
estimates are summarized in an uncertainty map (Figure~\ref{fig:entropyMap}), which we overlaid on the original prediction result to produce the final figure shown at the bottom right. 
Note that %
the obtained posterior uncertainty information of the model can be further used in an active learning setting, as demonstrated in Section~\ref{sec:activeLearning}.

\subsubsection{Temporal Partitioning}
\label{sec:KUKAtemporal}

Next, we %
turn our attention to the \ddBNIRLT\  model, which we %
test against the vanilla BNIRL approach. %
For this purpose, we consider the full collection of tasks shown in the supplementary video,
which comprises %
different time-dependent expert behaviors of varying complexity.
In order to %
obtain a quantitative performance measure for our evaluation, we %
conducted a manual segmentation of all recorded trajectories, %
thereby creating a set of ground truth subgoal labels %
for all %
observed decision times. The result of this segmentation step 
is depicted in the appendix (Figure~\ref{fig:wholeDataset}, center column). Note that the %
ground truth subgoals are assumed %
immediately at the ends of the corresponding segments.

\begin{figure}[t]
	\centering
	\includegraphics[scale=1]{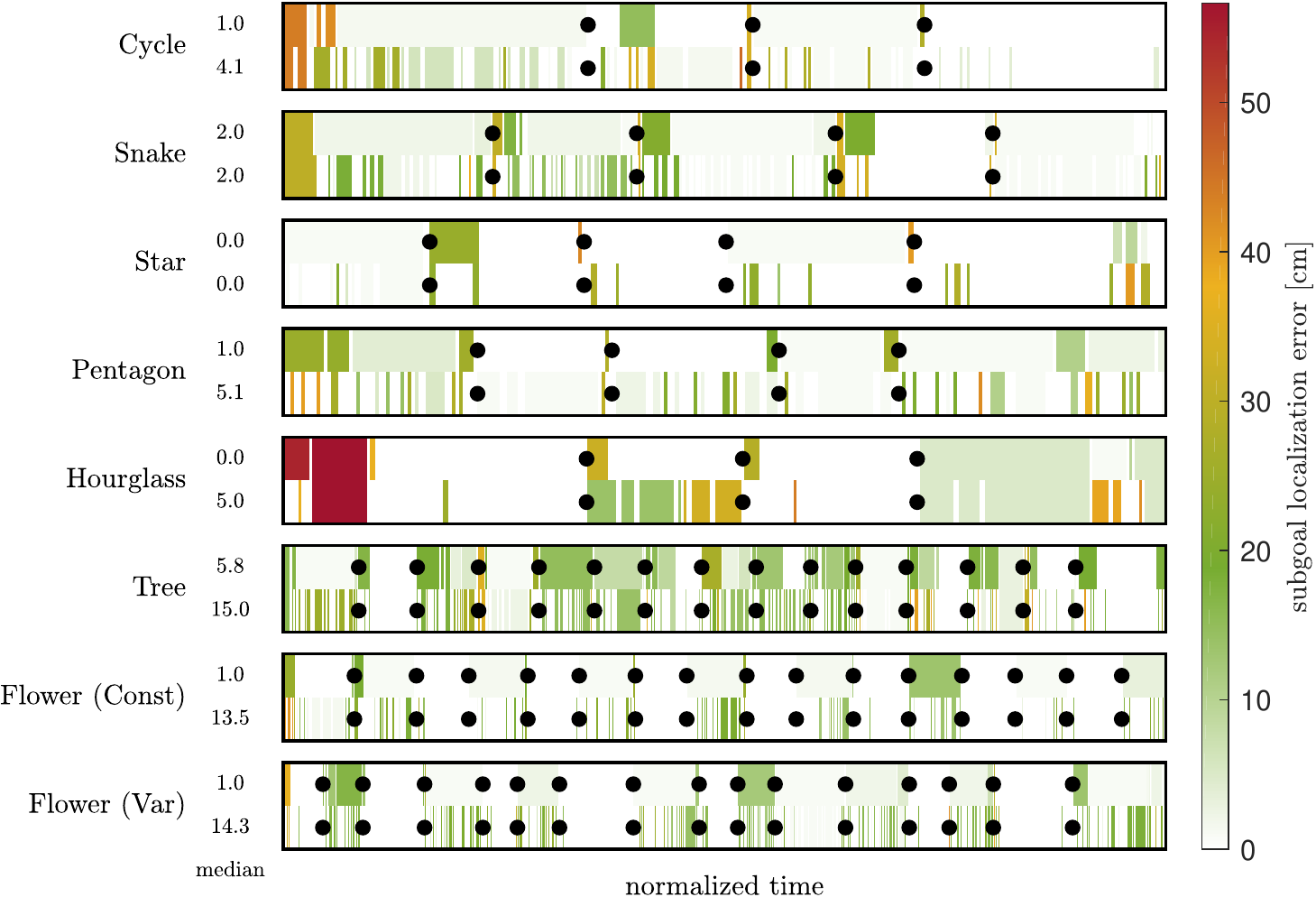}
	\caption{Instantaneous subgoal localization errors of \ddBNIRLT\  (upper rows) and BNIRL (lower rows) for the eight %
recorded data sets. %
The black dots indicate the subgoal switching times in the corresponding ground truth subgoal annotation, depicted in the center column of Figure~\ref{fig:wholeDataset}. On average, the localization error of \ddBNIRLT\ is significantly lower compared to the BNIRL approach, as indicated by the median values shown on the left. For a qualitative comparison of the underlying partitioning structures, see Appendix~\ref{sec:app}.} %
	\label{fig:robotMetric}
\end{figure}

The left and right column of Figure~\ref{fig:wholeDataset} show, respectively, the partitioning structures found by BNIRL and ddBNIRL-T, %
based on a uniform subgoal prior distribution with support at the visited states. The underlying state discretization block size is chosen as \SI{1}{\centi\metre}$\times$\SI{1}{\centi\metre}, as indicated by the regular grid in the background. A simple visual comparison of the learned %
structures reveals the clear superiority of \ddBNIRLT\  over vanilla BNIRL on this problem set. %

For our quantitative comparison, we consider the instantaneous subgoal localization errors of the two models over the entire course of a demonstration (Figure~\ref{fig:robotMetric}). Herein, the instantaneous localization error for a given state-action pair is measured in terms of the Euclidean distance between the grid location of the ground truth subgoal associated with the pair and the corresponding subgoal location predicted by the model. Note that the predictions of both models are based on the \textit{entire} trajectory data of an experiment, considering the full posterior information %
after completing the demonstration.
For ddBNIRL-T, which does not directly return a subgoal location estimate but instead provides access to the full subgoal posterior distribution, the error is computed with respect to
 the MAP subgoal locations~$\{\hat{g}_k\}$,
\begin{equation*}
	\hat{g}_k \defeq \argmax_{g_k\in\supp(p_g)}p(g_k \given \mathcal{D}, \mathbf{\tilde{c}}), 
\end{equation*}
using the \ddBNIRLT\ version of Equation~\eqref{eq:subgoalPosterior}\nachschub{see note at the beginning of Section~\ref{sec:predAndInf}}.

The black dots in the figure indicate the time instants where the ground truth annotations change. At those time instants, %
we observe significantly increased localization errors for both models, %
which can be explained by the fact that %
the ground truth annotation is somewhat %
subjective
around the switching points (see labeling in Figure~\ref{fig:wholeDataset}). Also, we notice a comparably high error at the beginning and the end of some 
trajectories, which stems from the imperfect synchronization between the recording interval and the execution of the task (recall that we %
skipped the corresponding data preprocessing step). Hence, to capture the %
accuracy in %
a single figure of performance, we consider the median localization error of each time series, as it masks out these outliers and provides a more realistic error quantification than the sample mean. %
The obtained values are shown next to the error plots in Figure~\ref{fig:robotMetric}, indicating that the \ddBNIRLT\  localization error is in the range of the discretization %
interval in most cases. \textit{Compared to BNIRL, the proposed method yields an error reduction of more than 70\% on average.}

\subsection{Active Learning}
\label{sec:activeLearning}
In Section~\ref{sec:robotExperiment}, %
we saw that the posterior predictive action distribution $p(a^* \given s^*, \mathcal{D})$ %
provides a natural way to quantify the prediction uncertainty of our model at any given query state~$s^*$. This offers the opportunity to apply the framework in an active learning setting, since the induced uncertainty map (see example in Figure~\ref{fig:entropyMap}) indicates in which parts of the state space the trained model %
can process further instructions from the expert most effectively. 

\begin{figure}[]
	\centering
	\includegraphics[scale=1]{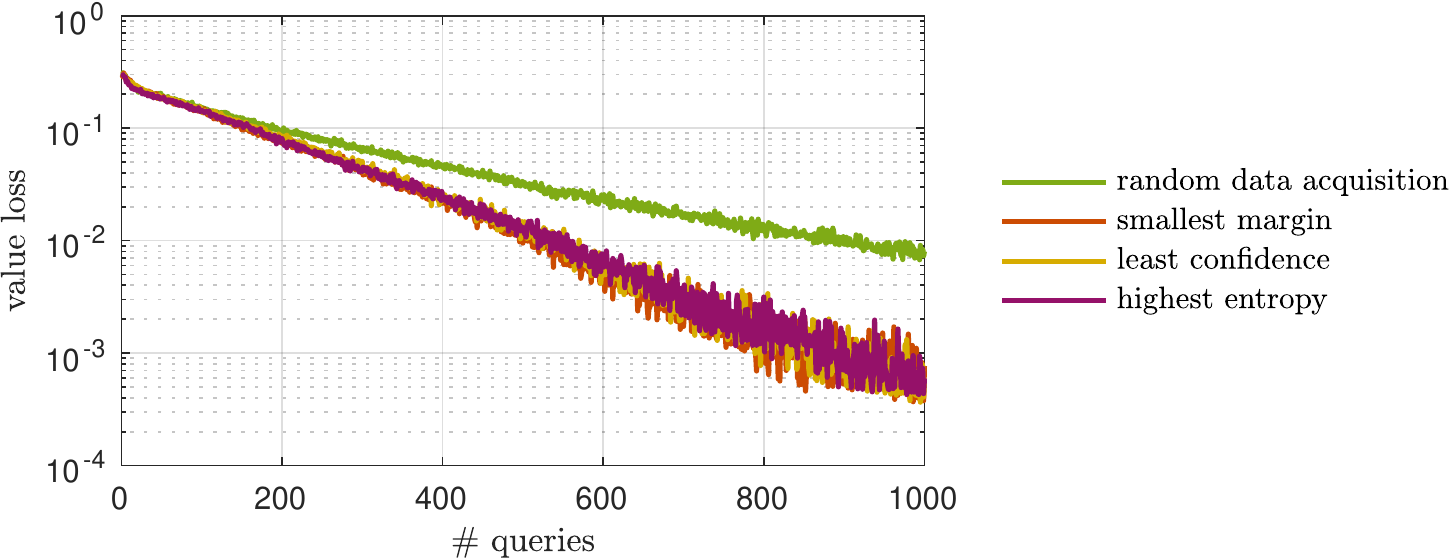}
	\caption{Comparison between random data acquisition and active learning in the random MDP scenario. Shown are the empirical mean value losses of the obtained policy models over the number of data queries, obtained from 200 Monte Carlo runs.}
	\label{fig:activeLearning}
\end{figure}

To demonstrate the basic %
procedure, we reconsider the random MDP problem from Section~\ref{sec:randomMDP} %
in an active learning context, where we compare %
different active strategies with %
the %
previously used random data acquisition scheme. %
As an initialization for the learning procedure, we %
request %
a single state-action pair $(s_1,a_1)$ from the demonstrator, which we store in the initial data set $\mathcal{D}_1\defeq\{(s_1,a_1)\}$. Herein, the state $s_1$ is drawn uniformly at random from $\mathcal{S}$ and the action %
$a_1\sim\pi_E(a \given s_1)$ is generated according to the %
noisy
expert policy $\pi_E:\mathcal{A}\times\mathcal{S}\rightarrow[0,1]$ described in Section~\ref{sec:randomMDP}. %
Continuing from this point, each of the considered active learning algorithms requests a series of subsequent demonstrations $((s_2,a_2), (s_3,a_3), \ldots)$, %
inducing a sequence of data sets $(\mathcal{D}_1, \mathcal{D}_2, \mathcal{D}_3, \ldots)$, where the next query state $s_{d+1}$ is chosen %
according to the specific data acquisition criterion $f_\text{acq}$ of the algorithm evaluated on the current predictive model,
\begin{align*}
\mathcal{D}_{d+1} &= \mathcal{D}_{d}\cup\{(s_{d+1}, a_{d+1})\} \\
s_{d+1} &= \argmax_{s^*\in\mathcal{S}}\, f_{\text{acq}}\big[p(a^* %
\given s^*, \mathcal{D}_{d})\big] \\ 
a_{d+1} &\sim \pi_E(a \given s_{d+1}) \eqpunkt
\end{align*}
The purpose of the acquisition criterion %
is to assess the uncertainty of the %
model at all possible query states, so that the next demonstration can be requested in the high uncertainty region of the state space \citep[see \textit{uncertainty sampling},][]{settles2010active}.
For our experiment, we consider the following three common choices,
\begin{enumerate}[topsep=-\parskip+1.5ex, itemsep=1ex, parsep=0mm]
	\item[$\bullet$] highest entropy: \tabto{3.2cm} $f_\text{acq}(p) \defeq -\sum\limits_{a\in\mathcal{A}}p(a) \log p(a)$,
    \item[$\bullet$] least confidence: \tabto{3.2cm} $f_\text{acq}(p) \defeq 1 - \max\limits_{a\in\mathcal{A}}p(a) \vphantom{\sum\limits_{a\in\mathcal{A}}}$,
    \item[$\bullet$] smallest margin: \tabto{3.2cm} $f_\text{acq}(p) \defeq p(\hat{a}_2) - p(\hat{a}_1)$, %
\end{enumerate} 
where $\hat{a}_1$ and $\hat{a}_2$ denote, respectively, the most likely and second most likely action according to the considered distribution $p$, \ie, $\hat{a}_1\defeq\argmax_{a\in\mathcal{A}}p(a)$ and $\hat{a}_2\defeq\argmax_{a\in\mathcal{A}\setminus\hat{a}_1}p(a)$.
At each iteration, we compute the value losses (Equation~\ref{eq:valueLoss}) of the induced policy models and compare them with the corresponding loss obtained from random data acquisition. The resulting curves are delineated in Figure~\ref{fig:activeLearning}. As expected, the learning speed of the model is significantly improved under all active acquisition schemes, which reduces the number of expert demonstrations required to successfully %
learn the observed task. %

\section{Conclusion}
\label{sec:conclusion}
Building upon the principle of Bayesian nonparametric inverse reinforcement learning, we proposed a new framework for data-efficient IRL that %
leverages the context %
information of the demonstration set to learn a predictive model of the expert behavior from small amounts of training data. %
Central to our framework are two model architectures, one designed for learning spatial subgoal plans, the other %
to capture %
time-varying %
intentions. In contrast to the original BNIRL model, both architectures explicitly consider the covariate information contained in the demonstration set, giving rise to predictive models that are inherently robust to demonstration noise. While the original BNIRL model can be recovered as a special case of our framework, the conducted experiments show a drastic improvement over the vanilla BNIRL approach in terms of the achieved subgoal localization accuracy, %
which stems from both 
an improved likelihood model and a context-aware clustering of the data. Most notably, our framework outperforms all tested reference methods in the analyzed benchmark scenarios %
while it additionally captures the full posterior information about the learned subgoal representation. %
The %
resulting prediction uncertainty %
about the expert behavior, reflected by the posterior predictive action distribution, provides a natural basis to apply our method in an active learning setting where the learning 
system can request additional demonstration data from the expert. %

The current limitation of our approach is that both presented architectures require an MDP model with discrete state and action space. While the subgoal principle carries over straightforwardly to continuous metric spaces, the construction of the likelihood model becomes difficult in these environments as it requires knowledge of the optimal state-action value functions for all potential subgoal locations. However, for BNIRL, there exist several ways to approximate the likelihood in these cases \citep{michini2013scalable} and the same concepts apply equally to ddBNIRL. Thus, an interesting future study would be to compare the efficacy of both model types on larger problems involving continuous spaces, where it appears %
even more natural to follow a distance-based %
approach. %

\acks{This project has received funding from the European Union's Horizon 2020 research and innovation program under grant agreement No~\#713010~(GOALRobots) and No~\#640554 (SKILLS4ROBOTS).}

\newpage
\appendix
\section{Robot experiment}
\phantomsection
\label{sec:app}
\newlength{\dist}
\setlength{\dist}{5cm}
\setlength{\vsep}{0.75\baselineskip}

\begin{figure}[h!]	
	\centering
	
	\subcaptionbox{Cycle}{
	\centering
	\begin{tikzpicture}
	\node (A) at (-\dist,0) {\includegraphics[]{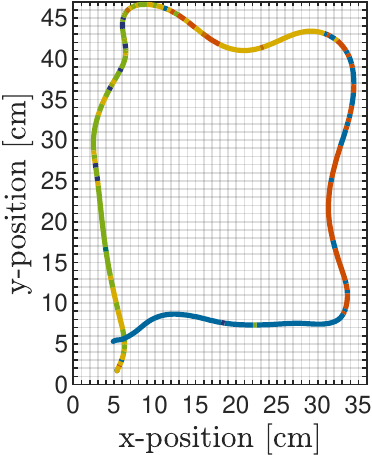}};
	\node (B) at (0,0) {\includegraphics[]{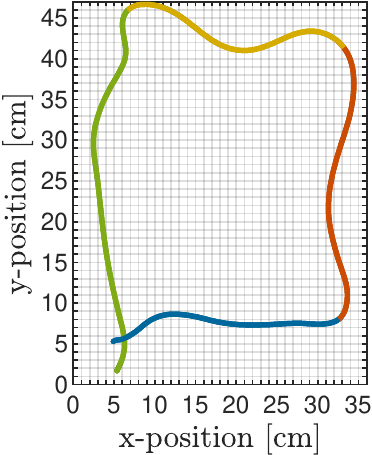}};
	\node (C) at (\dist,0) {\includegraphics[]{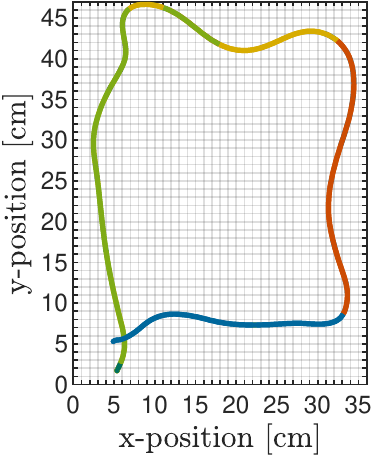}};
	\node [above=0cm of A, xshift=4mm] {BNIRL};
	\node [above=0cm of B, xshift=4mm] {ground truth};
	\node [above=0cm of C, xshift=4mm] {ddBNIRL-T};
	\end{tikzpicture}
	}

	\vspace{\vsep}

	\subcaptionbox{Snake}{
	\centering
	\begin{tikzpicture}
	\node at (-\dist,0) {\includegraphics[]{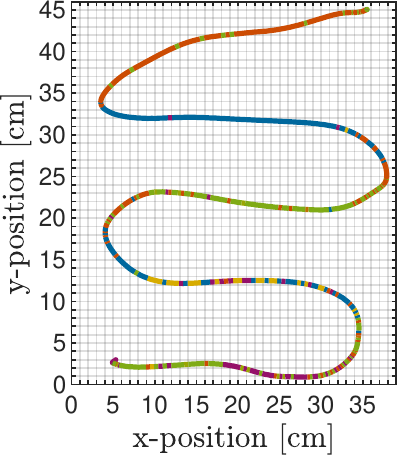}};
	\node at (0,0) {\includegraphics[]{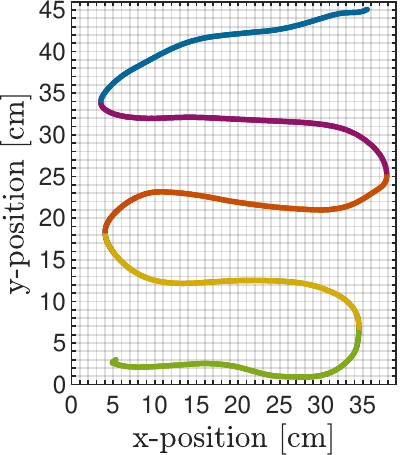}};
	\node at (\dist,0) {\includegraphics[]{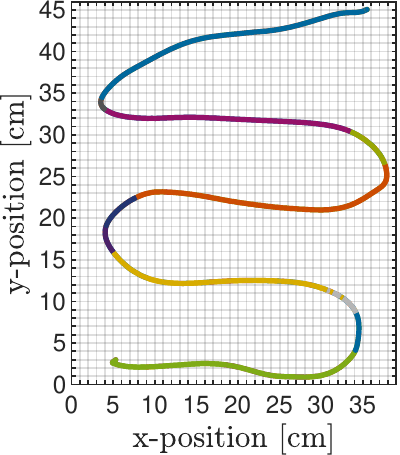}};
	\end{tikzpicture}
	}

	\vspace{\vsep}

	\subcaptionbox{Pentagon}{
	\centering
	\begin{tikzpicture}
	\node at (-\dist,0) {\includegraphics[]{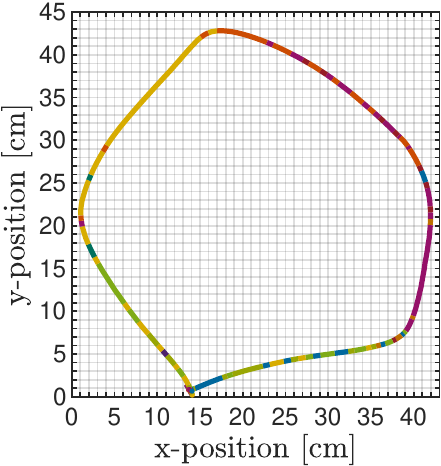}};
	\node at (0,0) {\includegraphics[]{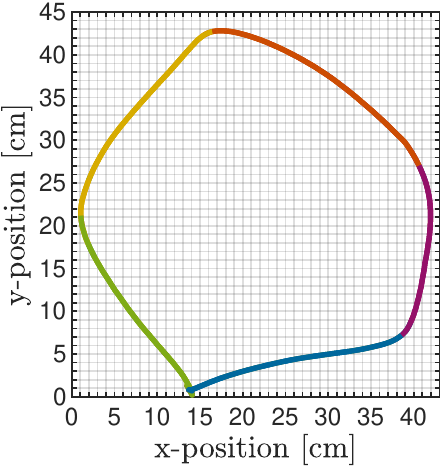}};
	\node at (\dist,0) {\includegraphics[]{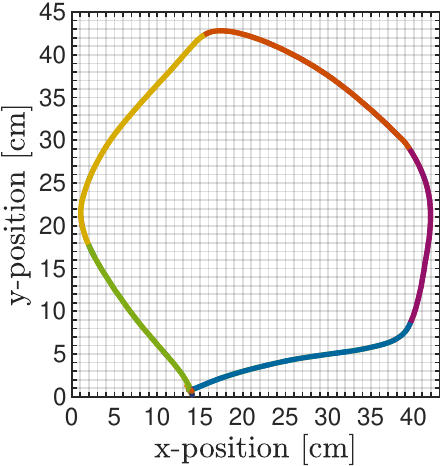}};
	\end{tikzpicture}
	}

	\caption{Motion sequences without trajectory crossings, which can be represented using a spatial subgoal pattern.}
	\label{fig:wholeDataset}

\end{figure}
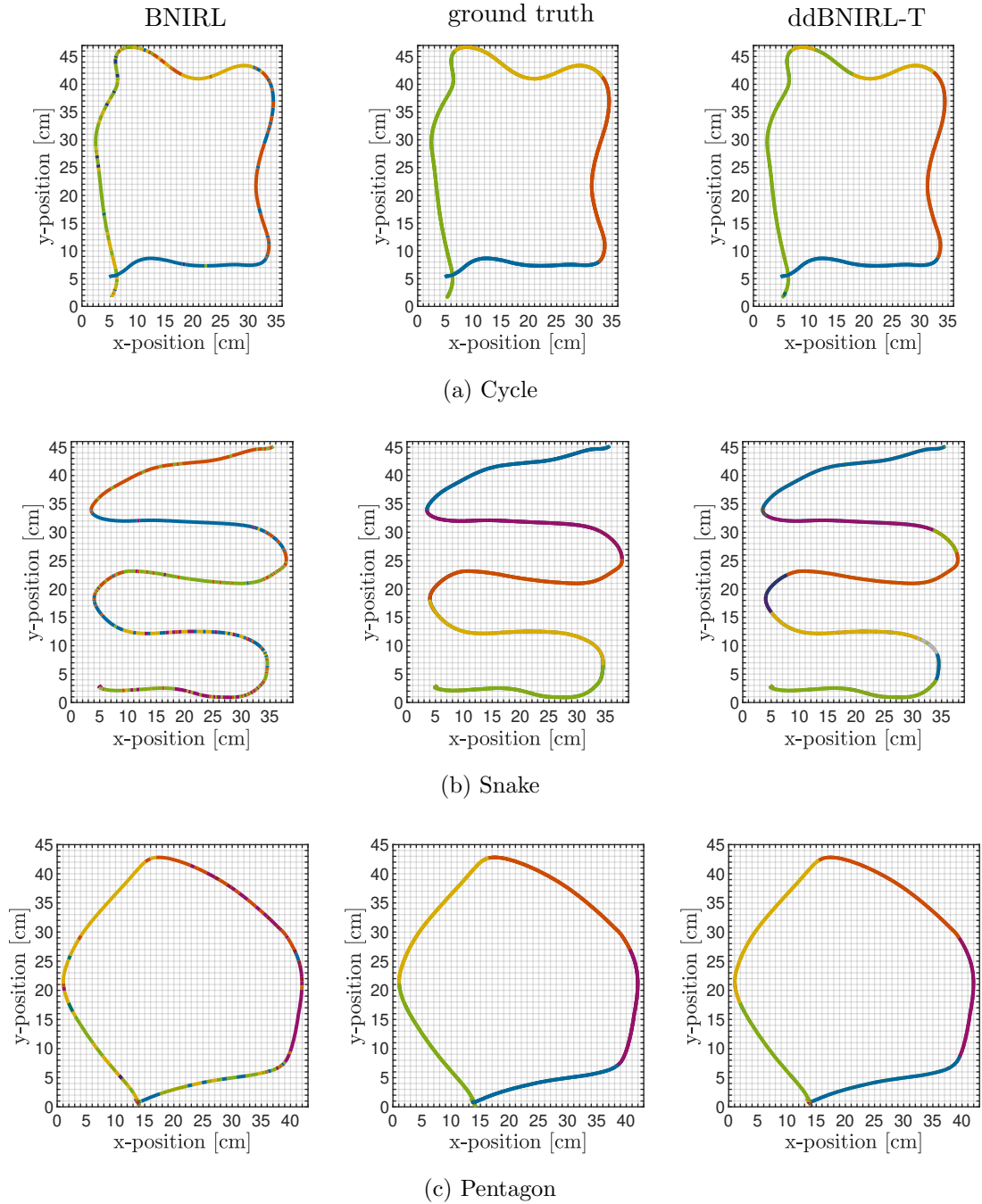

\renewcommand{\thefigure}{\arabic{figure} (continued)}
\begin{figure*}
	\centering
	
	\ContinuedFloat
	
	\subcaptionbox{Star}{
	\centering
	\begin{tikzpicture}
	\node (A) at (-\dist,0) {\includegraphics[]{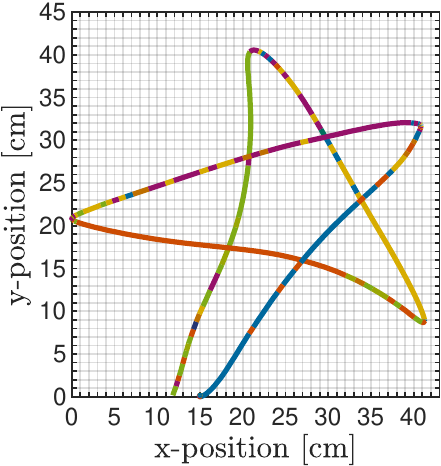}};
	\node (B) at (0,0) {\includegraphics[]{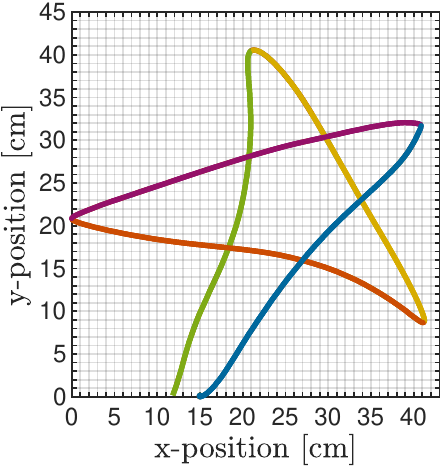}};
	\node (C) at (\dist,0) {\includegraphics[]{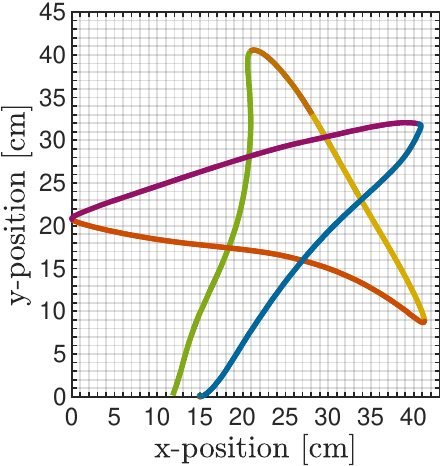}};
	\node [above=0cm of A, xshift=4mm] {BNIRL};
	\node [above=0cm of B, xshift=4mm] {ground truth};
	\node [above=0cm of C, xshift=4mm] {ddBNIRL-T};
	\end{tikzpicture}
	}

	\vspace{\vsep}

	\subcaptionbox{Hourglass}{
	\centering
	\begin{tikzpicture}
	\node at (-\dist,0) {\includegraphics[]{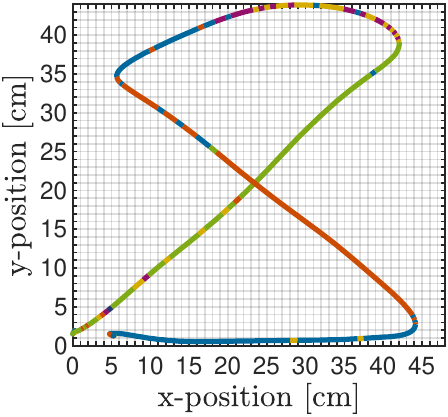}};
	\node at (0,0) {\includegraphics[]{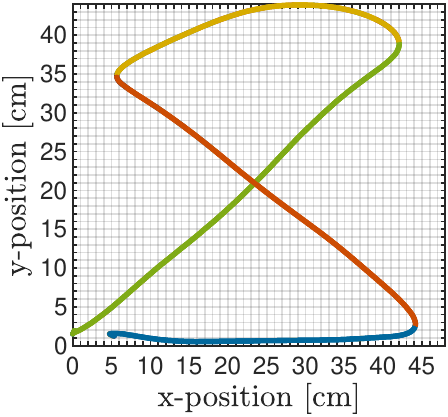}};
	\node at (\dist,0) {\includegraphics[]{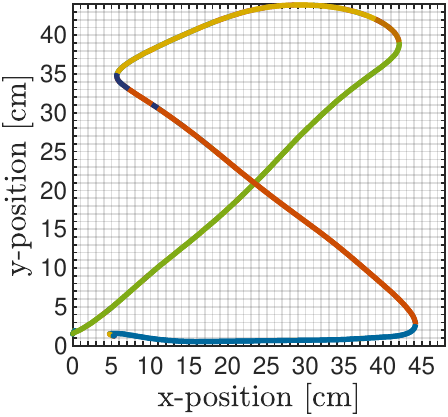}};
	\end{tikzpicture}
	}

\caption{Motion sequences with few trajectory crossings, requiring a time-varying subgoal representation.}
\end{figure*}

\begin{figure*}
	\centering
	
	\ContinuedFloat

	\subcaptionbox{Flower (Const)}{
	\centering
	\begin{tikzpicture}
	\node (A) at (-\dist,0) {\includegraphics[]{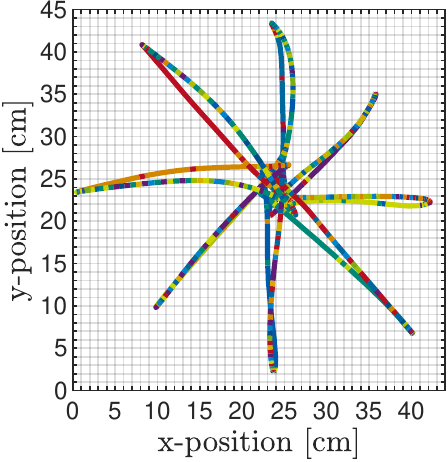}};
	\node (B) at (0,0) {\includegraphics[]{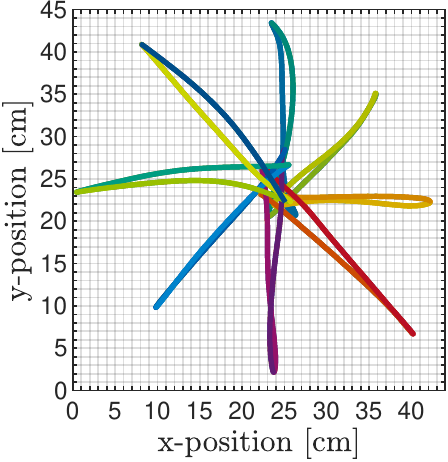}};
	\node (C) at (\dist,0) {\includegraphics[]{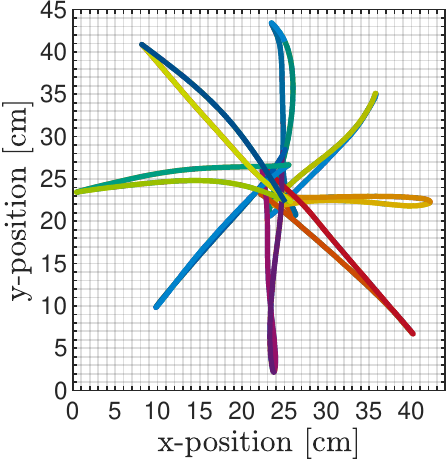}};
	\node [above=0cm of A, xshift=4mm] {BNIRL};
	\node [above=0cm of B, xshift=4mm] {ground truth};
	\node [above=0cm of C, xshift=4mm] {ddBNIRL-T};
	\end{tikzpicture}
	}

	\vspace{\vsep}

	\subcaptionbox{Flower (Var)}{
	\centering
	\begin{tikzpicture}
	\node at (-\dist,0) {\includegraphics[]{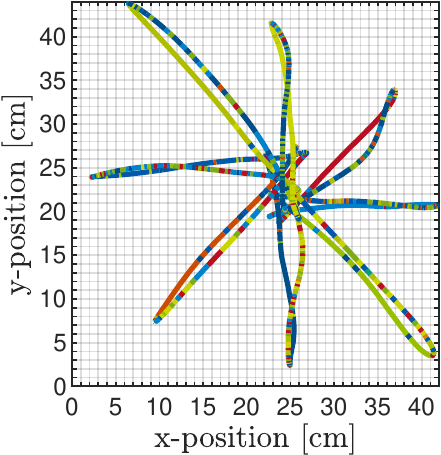}};
	\node at (0,0) {\includegraphics[]{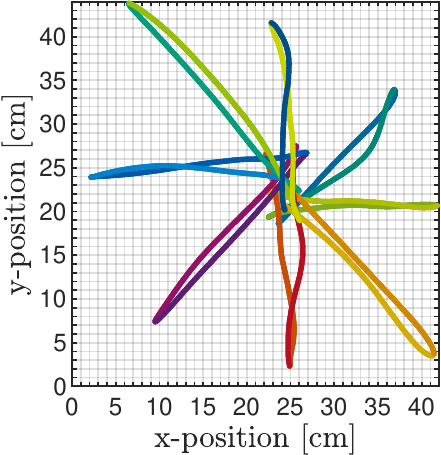}};
	\node at (\dist,0) {\includegraphics[]{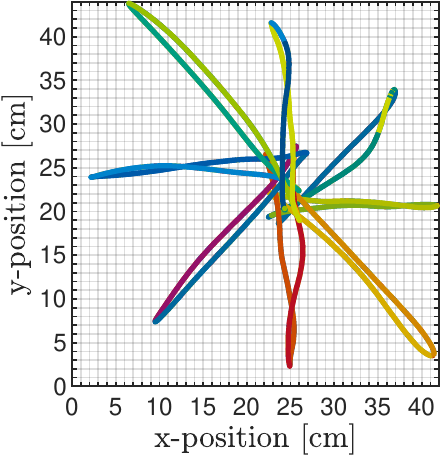}};
	\end{tikzpicture}
	}

	\vspace{\vsep}

	\subcaptionbox{Tree}{
	\centering
	\begin{tikzpicture}
	\node at (-\dist,0) {\includegraphics[]{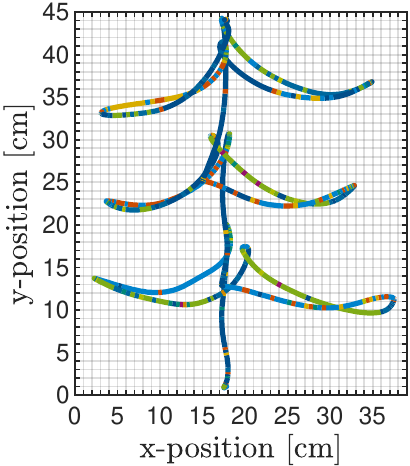}};
	\node at (0,0) {\includegraphics[]{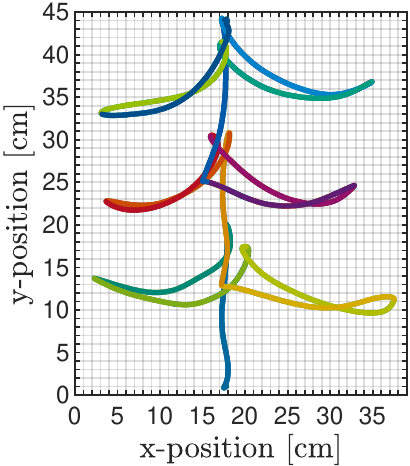}};
	\node at (\dist,0) {\includegraphics[]{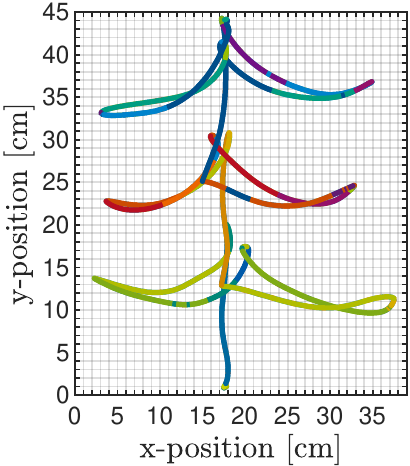}};
	\end{tikzpicture}
	}

\caption{Long motion sequences comprising a large number of sub-patterns with overlapping parts that can be only separated by considering the temporal context. Flower~(Const):~all strokes are performed with the same absolute velocity. Flower~(Var):~the individual strokes are performed with alternating velocity.}
\end{figure*}

\FloatBarrier
\newpage
\bibliography{bibliography}

\begin{thebibliography}{57}
\providecommand{\natexlab}[1]{#1}
\providecommand{\url}[1]{\texttt{#1}}
\expandafter\ifx\csname urlstyle\endcsname\relax
  \providecommand{\doi}[1]{doi: #1}\else
  \providecommand{\doi}{doi: \begingroup \urlstyle{rm}\Url}\fi

\bibitem[Abbeel and Ng(2004)]{abbeel2004apprenticeship}
P.~Abbeel and A.~Y. Ng.
\newblock Apprenticeship learning via inverse reinforcement learning.
\newblock In \emph{International Conference on Machine Learning}, page~1, 2004.

\bibitem[Al-Emran(2015)]{al2015hierarchical}
M.~Al-Emran.
\newblock Hierarchical reinforcement learning: a survey.
\newblock \emph{International Journal of Computing and Digital Systems},
  4\penalty0 (2), 2015.

\bibitem[Albrecht and Stone(2017)]{albrecht2017autonomous}
S.~V. Albrecht and P.~Stone.
\newblock Autonomous agents modelling other agents: a comprehensive survey and
  open problems.
\newblock \emph{\tt arXiv:1709.08071 [cs.AI]}, 2017.

\bibitem[Aldous(1985)]{aldous1985exchangeability}
D.~J. Aldous.
\newblock \emph{Exchangeability and Related Topics}.
\newblock Springer, 1985.

\bibitem[Argall et~al.(2009)Argall, Chernova, Veloso, and
  Browning]{argall2009survey}
B.~D. Argall, S.~Chernova, M.~Veloso, and B.~Browning.
\newblock A survey of robot learning from demonstration.
\newblock \emph{Robotics and Autonomous Systems}, 57\penalty0 (5):\penalty0
  469--483, 2009.

\bibitem[Babe\c{s}-Vroman et~al.(2011)Babe\c{s}-Vroman, Marivate, Subramanian,
  and Littman]{babes2011apprenticeship}
M.~Babe\c{s}-Vroman, V.~Marivate, K.~Subramanian, and M.~Littman.
\newblock Apprenticeship learning about multiple intentions.
\newblock In \emph{International Conference on Machine Learning}, pages
  897--904, 2011.

\bibitem[Baird(1993)]{baird1994reinforcement}
L.~C. Baird.
\newblock Advantage updating.
\newblock Technical report, Wright Lab, 1993.

\bibitem[Bhatnagar et~al.(2009)Bhatnagar, Sutton, Ghavamzadeh, and
  Lee]{bhatnagar2009natural}
S.~Bhatnagar, R.~Sutton, M.~Ghavamzadeh, and M.~Lee.
\newblock Natural actor-critic algorithms.
\newblock \emph{Automatica}, 45\penalty0 (11), 2009.

\bibitem[Blei and Frazier(2011)]{blei2011distance}
D.~M. Blei and P.~I. Frazier.
\newblock Distance dependent {C}hinese restaurant processes.
\newblock \emph{Journal of Machine Learning Research}, 12\penalty0
  (Nov):\penalty0 2461--2488, 2011.

\bibitem[Blei et~al.(2003)Blei, Ng, and Jordan]{blei2003latent}
D.~M. Blei, A.~Y. Ng, and M.~I. Jordan.
\newblock Latent {D}irichlet allocation.
\newblock \emph{Journal of Machine Learning Research}, 3\penalty0
  (Jan):\penalty0 993--1022, 2003.

\bibitem[Botvinick(2012)]{botvinick2012hierarchical}
M.~M. Botvinick.
\newblock Hierarchical reinforcement learning and decision making.
\newblock \emph{Current Opinion in Neurobiology}, 22\penalty0 (6):\penalty0
  956--962, 2012.

\bibitem[Bradtke and Duff(1994)]{bradtke1995reinforcement}
S.~J. Bradtke and M.~O. Duff.
\newblock Reinforcement learning methods for continuous-time {M}arkov decision
  problems.
\newblock In \emph{Advances in Neural Information Processing Systems}, pages
  393--400, 1994.

\bibitem[Cesa-Bianchi et~al.(2017)Cesa-Bianchi, Gentile, Neu, and
  Lugosi]{cesa2017boltzmann}
N.~Cesa-Bianchi, C.~Gentile, G.~Neu, and G.~Lugosi.
\newblock Boltzmann exploration done right.
\newblock In \emph{Advances in Neural Information Processing Systems}, pages
  6275--6284, 2017.

\bibitem[Choi and Kim(2012)]{NIPS2012_4737}
J.~Choi and K.-E. Kim.
\newblock Nonparametric {B}ayesian inverse reinforcement learning for multiple
  reward functions.
\newblock In \emph{Advances in Neural Information Processing Systems}, pages
  305--313, 2012.

\bibitem[Daniel et~al.(2016{\natexlab{a}})Daniel, Van~Hoof, Peters, and
  Neumann]{daniel2016probabilistic}
C.~Daniel, H.~Van~Hoof, J.~Peters, and G.~Neumann.
\newblock Probabilistic inference for determining options in reinforcement
  learning.
\newblock \emph{Machine Learning}, 104\penalty0 (2-3):\penalty0 337--357,
  2016{\natexlab{a}}.

\bibitem[Daniel et~al.(2016{\natexlab{b}})Daniel, Neumann, Kroemer, and
  Peters]{daniel2016hierarchical}
Christian Daniel, Gerhard Neumann, Oliver Kroemer, and Jan Peters.
\newblock Hierarchical relative entropy policy search.
\newblock \emph{Journal of Machine Learning Research}, 17\penalty0
  (1):\penalty0 3190--3239, 2016{\natexlab{b}}.

\bibitem[Dimitrakakis and Rothkopf(2011)]{dimitrakakis2011bayesian}
C.~Dimitrakakis and C.~A. Rothkopf.
\newblock Bayesian multitask inverse reinforcement learning.
\newblock In \emph{European Workshop on Reinforcement Learning}, pages
  273--284, 2011.

\bibitem[Foti and Williamson(2015)]{foti2015survey}
N.~J. Foti and S.~A. Williamson.
\newblock A survey of non-exchangeable priors for {B}ayesian nonparametric
  models.
\newblock \emph{IEEE Transactions on Pattern Analysis and Machine
  Intelligence}, 37\penalty0 (2):\penalty0 359--371, 2015.

\bibitem[Ghavamzadeh et~al.(2015)Ghavamzadeh, Mannor, Pineau, and
  Tamar]{ghavamzadeh2015bayesian}
M.~Ghavamzadeh, S.~Mannor, J.~Pineau, and A.~Tamar.
\newblock Bayesian reinforcement learning: a survey.
\newblock \emph{Foundations and Trends in Machine Learning}, 8\penalty0
  (5-6):\penalty0 359--483, 2015.

\bibitem[Kapron et~al.(2013)Kapron, King, and Mountjoy]{kapron2013dynamic}
B.~M. Kapron, V.~King, and B.~Mountjoy.
\newblock Dynamic graph connectivity in polylogarithmic worst case time.
\newblock In \emph{Annual ACM-SIAM Symposium on Discrete Algorithms}, pages
  1131--1142, 2013.

\bibitem[Kirkpatrick et~al.(1983)Kirkpatrick, Gelatt, and
  Vecchi]{kirkpatrick1983optimization}
S.~Kirkpatrick, C.~D. Gelatt, and M.~P. Vecchi.
\newblock Optimization by simulated annealing.
\newblock \emph{Science}, 220\penalty0 (4598):\penalty0 671--680, 1983.

\bibitem[Koller and Friedman(2009)]{koller2009probabilistic}
D.~Koller and N.~Friedman.
\newblock \emph{Probabilistic Graphical Models: Principles and Techniques}.
\newblock MIT Press, 2009.

\bibitem[Konidaris et~al.(2012)Konidaris, Kuindersma, Grupen, and
  Barto]{konidaris2012robot}
G.~Konidaris, S.~Kuindersma, R.~Grupen, and A.~Barto.
\newblock Robot learning from demonstration by constructing skill trees.
\newblock \emph{International Journal of Robotics Research}, 31\penalty0
  (3):\penalty0 360--375, 2012.

\bibitem[Krishnan et~al.(2016)Krishnan, Garg, Liaw, Miller, Pokorny, and
  Goldberg]{hirl2016}
S.~Krishnan, A.~Garg, R.~Liaw, L.~Miller, F.~T. Pokorny, and K.~Goldberg.
\newblock {HIRL}: hierarchical inverse reinforcement learning for long-horizon
  tasks with delayed rewards.
\newblock \emph{\tt arXiv:1604.06508 [cs.RO]}, 2016.

\bibitem[Levine et~al.(2011)Levine, Popovic, and Koltun]{levine2011nonlinear}
S.~Levine, Z.~Popovic, and V.~Koltun.
\newblock Nonlinear inverse reinforcement learning with {G}aussian processes.
\newblock In \emph{Advances in Neural Information Processing Systems}, pages
  19--27, 2011.

\bibitem[Lioutikov et~al.(2017)Lioutikov, Neumann, Maeda, and
  Peters]{lioutikov2017learning}
R.~Lioutikov, G.~Neumann, G.~Maeda, and J.~Peters.
\newblock Learning movement primitive libraries through probabilistic
  segmentation.
\newblock \emph{International Journal of Robotics Research}, 36\penalty0
  (8):\penalty0 879--894, 2017.

\bibitem[Littman et~al.(1995)Littman, Dean, and
  Kaelbling]{littman1995complexity}
M.~L. Littman, T.~L. Dean, and L.~P. Kaelbling.
\newblock On the complexity of solving {M}arkov decision problems.
\newblock In \emph{Conference on Uncertainty in Artificial Intelligence}, pages
  394--402, 1995.

\bibitem[Michini and How(2012)]{michini2012bayesian}
B.~Michini and J.~P. How.
\newblock Bayesian nonparametric inverse reinforcement learning.
\newblock In \emph{Joint European Conference on Machine Learning and Knowledge
  Discovery in Databases}, pages 148--163, 2012.

\bibitem[Michini et~al.(2013)Michini, Cutler, and How]{michini2013scalable}
B.~Michini, M.~Cutler, and J.~P. How.
\newblock Scalable reward learning from demonstration.
\newblock In \emph{IEEE International Conference on Robotics and Automation},
  pages 303--308, 2013.

\bibitem[Michini et~al.(2015)Michini, Walsh, Agha-Mohammadi, and
  How]{michini2015bayesian}
B.~Michini, T.~J. Walsh, A.-A. Agha-Mohammadi, and J.~P. How.
\newblock Bayesian nonparametric reward learning from demonstration.
\newblock \emph{IEEE Transactions on Robotics}, 31\penalty0 (2):\penalty0
  369--386, 2015.

\bibitem[Neu and Szepesv\'{a}ri(2007)]{neu2007}
G.~Neu and C.~Szepesv\'{a}ri.
\newblock Apprenticeship learning using inverse reinforcement learning and
  gradient methods.
\newblock In \emph{Conference on Uncertainty in Artificial Intelligence}, pages
  295--302, 2007.

\bibitem[Ng and Russell(2000)]{ng2000algorithms}
A.~Y. Ng and S.~J. Russell.
\newblock Algorithms for inverse reinforcement learning.
\newblock In \emph{International Conference on Machine Learning}, pages
  663--670, 2000.

\bibitem[Ng et~al.(1999)Ng, Harada, and Russell]{ng1999policy}
A.~Y. Ng, D.~Harada, and S.~Russell.
\newblock Policy invariance under reward transformations: Theory and
  application to reward shaping.
\newblock In \emph{International Conference on Machine Learning}, pages
  278--287, 1999.

\bibitem[Nguyen et~al.(2015)Nguyen, Low, and Jaillet]{nguyen2015inverse}
Q.~P. Nguyen, B.~K.~H. Low, and P.~Jaillet.
\newblock Inverse reinforcement learning with locally consistent reward
  functions.
\newblock In \emph{Advances in Neural Information Processing Systems}, pages
  1747--1755, 2015.

\bibitem[Niekum et~al.(2012)Niekum, Osentoski, Konidaris, and
  Barto]{niekum2012learning}
S.~Niekum, S.~Osentoski, G.~Konidaris, and A.~G. Barto.
\newblock Learning and generalization of complex tasks from unstructured
  demonstrations.
\newblock In \emph{IEEE/RSJ International Conference on Intelligent Robots and
  Systems}, pages 5239--5246, 2012.

\bibitem[Panella and Gmytrasiewicz(2017)]{panella2017interactive}
A.~Panella and P~Gmytrasiewicz.
\newblock Interactive {POMDP}s with finite-state models of other agents.
\newblock \emph{Autonomous Agents and Multi-Agent Systems}, pages 861--904,
  2017.

\bibitem[Puterman(1994)]{puterman1994}
M.~L. Puterman.
\newblock \emph{Markov Decision Processes: Discrete Stochastic Dynamic
  Programming}.
\newblock John Wiley \& Sons, Inc., 1994.

\bibitem[Ramachandran and Amir(2007)]{ramachandran2007bayesian}
D.~Ramachandran and E.~Amir.
\newblock Bayesian inverse reinforcement learning.
\newblock \emph{International Joint Conference on Artificial Intelligence},
  pages 2586--2591, 2007.

\bibitem[Ranchod et~al.(2015)Ranchod, Rosman, and
  Konidaris]{ranchod2015nonparametric}
P.~Ranchod, B.~Rosman, and G.~Konidaris.
\newblock Nonparametric {B}ayesian reward segmentation for skill discovery
  using inverse reinforcement learning.
\newblock In \emph{IEEE/RSJ International Conference on Intelligent Robots and
  Systems}, pages 471--477, 2015.

\bibitem[Roberts and Sahu(1997)]{roberts1997updating}
G.~O. Roberts and S.~K. Sahu.
\newblock Updating schemes, correlation structure, blocking and
  parameterization for the {G}ibbs sampler.
\newblock \emph{Journal of the Royal Statistical Society: Series B (Statistical
  Methodology)}, 59\penalty0 (2):\penalty0 291--317, 1997.

\bibitem[Rothkopf and Dimitrakakis(2011)]{rothkopf2011preference}
C.~A. Rothkopf and C.~Dimitrakakis.
\newblock Preference elicitation and inverse reinforcement learning.
\newblock In \emph{Joint European Conference on Machine Learning and Knowledge
  Discovery in Databases}, pages 34--48, 2011.

\bibitem[Rueckert et~al.(2013)Rueckert, Neumann, Toussaint, and
  Maass]{Rueckert2013}
E.~Rueckert, G.~Neumann, M.~Toussaint, and W.~Maass.
\newblock Learned graphical models for probabilistic planning provide a new
  class of movement primitives.
\newblock \emph{Frontiers in Computational Neuroscience}, 6:\penalty0 97, 2013.

\bibitem[Schaal et~al.(2005)Schaal, Peters, Nakanishi, and
  Ijspeert]{schaal2005learning}
S.~Schaal, J.~Peters, J.~Nakanishi, and A.~Ijspeert.
\newblock Learning movement primitives.
\newblock \emph{Robotics Research}, pages 561--572, 2005.

\bibitem[Settles(2010)]{settles2010active}
B.~Settles.
\newblock Active learning literature survey.
\newblock Technical report, University of Wisconsin-Madison, 2010.

\bibitem[{\c{S}}im\c{s}ek et~al.(2005){\c{S}}im\c{s}ek, Wolfe, and
  Barto]{Simsek2005}
\"{O} {\c{S}}im\c{s}ek, A.~P. Wolfe, and A.~G. Barto.
\newblock Identifying useful subgoals in reinforcement learning by local graph
  partitioning.
\newblock In \emph{International Conference on Machine Learning}, pages
  816--823, 2005.

\bibitem[{\v{S}}o\v{s}i\'{c} et~al.(2018{\natexlab{a}}){\v{S}}o\v{s}i\'{c},
  Zoubir, and Koeppl]{sosic2018}
A.~{\v{S}}o\v{s}i\'{c}, A.~M. Zoubir, and H.~Koeppl.
\newblock Inverse reinforcement learning via nonparametric subgoal modeling.
\newblock In \emph{AAAI Spring Symposium on Data-Efficient Reinforcement
  Learning}, 2018{\natexlab{a}}.

\bibitem[{\v{S}}o\v{s}i\'{c} et~al.(2018{\natexlab{b}}){\v{S}}o\v{s}i\'{c},
  Zoubir, and Koeppl]{sosic2018pami}
A.~{\v{S}}o\v{s}i\'{c}, A.~M. Zoubir, and H.~Koeppl.
\newblock A {B}ayesian approach to policy recognition and state representation
  learning.
\newblock \emph{IEEE Transactions on Pattern Analysis and Machine
  Intelligence}, 40\penalty0 (6):\penalty0 1295--1308, 2018{\natexlab{b}}.

\bibitem[Stolle and Precup(2002)]{stolle2002learning}
M.~Stolle and D.~Precup.
\newblock Learning options in reinforcement learning.
\newblock In \emph{International Symposium on Abstraction, Reformulation, and
  Approximation}, pages 212--223, 2002.

\bibitem[Surana and Srivastava(2014)]{surana2014bayesian}
A.~Surana and K.~Srivastava.
\newblock Bayesian nonparametric inverse reinforcement learning for switched
  markov decision processes.
\newblock In \emph{IEEE International Conference on Learning and Applications},
  pages 47--54, 2014.

\bibitem[Sutton and Barto(1998)]{sutton1998reinforcement}
R.~S. Sutton and A.~G. Barto.
\newblock \emph{Reinforcement Learning: An Introduction}.
\newblock MIT Press, 1998.

\bibitem[Sutton et~al.(1999)Sutton, Precup, and Singh]{sutton1999between}
R.~S. Sutton, D.~Precup, and S.~Singh.
\newblock Between {MDP}s and semi-{MDP}s: a framework for temporal abstraction
  in reinforcement learning.
\newblock \emph{Artificial Intelligence}, 112\penalty0 (1):\penalty0 181--211,
  1999.

\bibitem[Tamassia et~al.(2015)Tamassia, Zambetta, Raffe, and
  Li]{tamassia2015learning}
M.~Tamassia, F.~Zambetta, W.~Raffe, and X.~Li.
\newblock Learning options for an {MDP} from demonstrations.
\newblock In \emph{Australasian Conference on Artificial Life and Computational
  Intelligence}, pages 226--242, 2015.

\bibitem[Taylor and Karlin(1984)]{taylor2014introduction}
H.~M. Taylor and S.~Karlin.
\newblock \emph{An Introduction to Stochastic Modeling}.
\newblock Academic Press, 1984.

\bibitem[Tewari and Bartlett(2008)]{tewari2008optimistic}
A.~Tewari and P.~L. Bartlett.
\newblock Optimistic linear programming gives logarithmic regret for
  irreducible {MDP}s.
\newblock In \emph{Advances in Neural Information Processing Systems}, pages
  1505--1512, 2008.

\bibitem[Yang et~al.(2007)Yang, Jiang, Hauptmann, and Ngo]{yang2007evaluating}
J.~Yang, Y.-G. Jiang, A.~G. Hauptmann, and C.-W. Ngo.
\newblock Evaluating bag-of-visual-words representations in scene
  classification.
\newblock In \emph{International Workshop on Multimedia Information Retrieval},
  pages 197--206, 2007.

\bibitem[Zhifei and Joo(2012)]{zhifei2012}
S.~Zhifei and E.~M. Joo.
\newblock A survey of inverse reinforcement learning techniques.
\newblock \emph{International Journal of Intelligent Computing and
  Cybernetics}, 5\penalty0 (3):\penalty0 293--311, 2012.

\bibitem[Ziebart et~al.(2008)Ziebart, Maas, Bagnell, and
  Dey]{ziebart2008maximum}
B.~D. Ziebart, A.~L. Maas, J.~A. Bagnell, and A.~K. Dey.
\newblock Maximum entropy inverse reinforcement learning.
\newblock In \emph{AAAI Conference on Artificial Intelligence}, pages
  1433--1438, 2008.

\end{thebibliography}

\end{document}